\documentclass{article}

\usepackage{arxiv}
\usepackage[utf8]{inputenc} 
\usepackage[T1]{fontenc}    
\usepackage[square,numbers]{natbib}
\usepackage{hyperref}       
\usepackage{url}            
\usepackage{booktabs}       
\usepackage{amsmath}
\usepackage{amsfonts}       
\usepackage{nicefrac} 
\usepackage{booktabs}
\usepackage{multirow}
\usepackage{microtype}  
\usepackage{tabularray}
\usepackage{array}
\usepackage{longtable}
\usepackage{upgreek}
\usepackage{lipsum}         
\usepackage{graphicx}
\usepackage{makecell}
\usepackage{doi}
\usepackage{titlesec}
\usepackage{wrapfig}
\usepackage{subcaption}
\usepackage{array}
\usepackage{fix-cm}
\usepackage{soul}
\usepackage[toc,page]{appendix}
\usepackage{makecell}
\usepackage[table,xcdraw]{xcolor}
\usepackage[english]{babel}
\usepackage{cleveref}       
\usepackage{float}
\usepackage{amsthm}
\usepackage{thmtools}
\usepackage{mathrsfs}
\usepackage{amssymb}
\usepackage{pgfplots}
\usepackage{booktabs}
\usepackage{hhline}
\usepackage{tabularx}
\usepackage{caption}
\usepackage{multirow}
\usepackage{todonotes}
\usepackage{doi}
\usepackage{adjustbox}
\usepackage{xcolor}
\usepackage{anyfontsize}
\usepackage[linesnumbered,ruled,vlined]{algorithm2e}
\usepackage{tikz,pgfplots,pgf}
\usepackage{authblk}
\usepackage{boxedminipage}
\usepackage{neuralnetwork}
\usepackage[normalem]{ulem}

\usetikzlibrary{spy}
\usetikzlibrary{arrows,shapes,calc}
\usetikzlibrary{tikzmark}
\usetikzlibrary{positioning}
\usetikzlibrary{arrows,chains,positioning,scopes,quotes}
\usetikzlibrary{decorations.pathreplacing,calligraphy}
\usetikzlibrary{decorations.pathreplacing,patterns,positioning}
\usetikzlibrary{matrix}
\usetikzlibrary{chains}
\usetikzlibrary{shapes,arrows,matrix,fit,positioning,trees}
\usetikzlibrary{arrows,through,shapes,positioning,fit,calc}
\usetikzlibrary{patterns}
\usetikzlibrary{fadings}
\usetikzlibrary{positioning, shapes.geometric, arrows.meta, calc, fit, chains, decorations.pathreplacing, calligraphy}
\tikzstyle{line}=[draw]
\tikzstyle{arrow}=[draw, -latex]

\UseTblrLibrary{amsmath}

\newtheoremstyle{mystyle}%
  {3pt} 
  {3pt} 
  {\itshape} 
  {} 
  {\bfseries} 
  {.} 
  { } 
  {\thmname{#1} \thmnumber{#2}: \thmnote{\bfseries #3}} 

\newtheorem{example}{Example}

\newtheorem{proposition}{Proposition}
\newtheorem{definition}{Definition}

\newtheorem{corollary}{Corollary}
\newtheorem{lemma}{Lemma}
\newtheorem{notation}{Notation}

\pgfplotsset{compat=1.13}

\author{}

\usepackage{accents}

\title{Approximate matrices of systems of max-min fuzzy relational equations}

\hypersetup{
pdftitle={Approximate matrices of systems of max-min fuzzy relational equations},
pdfsubject={cs.AI},
pdfauthor={Ismaïl Baaj},
pdfkeywords={Fuzzy relations, Systems of max-min fuzzy relational equations, learning},
}

\author{Ismaïl Baaj \\ LEMMA, Paris-Panthéon-Assas University, Paris, 75006, France \\  \href{baaj@cril.fr}{ismail.baaj@assas-universite.fr}}

\begin{document}

\maketitle
     
\begin{abstract}
In this article, we address the inconsistency of a system of max-min fuzzy relational equations by minimally modifying the matrix governing the system in order to achieve consistency. Our method yields consistent systems that approximate the original inconsistent system in the following sense: the right-hand side vector of each consistent system is that of the inconsistent system, and the coefficients of the matrix governing each consistent system are obtained by modifying, exactly and minimally, the entries of the original matrix that must be corrected to achieve consistency, while leaving all other entries unchanged.

To obtain a consistent system that closely approximates the considered inconsistent system, we study the distance (in terms of a norm among $L_1$, $L_2$ or $L_\infty$) between the matrix of the inconsistent system and the set formed by the matrices of consistent systems that use the same right-hand side vector as the inconsistent system. We show that our method allows us to directly compute  matrices of consistent systems that use the same right-hand side vector as the inconsistent system whose distance in terms of $L_\infty$ norm  to the matrix of the inconsistent system is minimal (the computational costs are higher when using $L_1$ norm or $L_2$ norm).   We also give an explicit analytical formula for computing this minimal $L_\infty$ distance. Finally, we translate our results for systems of  min-max fuzzy relational equations and present some potential   applications.

\end{abstract}


\section{Introduction}

Systems of fuzzy relational equations are the basis of many fuzzy modeling approaches \cite{di1991fuzzy,dubois1995fuzzy,pedrycz1985applications}, including Zadeh's possibility theory for approximate reasoning \cite{ZadehTH}, and have been applied in some practical areas of Artificial Intelligence (AI), such as medical diagnosis \cite{adlassnig1980fuzzy,bandler1986survey,sanchez1977}. The pioneering work of Sanchez \cite{sanchez1976resolution,sanchez1977} provided necessary and sufficient conditions for  a system of max-min fuzzy relational equations to be consistent, i.e., when the system has solutions. Sanchez showed that, when a max-min system is consistent, the structure of its solutions set is given by a solution which is the greatest and a finite number of minimal solutions. His work has been extended to systems based on $\max-T$ compositions where $T$ is a continuous t-norm \cite{baets2000analytical,di1982solution,di1984fuzzy,li2008resolution,markovskii2005relation,MIYAKOSHI198553,pedrycz1982fuzzy,shieh2007solutions}, systems based on $\max-\ast$ compositions, where $\ast$ is an increasing and continuous function \cite{matusiewicz2013increasing}, and systems based on $\min-\mathcal{I}_T$ compositions, where $\mathcal{I}_T$ is the residual implicator associated with a continuous t-norm $T$, see \cite{bandler1980semantics,perfilieva2008system}.

However, addressing the inconsistency of systems of fuzzy relational equations remains an open problem today \cite{baets2000analytical,li2009fuzzy,pedrycz1990inverse}. For  max-min  systems, researchers have worked on finding approximate solutions \cite{cimler2018optimization,cuninghame1995residuation,li2010chebyshev,li2024linear}, and on the computation of approximate inverses of fuzzy matrices \cite{ WEN2022,wu2022analytical}. An emerging development is based on Pedrycz's approach \cite{pedrycz1990inverse}:  Perdycz proposed to slightly modify the right-hand side vector of an inconsistent max-min  system in order to obtain a consistent system. Then, the solutions of the obtained consistent system are considered as approximate solutions of the  inconsistent system. Some authors proposed algorithms  \cite{cimler2018optimization,cuninghame1995residuation,li2010chebyshev}  based on Pedrycz's approach for obtaining a consistent system close to a given inconsistent system.  More recently, the author of \cite{baaj2024maxProdMaxLuka,baaj2024maxmin} studied the inconsistency of systems of $\max-T$ fuzzy relational equations, where $T$ is a continuous t-norm among minimum, product, or Lukasiewicz's t-norm. For each of the three $\max-T$ systems, the author of \cite{baaj2024maxProdMaxLuka,baaj2024maxmin} provided an explicit analytical formula to compute the Chebyshev distance (defined by  L-infinity norm) between the right-hand side vector of an inconsistent system and the set of right-hand side vectors of consistent systems defined with the same composition and the same matrix as the inconsistent system. Based on these results, the author of \cite{baaj2024maxProdMaxLuka,baaj2024maxmin} studied the Chebyshev approximations of the right-hand side vector of a given inconsistent system of $\max-T$ fuzzy relational equations \cite{cuninghame1995residuation,li2010chebyshev}, where each of these approximations is the right-hand side vector of a consistent system defined with the same matrix as the inconsistent system and such that its distance to the right-hand side vector of the inconsistent system is equal to the computed Chebyshev distance. For each of the three $\max-T$ systems, the author of \cite{baaj2024maxProdMaxLuka,baaj2024maxmin} showed that the greatest Chebyshev approximation of the right-hand side vector of an inconsistent $\max-T$ system can be computed by an explicit analytical formula. Furthermore, for max-min  systems, the author of \cite{baaj2024maxmin} provided the complete description of the structure of the set of Chebyshev approximations.  The author of \cite{baaj2024maxProdMaxLuka,baaj2024maxmin} also studied the approximate solutions of inconsistent $\max-T$ systems, defined as the solutions of the consistent systems whose matrix is that of the inconsistent system and whose right-hand side vector is a Chebyshev approximation of the right-hand side vector of the inconsistent system. For each of the three $\max-T$ systems, the author of \cite{baaj2024maxProdMaxLuka,baaj2024maxmin} showed that the greatest approximate solution of a given inconsistent system can be computed by an explicit analytical formula. For systems based on the max-min composition, the author of \cite{baaj2024maxmin} gave a complete description of the structure of the approximate solutions set of the inconsistent system. Furthermore, for max-min systems, the author of \cite{li2024linear} proposed a linear optimization method for computing $L_\infty$ and $L_1$ approximate solutions.

The works \cite{baaj2024maxmin, cimler2018optimization, cuninghame1995residuation, li2010chebyshev} build on Pedrycz’s approach, focus on minimally modifying the right-hand side vector of an inconsistent  systems, while keeping the matrix unmodified. In this article, we follow a different approach: \textit{we focus on minimally modifying the matrix of an inconsistent max-min system in order to obtain a consistent  system, while keeping the right-hand side vector unchanged}. This problem was studied by \cite{CECHLAROVA2000123,li2010chebyshev} where the authors  proposed algorithms for estimating the Chebyshev distance between the  matrix of an inconsistent system and the closest matrix of a consistent system, where both systems have the same fixed right-hand side vector.

In this article, we consider an inconsistent system of max-min fuzzy relational equations of the form $A \Box_{\min}^{\max} x = b$, where  the matrix $A$ is of size $(n,m)$, the right-hand side vector $b$ has $n$ components, $x$ is an unknown vector of $m$ components and the matrix product $\Box_{\min}^{\max}$ uses  the t-norm $\min$   as the product and the function $\max$ as the addition. We begin by defining, in (\ref{eq:setT}), the set $\mathcal{T}$, which is formed by the matrices governing the  consistent max-min systems that use the same right-hand side vector: the right-hand side vector $b$ of the inconsistent system. Then, we introduce, in (\ref{eq:deltaAp}), the distance denoted by $\mathring{\Delta}_{p}$ between the matrix $A$ of the inconsistent system and the set $\mathcal{T}$,  which is measured with respect to a given norm  among $L_1$, $L_2$, or $L_\infty$.

To introduce our method for constructing matrices of the set $\mathcal{T}$, we begin with a preliminary study. We define an auxiliary matrix $A^{(i,j)}$, which minimally modify the matrix $A$ of the inconsistent system $A \Box_{\min}^{\max} x = b$, in order to satisfy the following constraint: $\delta^{A^{(i,j)}}(i,j) = 0$, see  (\ref{eq:deltaTij}) and Lemma \ref{lemma:deltaaijijequalzero}, where the scalar $\delta^{A^{(i,j)}}(i,j)$  is involved in the formula of \cite{baaj2024maxmin} of the Chebyshev distance $\Delta = \max\limits_{1 \leq k \leq n} \delta_k^{A^{(i,j)}} \text{ where } \delta_k^{A^{(i,j)}} =\min\limits_{1 \leq j \leq m} \delta^{A^{(i,j)}}(k,j)$, see (\ref{eq:Delta}), associated with the right-hand side vector $b$ of the  max-min system  $A^{(i,j)} \Box_{\min}^{\max} x = b$. We then define the matrices $A^{(\vec{i}, \vec{j})}$ in (\ref{eq:constructionAijvec}), whose construction is an iterative composition  (\ref{eq:constructionAijvec}) of auxiliary matrices $A^{(i,j)}$. In Proposition \ref{proposition:SimpleDefAijvec}, we give a simpler construction of these matrices $A^{(\vec{i}, \vec{j})}$.

When  the pair of vectors   $(\vec{i}, \vec{j})$ used to construct the matrices $A^{(\vec{i}, \vec{j})}$ of (\ref{eq:constructionAijvec})  are set from the subset of indices corresponding to the inconsistent equations of the system $A \Box_{\min}^{\max} x = b$ with respect to the Chebyshev distance $\Delta$ (each equation whose index $i \in \{1,2,\dots,n\}$ is such that $\delta_i^A > 0$, see (\ref{eq:Delta})) and a subset of column indices of the matrix $A$, we prove in Theorem \ref{theorem:theo1consSij} that the matrices $A^{(\vec{i}, \vec{j})}$ belong to the set $\mathcal{T}$. The matrices $A^{(\vec{i}, \vec{j})}$ of consistent systems  constructed  in Theorem \ref{theorem:theo1consSij} have important properties.
In Theorem \ref{theorem:theo2Aijnorm}, we show that the modifications applied to the original matrix $A$ of the inconsistent system to obtain the matrices $A^{(\vec{i}, \vec{j})}$ of the consistent systems are minimal.
This  result allows us to compute the distance $\mathring{\Delta}_{p}$, see (\ref{eq:deltaAp}),  in a simpler way, see Corollary \ref{corollary:normminJh}. Furthermore, when using the $L_\infty$ norm, i.e., $p = \infty$, we show that the distance $\mathring{\Delta}_{\infty}$, see (\ref{eq:deltaAp}), can be computed using an explicit analytical formula, see Corollary \ref{corollary:formulaforLinftydist}, and we provide a  method to construct 
a non-empty and finite subset of matrices  $A^{(\vec{i}, \vec{j})}$ 
  whose distance to the matrix  $A$ of the inconsistent system is equal to the scalar $\mathring{\Delta}_{\infty}$, see (\ref{eq:lAinfty}). Finally, we translate the tools and results obtained for max-min systems to the corresponding tools and results for min-max systems, see Subsection \ref{subsec:minmaxsysresults}.

The article is structured as follows. In Section \ref{sec:bg}, we give the necessary background for studying inconsistent max-min systems. In Section \ref{sec:minimallymod}, we present our problem formally and we introduce the construction of the auxiliary matrices, and the iterative construction (\ref{eq:constructionAijvec}) based on them.  In Section \ref{sec:mainresults}, we show our main results:   Theorem \ref{theorem:theo1consSij}, Theorem \ref{theorem:theo2Aijnorm}, Corollary \ref{corollary:normminJh} and Corollary \ref{corollary:formulaforLinftydist}. In Section \ref{sec:minmaxres}, we translate our results with max-min systems to min-max systems. Finally, we conclude with some perspectives.

\section{Background}
\label{sec:bg}

In this section, we give the necessary background for studying max-min systems. 
We begin by giving some notations. We then remind the solving of max-min systems and some of the results of \cite{baaj2024maxmin} for handling the inconsistency of max-min systems. 

\subsection{Notations and matrix products}

\noindent We use the following notations:
\begin{notation}\label{not:basic}
$[0,1]^{n\times m}$ denotes the set of matrices of size $(n,m)$ i.e., $n$ rows and $m$ columns, whose components are in $[0,1]$. In particular:
\begin{itemize}
    \item $[0,1]^{n\times 1}$ denotes the set of column vectors of $n$ components,
    \item $[0,1]^{1\times m}$ denotes  the set of row matrices of $m$ components. 
\end{itemize}

\noindent In the set $[0,1]^{n\times m}$, we use the order relation $\leq$ defined by:
\[ A \leq B \quad \text{iff we have} \quad  a_{ij} \leq b_{ij} \quad \text{ for all } \quad 1 \leq i \leq n, 1 \leq j \leq m,    \]
\noindent where $A=[a_{ij}]_{1 \leq i \leq n, 1 \leq j \leq m}$ and $B=[b_{ij}]_{1 \leq i \leq n, 1 \leq j \leq m}$. \\ 
\end{notation}

\noindent The two matrix products $\Box_{\min}^{\max}$,  and $\Box_{\rightarrow_G}^{\min}$  used in this article are defined by modifying the usual matrix product: 
\begin{itemize}
    \item the matrix product $\Box_{\min}^{\max}$ is obtained by taking the t-norm $\min$ as the product and the function $\max$ as the addition.
 
  \item the matrix product $\Box_{\rightarrow_G}^{\min}$ is obtained by taking the Gödel implication $\rightarrow_G$, which is defined by: \begin{equation}\label{eq:godelimp}
      x \rightarrow_G y = \begin{cases}
      1 & \text{ if } x \leq y\\
      y & \text{otherwise.}
  \end{cases}  \text{ in } [0,1]
  \end{equation}  as the product and the function $\min$ as the addition. 
  
\end{itemize}

\subsection{Systems of \texorpdfstring{max-min}{max-min} fuzzy relational equations}

\noindent Let $A=[a_{ij}] \in [0,1]^{n\times m}$ be a matrix of size $(n,m)$ and $b=[b_{i}] \in [0,1]^{n\times 1}$ be a vector of $n$ components. The system of max-min fuzzy relational equations  associated with $(A,b)$  is denoted by:
\begin{equation}\label{eq:maxminsys}
    (S): A \Box_{\min}^{\max} x = b,
\end{equation}
\noindent where $x = [x_j]_{1 \leq j \leq m} \in [0,1]^{m\times 1}$ is  an unknown vector of $m$ components. By definition of the matrix product $\Box_{\min}^{\max}$, the system $(S)$ is equivalent to the following set of $n$ max-min equations:
\begin{equation*}
    \max_{1 \leq j \leq m} \min(a_{ij},x_j) = b_i,\, \forall i \in \{1, 2,\dots, n\}.
\end{equation*}

\noindent To check if the system $(S)$, see (\ref{eq:maxminsys}), is consistent, we compute the following vector: \begin{equation}\label{eq:egretestsol}
e = A^t \Box_{\rightarrow_G}^{\min} b,\end{equation}
\noindent where $A^t$ is the transpose of $A$  and we use the matrix product $\Box_{\rightarrow_G}^{\min}$ (Notation \ref{not:basic}). The components of the vector $e = [e_j]_{1\leq j \leq m}$ are given by:
\begin{equation*}
    e_j = \min_{1 \leq i \leq n} (a_{ij} \rightarrow_G b_i), \forall j \in \{1,2,\dots,m\}.
\end{equation*}
\noindent The vector $e$ is the potential greatest solution of the system $(S)$. 

\noindent Thanks to Sanchez's seminal work \cite{sanchez1976resolution}  on the solving of systems max-min fuzzy relational equations, we have the following equivalence:
\begin{equation}\label{eq:consiste1}
    \text{The system } (S) \text{ is consistent}\Longleftrightarrow A \Box_{\min}^{\max} e = b. 
\end{equation}
\noindent  The set of solutions of the system $(S)$ is denoted by:
    \begin{equation}\label{eq:setofsolutionsS}
{\cal  S  } = {\cal  S  }(A, b) = \{ v \in [0, 1]^{m \times 1} \,\mid \,  A \, \Box_{\min}^{\max} v  = b \}. 
\end{equation}
\noindent The structure of the solution set was described by Sanchez in \cite{sanchez1976resolution} for systems of max-min fuzzy relational equations. He proved that if  ${\cal  S  }$ is non-empty i.e., the system $(S)$ is consistent, the structure of ${\cal  S  }$ is given by the vector $e = A^t \Box_{\rightarrow_G}^{\min} b$, see (\ref{eq:egretestsol}), which is the greatest solution of the system $(S)$ and a finite number of minimal solutions.

\begin{example}
    Let a max-min system $A \Box_{\min}^{\max} x = b$ be defined by $A = \begin{bmatrix}
    0.79& 0.29& 0.28\\
    0.76& 0.79& 0.18\\ 
    0.53& 0.21& 0.11
\end{bmatrix}$ and $b = \begin{bmatrix}
    0.35\\ 0.6\\ 0.35
\end{bmatrix}$.

We compute the potential greatest solution of the  system:
\begin{equation*}
    e = A^t \Box_{\rightarrow_G}^{\min} b = \begin{bmatrix}
        \min( 0.79 \rightarrow_G 0.35, 0.76 \rightarrow_G 0.6, 0.53 \rightarrow_G 0.35) \\
         \min( 0.29 \rightarrow_G 0.35, 0.79 \rightarrow_G 0.6, 0.21 \rightarrow_G 0.35) \\
          \min( 0.28 \rightarrow_G 0.35, 0.18 \rightarrow_G 0.6, 0.11 \rightarrow_G 0.35) \\
    \end{bmatrix} = \begin{bmatrix}
        0.35\\ 0.6\\ 1
    \end{bmatrix}.
\end{equation*}
The equation  system $A \Box_{\min}^{\max} x = b$ is consistent because: $$A \Box_{\min}^{\max} e = \begin{bmatrix}
    \max(\min(0.79, 0.35), \min(0.29, 0.6), \min(0.28, 1))\\ 
    \max(\min(0.76, 0.35), \min(0.79, 0.6), \min(0.18, 1))\\
    \max(\min(0.53, 0.35), \min(0.21, 0.6), \min(0.11, 1))\\
\end{bmatrix} =\begin{bmatrix}
    0.35 \\ 0.6 \\ 0.35
\end{bmatrix} = b.$$ So the vector $e$ is the greatest solution of the   system. 
\end{example}

\subsection{Handling inconsistent systems of \texorpdfstring{max-min}{max-min} fuzzy relational equations}

Recently, the author of \cite{baaj2024maxmin} studied how to handle inconsistent systems of max-min  fuzzy relational equations. The tools introduced by \cite{baaj2024maxmin} follow the strategy proposed by Pedrycz in \cite{pedrycz1990inverse}, which is as follows: given an inconsistent system, we perturb, as slighly as possible, the components of the right-hand side vector of the inconsistent system in order to obtain a consistent system. 

To remind the tools of the author \cite{baaj2024maxmin} we reuse some of his notations:
\begin{notation}
    For $x,y,z,u,\delta \in [0,1]$, we use the following notations:
\begin{itemize}
    \item $x^+ = \max(x,0)$,
    \item $\overline{z}(\delta) = \min(z+\delta,1)$, 
    \item $\underline{z}(\delta) = \max(z-\delta,0) = (z-\delta)^+$.
\end{itemize}
\end{notation}
We rely on the following function, which was introduced in  \cite{baaj2024maxmin}:
\begin{equation}\label{eq:sigmaG}
    \sigma_G(x, y, z) = \min( \frac{(x - z)^+}{2},(y - z)^+)  \quad \text{ where } \quad x,y,z \in [0,1].
\end{equation}

Let $T = \begin{bmatrix}
    t_{ij}
\end{bmatrix}_{1 \leq i \leq n, 1 \leq j \leq m} \in [0,1]^{n\times m}$ be a matrix. We rely on  the following function defined for all $i\in\ \{1,2,\cdots,n\}$, which appears in \cite{baaj2024maxmin} and uses the matrix $T$ and the right-hand side vector $b=[b_{i}] \in [0,1]^{n\times 1}$ of the system $T \Box_{\min}^{\max} x = b$:
\begin{equation}\label{eq:deltaTi}
    [0 , 1]^{n \times m} \rightarrow [0 , 1] : T \mapsto \delta^T_i = \min_{1 \leq j \leq m} \delta^T{(i,j)}
\end{equation}
where:
\begin{equation}\label{eq:deltaTij}
\delta^T{(i,j)} = \max[ (b_i - t_{ij})^+,  \max_{1 \leq k \leq n}\,  \sigma_G\,(b_i, t_{kj}, b_k)] \quad \quad \text{(see (\ref{eq:sigmaG}) for the definition of }\sigma_G)
\end{equation}
 Since the operations involved (minimum, maximum, and taking the positive part) are continuous on the compact set $[0,1]^{n \times m}$, we get the following result:
\begin{lemma}\label{lemma:cont}
For all $i\in\ \{1,2,\cdots,n\}$,  the function   
$[0 , 1]^{n \times m} \rightarrow [0 , 1] : T \mapsto \delta^T_i$ is continuous.
\end{lemma} 
\qed

Given the matrix $A$ of an inconsistent max-min system $(S): A \Box_{\min}^{\max} x = b$, see (\ref{eq:maxminsys}), the author of \cite{baaj2024maxmin} defines:
\begin{definition}\label{def:setC}
The set formed by the right-hand side vectors of the consistent systems defined with the same matrix $A$ of the system $(S): A \Box_{\min}^{\max} x = b$ is:
\begin{equation}\label{eq:anymaxminRHS}
    \mathcal{C} = \{ c \in [0,1]^{n \times 1} \mid \text{ the system } A \Box_{\min}^{\max} x = c \text{ is consistent } \}. 
\end{equation}
\end{definition}
\begin{definition}\label{def:chebdist}
The  Chebyshev distance (using L-infinity norm) associated with the inconsistent max-min system $(S): A \Box_{\min}^{\max} x = b$ (see (\ref{eq:maxminsys})) is:
\begin{equation}\label{eq:cheb:def}
    \Delta = \min_{c \in \mathcal{C}} \Vert b -c \Vert_{\infty} \text{ where } \Vert b - c \Vert = \max_{1 \leq i \leq n} \mid b_i - c_i \mid.
\end{equation}
\end{definition}

The author of \cite{baaj2024maxmin} introduced an analytical formula for computing  the Chebyshev distance $\Delta$ (Definition \ref{def:chebdist}) associated to the max-min  system $(S): A \Box_{\min}^{\max} x = b$, which relies on $\delta^A_i$, see (\ref{eq:deltaTi}): 
\begin{equation}\label{eq:Delta}
    \Delta = \Delta(A,b)= \max_{1 \leq i \leq n} \delta^A_i 
\end{equation}

The following results were proven in \cite{baaj2024maxmin}:
\begin{equation}\label{eq:equivDelta}
    \Delta = 0 \Longleftrightarrow \text{ the  system } (S): A \Box_{\min}^{\max} x = b \text { is consistent. } 
\end{equation}
We define a subset of $\mathcal{C}$, see (\ref{eq:anymaxminRHS}):
\begin{definition}\label{def:chebapproxset}
The set of  Chebyshev approximations of the right-hand side vector $b$ of the  system $(S): A \Box_{\min}^{\max} x = b$ (see (\ref{eq:maxminsys})) is defined using the Chebyshev distance $\Delta$ associated with the right-hand side vector $b$ of $(S)$ (see Definition \ref{def:chebdist}): 
\begin{equation}
    \mathcal{C}_b = \{ c \in [0,1]^{n \times 1} \mid \text{ the system } A \Box_{\min}^{\max} x = c \text{ is consistent and } \Vert b - c \Vert = \Delta \} \subseteq \mathcal{C}.
\end{equation}
\end{definition}

The author of  \cite{baaj2024maxmin} showed that we can directly obtain one vector in $\mathcal{C}_b$: the greatest Chebyshev approximation of the right-hand side vector $b$:
\begin{equation}\label{eq:greatestcheb}
    \widehat{b} = A \Box_{\min}^{\max} (A^t \Box_{\rightarrow_G}^{\min} \overline{b}(\Delta)) \in \mathcal{C}_b,
\end{equation}
where  $\overline{b}(\Delta) = \begin{bmatrix}
    \min(b_i + \Delta , 1)
\end{bmatrix}_{1 \leq i \leq n}$. So, we have $\Vert b -\widehat{b} \Vert = \Delta$ and the  system $A \Box_{\min}^{\max} x =  \widehat{b}$ is consistent. Obtaining the minimal Chebyshev approximations of $b$ is possible with a method developed in  \cite{baaj2024maxmin}, but minimal  approximations are more expensive to obtain in terms of computational cost than the greatest Chebyshev approximation $\widehat{b}$. 

The author of  \cite{baaj2024maxmin} defines:

\begin{definition}
An approximate solution of the system $(S): A \Box_{\min}^{\max} x = b$ is a vector $x^\ast$ such that the distance between the vector $c = A \Box_{\min}^{\max} x^\ast$ and the vector $b$ is equal to $\Delta$, i.e.,  we have $\Vert b - c \Vert = \Delta$. Thus, the vector $x^\ast$ is a solution of a consistent system $A \Box_{\min}^{\max} x = c$ defined with the matrix $A$ and the vector $c$, which is a Chebyshev approximation of $b$ (Definition \ref{def:chebapproxset}), i.e., we have $c \in \mathcal{C}_b$. 
\end{definition}

In \cite{baaj2024maxmin}, the author   showed that we can directly obtain the greatest approximate solution of the system $(S)$, which is the following vector:
\begin{equation}\label{eq:greatestapproxsol}
    \eta = A^t \Box_{\rightarrow_G}^{\min} \overline{b}(\Delta).
\end{equation}
Of course, the greatest approximate solution $\eta$ of the system $(S)$ is the greatest solution of the consistent system $A \Box_{\min}^{\max} x =  \widehat{b}$. In \cite{baaj2024maxmin}, the author describes how to obtain  a set of minimal approximate solutions.

\begin{example}\label{ex:incons}
    Let a max-min system $A \Box_{\min}^{\max} x = b$ be defined by: 
\begin{equation}\label{eq:ex:Ab}
    A = \begin{bmatrix}
    0.04& 0.73& 0.50\\
    0.33& 0.35& 0.94\\ 
    0.55& 0.90&  0.35
\end{bmatrix} \quad \text{ and } \quad b = \begin{bmatrix}
    0.30\\  0.42\\ 0.76
\end{bmatrix}
\end{equation}

We compute the Chebyshev distance $\Delta$, see (\ref{eq:Delta}), associated with the right-hand side vector $b$ of the  max-min system $A \Box_{\min}^{\max} x = b$:
\begin{equation*}
    \Delta = \max_{1 \leq i \leq 3} \delta^A_i
\end{equation*}
where:
\begin{align*}
    \delta^A_1 &= \min(\delta(1,1), \delta(1,2), \delta(1,3)) = \min(0.26, 0, 0) = 0, \\
    \delta^A_2 &= \min(\delta(2,1), \delta(2,2), \delta(2,3)) = \min(0.09, 0.07, 0.06) = 0.06, \\
    \delta^A_3 &= \min(\delta(3,1), \delta(3,2), \delta(3,3)) = \min(0.21, 0.23, 0.41) = 0.21. \\
\end{align*}
Therefore, we obtain $\Delta = 0.21$, so,  from (\ref{eq:equivDelta}), the  system $A \Box_{\min}^{\max} x = b$ is inconsistent. Let us compute $\overline{b}(\Delta) = \begin{bmatrix}0.51 \\ 0.63 \\ 0.97\end{bmatrix}$ and the greatest Chebyshev approximation of $b$, see (\ref{eq:greatestcheb}):
\begin{equation*}
     \widehat{b} = A \Box_{\min}^{\max} (A^t \Box_{\rightarrow_G}^{\min} \overline{b}(\Delta)) = \begin{bmatrix}0.51\\ 0.63\\ 0.55\end{bmatrix}.
\end{equation*}
One can check that $\Vert b -  \widehat{b} \Vert = \Delta$ and that the  system  $A \Box_{\min}^{\max} x = \widehat{b}$ is consistent.

We compute the greatest approximate solution of the inconsistent  system $A \Box_{\min}^{\max} x = b$, see (\ref{eq:greatestapproxsol}):
\begin{equation*}
 \eta = A^t \Box_{\rightarrow_G}^{\min} \overline{b}(\Delta) = \begin{bmatrix}1.0 \\ 0.51\\ 0.63\end{bmatrix}.
\end{equation*}
One can check that $\eta$ is the greatest solution of the consistent  system  $A \Box_{\min}^{\max} x =  \widehat{b}$.
\end{example}

\section{On minimally modifying the matrix of an inconsistent max-min system}
\label{sec:minimallymod}

In this section, we present our problem formally: how to minimally modify the matrix $A$ of size $(n,m)$ of an inconsistent max-min  system $(S): A \Box_{\min}^{\max} x = b$, see (\ref{eq:maxminsys}), in order to obtain a consistent  system, while keeping the right-hand side vector $b$ unchanged. 

For this purpose, we introduce the following set, which is formed by  the matrices of size $(n,m)$ governing the  consistent max-min  systems which use the same right-hand side vector: the right-hand side vector $b$ of the system $(S)$:

\begin{equation}\label{eq:setT}
    {\cal T} = \{T\in [0 , 1]^{n \times m}\mid \text{the system }  T \Box_{\min}^{\max} x = b  \text{ is consistent}\}.
\end{equation}

\noindent
It is easy to see that the particular matrix $T = [t_{kl}]_{1 \leq k \leq n , 1\leq l \leq m}$
whose coefficients are $t_{kl}:= b_k$, belongs to the set ${\cal T}$.  Therefore ${\cal T}$   is a non-empty set.

\noindent For each $p \in \{1,2,\infty\}$, we define the distance from the fixed matrix $A$ to the set $\mathcal{T}$ as:
\begin{equation}\label{eq:deltaAp}
    \mathring{\Delta}_{p} := \inf_{T \in \mathcal{T}} \|T - A\|_p.
\end{equation}
Here, $\|\cdot\|_p$ denotes the matrix norm with exponent $p$, namely, for a matrix $X=[x_{ij}]\in\mathbb{R}^{n\times m}$:
\begin{equation}\label{eq:entrywise_p_norm}
  \|X\|_p \;=\;
  \begin{cases}
    \displaystyle\sum_{i=1}^n\sum_{j=1}^m \bigl|x_{ij}\bigr|,
      & p=1,\\[1em]
    \displaystyle\Bigl(\sum_{i=1}^n\sum_{j=1}^m \bigl|x_{ij}\bigr|^2\Bigr)^{\!1/2},
      & p=2,\\[1em]
    \displaystyle\max_{1\le i\le n,\;1\le j\le m}\bigl|x_{ij}\bigr|,
      & p=\infty.
  \end{cases}
\end{equation}

Our aim is to compute  matrices in the set $\mathcal{T}$, i.e., matrices governing  consistent max-min systems whose right-hand side vector is $b$, which are the closest to the matrix $A$ of the system $(S): A \Box_{\min}^{\max} x = b$ according to the distance $\mathring{\Delta}_{p}$.  We also study how to compute the distance $\mathring{\Delta}_{p}$ with respect to $p$  according to the coefficients of $A$ and $b$.

In this section, we   present some preliminary results and tools that will be useful for constructing matrices in the set $\mathcal{T}$ in subsequent sections. 
We show that the matrices of the set $\mathcal{T}$ defined in (\ref{eq:setT}) dominate in a certain sense the auxiliary matrices noted $A^{(i,j)}$ associated with pairs of indices $(i,j) \in \{1,2,\dots,n\} \times \{1,2,\dots,m\}$, that we introduce in Definition \ref{def:matAij}. We   establish useful  properties  of these  auxiliary matrices $A^{(i,j)}$ for our problem.

\begin{restatable}{lemma}{lemmaDistanceDeltaApmin}
\label{lemma:DistanceDeltaApmin}

The formula (\ref{eq:deltaAp}) can be rewritten as: $$\mathring{\Delta}_{p}  = \min\limits_{T\in {\cal T}}\, \Vert T - A\Vert_p.$$
\end{restatable}
The proof is given in Subsection \ref{sec:proof:minimallymod}.

We deduce the following equivalence, which is analog to (\ref{eq:equivDelta}):
\begin{corollary}
    \begin{equation}
        \mathring{\Delta}_{p} = 0 \Longleftrightarrow \text{ the system $(S)$ is consistent.}
    \end{equation}
\end{corollary}
\qed

\subsection{Constructing the auxiliary matrices \texorpdfstring{$A^{(i,j)}$}{A(i,j)} from the matrix \texorpdfstring{$A$}{A} of the inconsistent system \texorpdfstring{$(S): A \Box_{\min}^{\max} x = b$}{(S): A maxmin x = b}}
\label{sec:auxiliarymatrix}

We introduce the following notations associated with the system  $(S): A \Box_{\min}^{\max} x = b$, see (\ref{eq:maxminsys}). 
\begin{notation}\label{notation:basicNotForS}
 Let us use the following set of indices:
\begin{equation}
    I = \{1,2,\cdots,n\} \quad \text{and} \quad J = \{1,2,\dots,m\}.
\end{equation}
For all $(i,j) \in I \times J$, a given matrix $T = \begin{bmatrix}
    t_{ij}
\end{bmatrix}_{1 \leq i \leq n, 1 \leq j \leq m} \in [0,1]^{n\times m}$ and the right-hand side vector $b$ of the system $(S)$, we define the following set using the function $\sigma_G$, see (\ref{eq:sigmaG}):
\begin{equation}\label{eq:UDij}
U^T_{ij} := \{k\in I  \mid  \sigma_G(b_i , t_{kj} , b_k) > 0\}.\end{equation}
Using (\ref{eq:sigmaG}), we get the following  equivalent definition of the set $U^T_{ij}$:
\begin{equation}\label{eq:UDij1}
U^T_{ij}  = \{k\in I  \mid  \theta(i,k) \cdot (t_{kj} - b_k)^+ > 0\}\quad \text{where} \quad
\theta: I \times I \mapsto \{0,1\}: \theta(i,k) = \begin{cases}
1 & \text{ if } b_i > b_k\\
0 & \text{ otherwise}
\end{cases}.\end{equation}

In particular, for any pair of indices $(i,j) \in I \times J$, we can compute the set $U^A_{ij}$ using the matrix $A$ of the system $(S)$.\\
From the formula of the Chebyshev distance $\Delta = \max_{1 \leq i \leq n} \delta^A_i$, see (\ref{eq:Delta}), associated with the right-hand side vector $b$ of the system $(S)$, we set:
\begin{equation}\label{eq:nconsninc}
    N_{\text{cons}} := \{ i \in I \mid \delta_i^A = 0 \} \quad \text{ and } \quad N_{\text{inc}} := \overline{N_{\text{cons}}} = \{ i \in I \mid \delta_i^A > 0 \}.
\end{equation}
The two disjoint  sets $N_{\text{cons}}$ and $N_{\text{inc}}$ form a partition of the set $I$ associated with the pair $(A , b)$. We know by   (\ref{eq:equivDelta}) that the equality $N_{\text{cons}} = I$ is equivalent to the consistency of the system  $(S): A \Box_{\min}^{\max} x = b$.

\end{notation}

\begin{definition}\label{def:matAij}
For any fixed pair of indices $(i,j) \in I \times J$, we define the matrix 
$A^{(i,j)} := \begin{bmatrix}
    a^{(i,j)}_{kl}
\end{bmatrix}_{1 \le k \le n,\, 1 \le l \le m} \in [0,1]^{n \times m}$, 
whose coefficients are defined using the matrix  $A$,  the right-hand side vector $b$ of the inconsistent system $(S): A \Box_{\min}^{\max} x = b$, and the set $U^A_{ij}$, see (\ref{eq:UDij}), by:

\begin{subequations}\label{eq:Aijklgroup}
\begin{equation}\label{eq:Aijkl1}
a^{(i,j)}_{ij} := \begin{cases}
    b_i \quad & \text{if } \quad      b_i > a_{ij}    \\
 a_{ij}   \quad & \text{if } \quad b_i \leq a_{ij}
\end{cases} \quad  \text{and for } \quad \forall k \in I\backslash\{i\}, \quad         a^{(i , j)}_{kj} := \begin{cases}
b_k  \quad &  \text{if } \quad k \in U^A_{ij}  \\
 a_{kj}   \quad & \text{if } \quad k \notin U^A_{ij}  
 \end{cases}.     
\end{equation}
\begin{equation}\label{eq:Aijkl2}
\forall k \in I \, \forall l\in J \backslash\{j\}, 
  \quad  a^{(i , j)}_{kl} := a_{kl}.   
\end{equation}
\end{subequations}
\end{definition}
The matrix $A^{(i,j)}$ is derived from the matrix $A$ by \textit{only modifying some entries in the column with index $j$}. Specifically, we have:
\begin{equation}\label{eq:aijklisbkorakl}
\forall (k , l)\in I \times J, \quad a^{(i , j)}_{kl} \in \{b_k , a_{kl}\}.
\end{equation}

We begin by establishing some useful properties of the scalars $\delta^{A^{(i , j)}}(s ,l)$, which is computed using (\ref{eq:deltaTij}) and associated with the matrix $A^{(i , j)}$ and the right-hand side vector $b$:

\begin{restatable}{lemma}{lemmadeltaaijijequalzero}
\label{lemma:deltaaijijequalzero}
Let $(i,j) \in I \times J$ be a fixed pair of indices. Then, we have:
\begin{enumerate}
  \item $\delta^{A^{(i , j)}}(i , j) = 0$.
  \item For all  $1\leq s \leq n$, we have: $$\delta^{A}(s , j) = 0 \Longrightarrow \delta^{A^{(i , j)}}(s , j) = 0.$$ 
    \item For all $1 \leq s \leq n$ and   
    $1 \leq l \not=j \leq m$, we have: $$\delta^{A^{(i , j)}}(s , l) = \delta^A(s , l).$$
    \item $\Vert A^{(i , j)} - A \Vert_\infty = 
    \max\Bigl[(b_i - a_{ij})^+, \, \max_{k \in I, k\neq i}\Bigl(\theta(i,k) \cdot (a_{kj} - b_k)^+\Bigr)\Bigr]$ 
    
    where the function $\theta$ is defined in (\ref{eq:UDij1}). 
\end{enumerate}
\end{restatable}

\noindent Any matrix $T = [t_{kl}]_{1 \leq k \leq n, 1 \leq l \leq m}$ satisfying $\delta^T(i,j)=0$ for a given pair of indices $(i , j)\in I \times J$ dominates the auxiliary matrix 
$A^{(i,j)}$ in the following sense: the distance $|t_{kl}-a_{kl}|$ between  a coefficient $t_{kl}$ of the matrix $T$ and the corresponding coefficient $a_{kl}$ of the original matrix $A = [a_{kl}]$ is always at least as large as the distance  $|a^{(i,j)}_{kl}-a_{kl}|$ between the coefficient $a^{(i,j)}_{kl}$ of the matrix $A^{(i,j)}$ and $a_{kl}$: 
\begin{restatable}{lemma}{lemmamatrixTwithdetaTijzero}
\label{lemma:matrixTwithdetaTijzero}
Let 
$T = [t_{kl}]_{1 \leq k \leq n , 1 \leq l \leq m}$ be a matrix such that  $\delta^T(i , j) = 0$. 
Then,   we have:
\begin{enumerate}
    \item   $\mid a^{(i , j)}_{kl} - a_{kl} \mid \leq \mid t_{kl} -a_{kl}\mid$ 
    for all pairs $(k , l)  \in I \times J$,
    \item $\Vert A^{(i , j)}  - A\Vert_p \leq  \Vert T  - A\Vert_p $ for $p\in\{1 , 2 , \infty\}$.
\end{enumerate}
\end{restatable}

\noindent
In the following lemma, we show that if we first apply the transformation 
$A \mapsto A^{(i,j)}$ with parameters $(i_1,j_1)$ to the fixed matrix $A$, and then apply to the resulting matrix the same transformation with parameters $(i_2,j_2)$ {\it  with $i_1 \neq i_2$}, the final result is the same as when the order of these two transformations is reversed.
\begin{restatable}{lemma}{lemmacommutativeAij}
\label{lemma:commutativeAij}
Let $i_1,i_2 \in I$ with $i_1 \neq i_2$ and $j_1, j_2 \in J$. Define
$$
C := A^{(i_1,j_1)}, \quad D := C^{(i_2,j_2)}, \quad E := A^{(i_2,j_2)}, \quad F := E^{(i_1,j_1)}.
$$
Then,
$$
D = F.
$$
\end{restatable}

The proofs of Lemma \ref{lemma:deltaaijijequalzero}, \ref{lemma:matrixTwithdetaTijzero} and \ref{lemma:commutativeAij} are written in Subsection \ref{sec:proofs:auxiliarymatrix}.

\begin{example}\label{ex:matAij}
(continued)
We reuse the inconsistent max-min system $(S) : A \Box_{\min}^{\max} x = b$ of Example \ref{ex:incons} where:
\[
A = \begin{bmatrix}
0.04 & 0.73 & 0.5 \\
0.33 & 0.35 & 0.94 \\
0.55 & 0.9  & 0.35
\end{bmatrix}, \quad
b = \begin{bmatrix}
0.3\\[0.5em]
0.42\\[0.5em]
0.76
\end{bmatrix}.
\]
We computed:
\[
\Delta = \max_{1 \leq i \leq n} \delta^A_i = 0.21, \quad \text{ where }  \delta^A_1 = 0, \delta^A_2 = 0.06, \text{ and } \delta^A_3 = 0.21.
\]
Therefore, we have:
\[
N_{\text{cons}} = \{1\}, \quad N_{\text{inc}} = \{2, 3\}.
\]

We construct the matrices $A^{(i,j)}$, see Definition \ref{def:matAij}.

Case $(i,j)=(1,1)$:
\[
U^A_{1,1} = \emptyset.
\]
Modified matrix:
\[
A^{(1,1)} = \begin{bmatrix}
0.3  & 0.73 & 0.5 \\
0.33 & 0.35 & 0.94 \\
0.55 & 0.9  & 0.35
\end{bmatrix},
\]
with the $(1,1)$-entry changed from 0.04 to 0.3, and we have 
$\delta^{A^{(1,1)}}(1,1)=0$.

Case $(i,j)=(1,2)$:
\[
U^A_{1,2} = \emptyset.
\]
Matrix remains:
\[
A^{(1,2)} = \begin{bmatrix}
0.04 & 0.73 & 0.5 \\
0.33 & 0.35 & 0.94 \\
0.55 & 0.9  & 0.35
\end{bmatrix},
\]
with no changes, and we have $\delta^A(1,2) = \delta^{A^{(1,2)}}(1,2)=0$.

Case $(i,j)=(1,3)$:
\[
U^A_{1,3} = \emptyset.
\]
Matrix remains:
\[
A^{(1,3)} = \begin{bmatrix}
0.04 & 0.73 & 0.5 \\
0.33 & 0.35 & 0.94 \\
0.55 & 0.9  & 0.35
\end{bmatrix},
\]
with no changes, and we have $\delta^A(1,3) = \delta^{A^{(1,3)}}(1,3)=0.$

Case $(i,j)=(2,1)$:
\[
U^A_{2,1} = \emptyset.
\]
Modified matrix:
\[
A^{(2,1)} = \begin{bmatrix}
0.04 & 0.73 & 0.5 \\
0.42 & 0.35 & 0.94 \\
0.55 & 0.9  & 0.35
\end{bmatrix},
\]
with the $(2,1)$-entry changed from 0.33 to 0.42, and we have
$
\delta^{A^{(2,1)}}(2,1)=0.
$

Case $(i,j)=(2,2)$:
\[
U^A_{2,2} = \{1\},\quad \text{since } \sigma_G(b_2, a_{1,2}, b_1)=\sigma_G(0.42,\,0.73,\,0.3)=0.06>0.
\]
Modified matrix:
\[
A^{(2,2)} = \begin{bmatrix}
0.04 & 0.3  & 0.5 \\
0.33 & 0.42 & 0.94 \\
0.55 & 0.9  & 0.35
\end{bmatrix},
\]
with the (2,2)-entry changed from 0.35 to 0.42 and the (1,2)-entry changed from 0.73 to 0.3, and we have 
$
\delta^{A^{(2,2)}}(2,2)=0.
$

Case $(i,j)=(2,3)$:
\[
U^A_{2,3} = \{1\},\quad \text{since } \sigma_G(b_2, a_{1,3}, b_1)=\sigma_G(0.42,\,0.5,\,0.3)=0.06>0.
\]
Modified matrix:
\[
A^{(2,3)} = \begin{bmatrix}
0.04 & 0.73 & 0.3 \\
0.33 & 0.35 & 0.94 \\
0.55 & 0.9  & 0.35
\end{bmatrix},
\]
with the (1,3)-entry changed from 0.5 to 0.3, and we have $\delta^{A^{(2,3)}}(2,3)=0$.

Case $(i,j)=(3,1)$:
\[
U^A_{3,1} = \emptyset.
\]
Modified matrix:
\[
A^{(3,1)} = \begin{bmatrix}
0.04 & 0.73 & 0.5 \\
0.33 & 0.35 & 0.94 \\
0.76 & 0.9  & 0.35
\end{bmatrix},
\]
with the (3,1)-entry changed from 0.55 to 0.76, and we have 
$\delta^{A^{(3,1)}}(3,1)=0$.

Case $(i,j)=(3,2)$:
\[
U^A_{3,2} = \{1\},\quad \text{since } \sigma_G(b_3, a_{1,2}, b_1)=\sigma_G(0.76,\,0.73,\,0.3)=0.23>0.
\]
Modified matrix:
\[
A^{(3,2)} = \begin{bmatrix}
0.04 & 0.3  & 0.5 \\
0.33 & 0.35 & 0.94 \\
0.55 & 0.9  & 0.35
\end{bmatrix},
\]
with the (1,2)-entry changed from 0.73 to 0.3, and we have 
$\delta^{A^{(3,2)}}(3,2)=0$.

Case $(i,j)=(3,3)$:
\[
U^A_{3,3} = \{1,2\},\quad \text{since } \sigma_G(b_3, a_{1,3}, b_1)=\sigma_G(0.76,\,0.5,\,0.3)=0.2>0 \quad (k=1)
\]
and
\[
\sigma_G(b_3, a_{2,3}, b_2)=\sigma_G(0.76,\,0.94,\,0.42)=0.17>0 \quad (k=2).
\]
Modified matrix:
\[
A^{(3,3)} = \begin{bmatrix}
0.04 & 0.73 & 0.3 \\
0.33 & 0.35 & 0.42 \\
0.55 & 0.9  & 0.76
\end{bmatrix},
\]
with the (3,3)-entry changed from 0.35 to 0.76, the (1,3)-entry changed from 0.5 to 0.3, and the (2,3)-entry changed from 0.94 to 0.42, and we have 
$\delta^{A^{(3,3)}}(3,3)=0$.

We now illustrate Lemma \ref{lemma:commutativeAij}. For instance, let $i_1 = 1$, $j_1 = 1$, $i_2 = 2$, and $j_2 = 1$. 
We have $C = A^{(1,1)}$ and $E = A^{(2,1)}$. We compute:

\[
D = C^{(2,1)} =  \begin{bmatrix}
0.3 & 0.73 & 0.5 \\
0.42 & 0.35 & 0.94 \\
0.55 & 0.9  & 0.35
\end{bmatrix}.
\]
where the entry at coordinate $(2,1)$ is changed from $0.33$ in $C$ to $0.42$ in $D$. Similarly,
\[
F = E^{(1,1)} =  \begin{bmatrix}
0.3 & 0.73 & 0.5 \\
0.42 & 0.35 & 0.94 \\
0.55 & 0.9  & 0.35
\end{bmatrix}.
\]
where the entry at coordinate $(1,1)$ is changed from $0.04$ in $E$ to $0.3$ in $F$. Thus, we observe that
\[
D = F,
\]
which confirms the lemma.
\end{example}

\subsection{Constructing the matrices \texorpdfstring{$A^{(\vec{i},\vec{j})}$}{A(i,j)} from the matrix \texorpdfstring{$A$}{A} of the inconsistent system \texorpdfstring{$(S): A \Box_{\min}^{\max} x = b$}{(S): A maxmin x = b}}
\label{sec:constructionAijvec}

In this subsection, we generalize the construction of auxiliary matrices $A^{(i , j)}$ associated with pairs $(i , j)\in I \times J$ to pairs of index uplets $(\vec i , \vec j)\in I^h \times J^h$  where $1 \leq h \leq n$, see (\ref{eq:constructionAijvec}). In the next section, Theorem \ref{theorem:theo1consSij} will show that for specific pairs $(\vec i , \vec j)$, matrices $A^{(\vec{i},\vec{j})}$ obtained with the construction (\ref{eq:constructionAijvec})  belong to the set $\mathcal{T}$, see (\ref{eq:setT}), i.e., using these specific pairs $(\vec{i},\vec{j})$, the matrices $A^{(\vec{i},\vec{j})}$ are matrices governing consistent max-min systems whose right-hand side vector is the right-hand side vector $b$ of the system $(S)$.

\subsubsection{Iterative construction of the matrices \texorpdfstring{$A^{(\vec{i},\vec{j})}$}{A vec i vec j}}

Let $h$ be an integer satisfying $1 \leq h \leq n$ where 
 $n =\text{card}(I)$, where the set $I$ is defined in Notation \ref{notation:basicNotForS}, along with the set $J$.  
\begin{definition}\label{def:vecivecj}
For any  $h$-tuple  $\vec{i} = (i_1, i_2, \dots, i_h) \in I^h$ consisting of two-by-two distinct indices from the set $I$ (i.e., for all $u,v \in \{1,2,\dots,h\}$ with $u \neq v$, $i_u \neq i_v$) and a  $h$-tuple  $\vec{j} = (j_1, j_2, \dots, j_h) \in J^h$ of   indices of the set $J$, we associate  the  matrix  denoted $A^{(\vec{i},\vec{j})}$ which is constructed by successively applying the modifications indicated by the pairs $(i_1,j_1), (i_2,j_2), \dots, (i_h,j_h)$ (using Definition \ref{def:matAij}) as follows:
\begin{equation}\label{eq:constructionAijvec}
A^{(\vec{i} , \vec{j})} =  \begin{cases}
 A^{(i_1 , j_1)}    &   \text{if }       h = 1,\hfill     \\
 C^{(i_h , j_h)}   & \text{if } h > 1 \,\text{and} \, C := A^{(i_1 , j_1) \ast (i_2 , j_2) \ast \dots \ast (i_{h-1} , j_{h-1})}.
 \end{cases}
\end{equation}.    
\end{definition} 
In (\ref{eq:constructionAijvec}), the matrix $C$ is the intermediate matrix obtained by successively applying the modifications corresponding to the pairs $(i_1,j_1), (i_2,j_2), \dots, (i_{h-1},j_{h-1})$ to the original matrix $A$. Then, the final matrix $A^{(\vec{i},\vec{j})}$ is obtained by applying the modification $(i_h,j_h)$ to $C$.  This construction can also be written as
$$
A^{(\vec{i},\vec{j})} = (((A^{(i_1,j_1)})^{(i_2,j_2)})^{(i_3,j_3)} \cdots )^{(i_h,j_h)}.
$$

\subsubsection{Properties of the matrices \texorpdfstring{$A^{(\vec{i},\vec{j})}$}{A vec i vec j}}

By induction on $h$, one can verify the following properties:
\begin{restatable}{proposition}{propositionaijklbkorakl}
\label{proposition:aijklbkorakl}\mbox{}
\begin{enumerate}
    \item For every $(k,l) \in I \times J$,
    $$
    a^{(\vec{i},\vec{j})}_{kl} \in \{b_k,\, a_{kl}\}.
    $$
    \item For every $k \in I$ and for every $l \notin \{j_1,j_2,\dots,j_h\}$,
    $$
    a^{(\vec{i},\vec{j})}_{kl} = a_{kl}.
    $$
\end{enumerate}
These properties extend those  of the matrices $A^{(i,j)}$ given in   (\ref{eq:Aijkl1}), (\ref{eq:Aijkl2}) and     (\ref{eq:aijklisbkorakl}). They  ensure that the modifications introduced by $A^{(\vec{i},\vec{j})}$ affect only the columns indexed by $j_1, j_2, \dots, j_h$, while all other coefficients of the matrix $A$ remain unchanged. 
\end{restatable}

The following result generalizes Lemma \ref{lemma:commutativeAij} in the following sense:
\begin{restatable}{proposition}{propositioncommutativePi}
\label{proposition:commutativePi}
Let $\vec{i} = (i_1, i_2, \dots, i_h) \in I^h$ be an $h$-tuple of   two-by-two distinct indices of the set $I$ and $\vec{j} = (j_1, j_2, \dots, j_h) \in J^h$. For any permutation $\pi$ of the set $\{1, 2, \dots, h\}$ i.e., $\pi : \{1 , 2 , \dots , h\} \rightarrow  \{1 , 2 , \dots , h\} : \lambda \mapsto  i_\lambda$, define
\begin{equation}\label{eq:vecpi}
\vec{i} \triangleleft \pi = \bigl(i_{\pi(1)}, i_{\pi(2)}, \dots, i_{\pi(h)}\bigr)
\quad \text{and} \quad
\vec{j} \triangleleft \pi = \bigl(j_{\pi(1)}, j_{\pi(2)}, \dots, j_{\pi(h)}\bigr).    
\end{equation}

Then, the matrix obtained by successively modifying $A$ using the pairs $(i_1, j_1), (i_2, j_2), \dots, (i_h, j_h)$ is identical to the matrix obtained by applying these modifications in any permuted order; that is,
\begin{equation}\label{eq:mainavecipivecjpi}
    A^{(\vec{i},\vec{j})} = A^{(\vec{i}\triangleleft\pi,\, \vec{j}\triangleleft\pi)},
\end{equation}

or equivalently,
\[
A^{(i_1, j_1) \ast (i_2, j_2) \ast \cdots \ast (i_h, j_h)}
=
A^{(i_{\pi(1)}, j_{\pi(1)}) \ast (i_{\pi(2)}, j_{\pi(2)}) \ast \cdots \ast (i_{\pi(h)}, j_{\pi(h)})}.
\]
\end{restatable}

Proposition \ref{proposition:commutativePi}  allows us to generalize Lemma \ref{lemma:deltaaijijequalzero} and Lemma \ref{lemma:matrixTwithdetaTijzero} to the matrices $A^{(\vec i, \vec j)}$. We rely on the scalars  $\delta^{A^{(\vec i , \vec j)}}(s , l)$ (resp. $\delta^{A}(s , l)$), which are  associated with the matrix $A^{(\vec i , \vec j)}$ (resp.  $A$),  see  (\ref{eq:deltaTij}).

\begin{restatable}{proposition}{propositionmodOnegAijvec}
\label{proposition:modOnegAijvec}

Let  $A = [a_{kl}]\in [0 , 1]^{n \times m}$, $\vec{i} = (i_1, i_2, \dots, i_h) \in I^h$ be an $h$-tuple of two-by-two distinct indices and $\vec{j} = (j_1 , \dots , j_h)\in J^h$. We construct the matrix $A^{(\vec i , \vec j)} = [a^{(\vec i , \vec j)}_{kl}]_{1 \le k \le n,\, 1 \le l \le m}$ using (\ref{eq:constructionAijvec}). We have:
\begin{enumerate}
\item  For all $(k , l)\in I \times J$, we have $a^{(\vec i , \vec j)}_{kl}\in \{b_k , a_{kl}\}$.
   \item  For all $\lambda\in\{1 , \dots , h\}$, we have $\delta^{A^{(\vec i , \vec j)}}(i_\lambda, j_\lambda) = 0$.
\item For all $s \in I$ and for all $ l \in \{j_1 , j_2 , \dots , j_h\}$,  we have $\delta^A(s , l) = 0 \Longrightarrow \delta^{A^{(\vec i , \vec j)}}(s , l) = 0$. 
    \item For all $s \in I$  and for all $ l\notin \{j_1 , j_2 , \dots , j_h\}$,  we have $\delta^{A^{(\vec i , \vec j)}}(s , l) = \delta^A(s , l)$. 
\end{enumerate}
\end{restatable}
\begin{restatable}{proposition}{propositionmodtwogAijvec}
\label{proposition:mod2gAijvec}
Let  $A = [a_{kl}]\in [0 , 1]^{n \times m}$, Let $\vec{i} = (i_1, i_2, \dots, i_h) \in I^h$ be an $h$-tuple  of   two-by-two distinct indices of the set $I$ and   $\vec{j} = (j_1 , \dots , j_h)\in J^h$.  We construct the matrix $A^{(\vec i , \vec j)} = [a^{(\vec i , \vec j)}_{kl}]$ using (\ref{eq:constructionAijvec}). 
For any matrix $T\in[0 , 1]^{n \times m}$   verifying 
$\delta^T(i_\lambda , j_\lambda) = 0$  for all $\lambda\in\{1 , 2 , \dots , h\}$, we have:
\begin{equation}\label{eq:ilambdajlambdaGNew}
\forall (k , l)\in I \times J, \quad  \mid   
a^{(\vec i , \vec j)}_{kl} - a_{kl} \mid \leq \mid t_{kl} - a_{kl} \mid.
\end{equation}
It clearly follows that: $$\Vert A^{(\vec i , \vec j)} - A \Vert_p \leq \Vert T - A \Vert_p\quad \text{ for }\quad p\in\{1 , 2 , \infty\}.$$
\end{restatable}

The proofs of Propositions \ref{proposition:aijklbkorakl}, \ref{proposition:commutativePi}, \ref{proposition:modOnegAijvec} and \ref{proposition:mod2gAijvec} are given in Subsection \ref{sec:proofs:constructionAijvec}.

\subsubsection{Computing the coefficients of \texorpdfstring{$A^{(\vec{i},\vec{j})}$}{A vec i vec j}}

Let us consider $(\vec{i} = (i_1, i_2, \dots, i_h),\vec{j}  = (j_1 , \dots , j_h))$ a pair of $h$-tuples   as in Definition \ref{def:vecivecj}. In the following, we present a method for computing the coefficients of the matrix $A^{(\vec{i},\vec{j})}$ in a simpler way than the construction provided in (\ref{eq:constructionAijvec}).

We introduce three sets ${\cal E}_1 , {\cal E}_2$ and     ${\cal E}$ of index pairs that depend on $A$, $b$ and  $(\vec{i} , \vec{j})$ :

\begin{subequations}\label{eq:combinedE1E2Eh}
\begin{enumerate}
    \item The set ${\cal E}_1 $ is formed by the pairs $(i_\lambda, j_\lambda)$ where $\lambda \in \{1,2,\dots,h\}$ for which the corresponding component $b_{i_\lambda}$ of the right-hand side vector $b$ is strictly greater than the entry $a_{i_\lambda j_\lambda}$ of the matrix $A$:
    \begin{flalign}
     &  {\cal E}_1 := \{ (i_\lambda, j_\lambda) \in I \times J \mid \lambda \in \{1,\dots,h\} \text{ and } b_{i_\lambda} > a_{i_\lambda j_\lambda} \}.
    \end{flalign}
    
    \item The set ${\cal E}_2$ is defined using the sets
    $
    U^A_{i_\lambda j_\lambda} = \{ k \in I \mid \sigma_G(b_{i_\lambda}, a_{k j_\lambda}, b_k) > 0 \},
    $
    see (\ref{eq:UDij}):
    \begin{flalign}
    & {\cal E}_2 := \bigcup_{\lambda=1}^h \{ (k, j_\lambda) \mid k \in U^A_{i_\lambda j_\lambda} \} \subseteq I \times J.
    \end{flalign}
    
    \item Finally, the set ${\cal E}$ is defined as the union of the previous two sets:
    \begin{flalign}
     & {\cal E} := {\cal E}_1 \cup {\cal E}_2 \subseteq I \times J.
    \end{flalign}
\end{enumerate}
\end{subequations}

Note that for every pair $(k,l)$ where $k \in I$ and $l \in J \setminus \{j_1, j_2, \dots, j_h\}$, we have $(k,l) \notin {\cal E}$, since both ${\cal E}_1$ and ${\cal E}_2$ only contain pairs with the second component equal to one of $j_1,\dots,j_h$.

Let us fix an index $\lambda \in \{1,2,\dots,h\}$ and define the reduced $(h-1)$-tuples $\vec{i}'$ and $\vec{j}'$ by removing the $\lambda$-th component from $\vec{i}$ and $\vec{j}$, respectively. 
We need to    establish relationships between the sets ${\cal E}_1 , {\cal E}_2 , {\cal E} $, associated to the original pair $(\vec{i},\vec{j})$, and the sets ${\cal E}'_1 $, ${\cal E}'_2 $, ${\cal E}' $, associated to the reduced   pair   $(\vec{i}',\vec{j}')$.

Set $\vec{i}' = (i'_1, i'_2, \dots, i'_{h-1}) \quad \text{and} \quad \vec{j}' = (j'_1, j'_2, \dots, j'_{h-1}).$  We have for every $\lambda' \in \{1,2,\dots,h-1\}$:
$$
i'_{\lambda'} =
\begin{cases}
i_{\lambda'}, & \text{if } \lambda' \le \lambda - 1,\\[1mm]
i_{\lambda'+1}, & \text{if } \lambda' \ge \lambda,
\end{cases}
\quad
j'_{\lambda'} =
\begin{cases}
j_{\lambda'}, & \text{if } \lambda' \le \lambda - 1,\\[1mm]
j_{\lambda'+1}, & \text{if } \lambda' \ge \lambda.
\end{cases}
$$

The sets ${\cal E}'_1 , {\cal E}'_2 , {\cal E}'$ associated with  the matrix $A$,    the right-hand side vector $b$ of the system $(S)$ and the pair of $(h-1)$-tuples   $(\vec{i}' , \vec{j}')$ , are   the following subsets of $I \times J$:
\begin{subequations}\label{eq:ehminusonesets}
\begin{flalign}
{\cal E}'_1  &= \{ (i'_{\lambda'}, j'_{\lambda'}) \in I \times J \mid \lambda' \in \{1,\dots,h-1\} \text{ and } b_{i'_{\lambda'}} > a_{i'_{\lambda'} j'_{\lambda'}} \},&\\[1mm]
{\cal E}'_2  &= \bigcup_{\lambda'=1}^{h-1} \{ (k, j'_{\lambda'}) \mid k \in U^A_{i'_{\lambda'} j'_{\lambda'}} \},&\\[1mm]
{\cal E}'  &= {\cal E}'_1  \cup {\cal E}'_2 .&
\end{flalign}
\end{subequations}

It is clear from the definitions of the sets ${\cal E}_1 $, ${\cal E}_2 $, and ${\cal E} $ that the following properties hold.

\begin{restatable}{lemma}{lemmaIncEhEprimeh}
\label{lemma:IncEhEprimeh}
Let   $(\vec{i} = (i_1, i_2, \dots, i_h),\vec{j}  = (j_1 , \dots , j_h))$ be a pair of $h$-tuples   as in Definition \ref{def:vecivecj}. For $h > 1$, let ${\cal E}_1$, ${\cal E}_2$, and ${\cal E}$ be the sets associated to the pair  $(\vec{i},\vec{j})$ as in (\ref{eq:combinedE1E2Eh}). Given a fixed index $\lambda \in \{1,2,\dots,h\}$, let $(\vec{i}',\vec{j}')$ denote the pair obtained by removing the $\lambda$-th component from $(\vec{i},\vec{j})$. Then, the corresponding sets ${\cal E}'_1$, ${\cal E}'_2$, and ${\cal E}'$, see (\ref{eq:ehminusonesets}), satisfy:

\begin{enumerate}
    \item
    $
    {\cal E}'_1  = {\cal E}_1  \setminus \{(i_\lambda, j_\lambda)\},
    $
    \item
    $
    {\cal E}'_2  \subseteq {\cal E}_2 ,
    $
    \item
    $
    {\cal E}'  = {\cal E}'_1  \cup {\cal E}'_2  \subseteq {\cal E} .
    $
\end{enumerate}
\end{restatable}
\qed \\

The set ${\cal E}$, see (\ref{eq:combinedE1E2Eh}), allows us to construct the matrix $A^{(\vec{i},\vec{j})}$, see (\ref{eq:constructionAijvec}), in a simpler way:

\begin{restatable}{proposition}{propositionSimpleDefAijvec}
\label{proposition:SimpleDefAijvec}
Let   $(\vec{i} = (i_1, i_2, \dots, i_h),\vec{j}  = (j_1 , \dots , j_h))$ be a pair of $h$-tuples   as in Definition \ref{def:vecivecj}. \\
Then, the coefficients of the modified matrix $A^{(\vec i , \vec j)} = [a^{(\vec i , \vec j)}_{kl}]_{1 \le k \le n,\, 1 \le l \le m}$, see (\ref{eq:constructionAijvec}), are given by: 
$$\text{for every $(k,l) \in I \times J$}, \quad 
a^{(\vec{i},\vec{j})}_{kl} =
\begin{cases}
b_k, & \text{if } (k,l) \in {\cal E} ,\\[1mm]
a_{kl}, & \text{if } (k,l) \notin {\cal E} .
\end{cases}
$$

\end{restatable}
The proof is by induction on $h \geq 1$, see Section \ref{sec:proofs:constructionAijvec}.

The $L_\infty$ distance $\Vert A^{(\vec{i},\vec{j})} - A \Vert_\infty$ can be easily computed  using: 
\begin{restatable}{corollary}{corollarynormlambdainfty}
\label{corollary:normlambdainfty}
Let   $(\vec{i} = (i_1, i_2, \dots, i_h),\vec{j}  = (j_1 , \dots , j_h))$ be a pair of $h$-tuples   as in Definition \ref{def:vecivecj}. \\
Then, we have:  
\begin{equation}\label{eq:normlambda}
\Vert A^{(\vec{i},\vec{j})} - A \Vert_\infty = \max_{1 \leq \lambda \leq h}\,
\Vert A^{(i_\lambda,j_\lambda)} - A \Vert_\infty    
\end{equation}

\end{restatable}
The proof of Corollary \ref{corollary:normlambdainfty}  is in Section \ref{sec:proofs:constructionAijvec}

\begin{example}\label{ex:aijmatrices}
(continued)

We reuse the inconsistent max-min system $(S) : A \Box_{\min}^{\max} x = b$ of Example \ref{ex:incons} where:
\[
A = \begin{bmatrix}
0.04 & 0.73 & 0.5 \\
0.33 & 0.35 & 0.94 \\
0.55 & 0.9  & 0.35
\end{bmatrix}, \quad
b = \begin{bmatrix}
0.3\\[0.5em]
0.42\\[0.5em]
0.76
\end{bmatrix}. \]

Let $\vec{i} = (i_1, i_2) = (1,2)$ and $\vec{j} = (j_1, j_2) = (1,1)$.  We get $\mathcal{E}_1 = \{ (1, 1), (2, 1) \}$ and $\mathcal{E}_2 = \emptyset$. So, we have $\mathcal{E} = \mathcal{E}_1$. 

The matrix $A^{(\vec{i},\vec{j})}$ is:
\[ 
\begin{bmatrix}
    0.30&  0.73&  0.50\\
  0.42&  0.35 & 0.94\\
  0.55&  0.90&  0.35.
\end{bmatrix}
\]
which is equal to the matrices $D$ and $F$ obtained in the end of Example \ref{ex:aijmatrices} for illustrating  Lemma \ref{lemma:commutativeAij}.
\end{example}

\section{Main results on constructing consistent systems by minimally modifying the matrix of an inconsistent max-min system}
\label{sec:mainresults}

In this section, we establish our main results for handling the inconsistency of a max-min system $(S): A \Box_{\min}^{\max} x = b$ by constructing consistent systems close to the system  $(S)$: for  each obtained consistent system  that we obtain, its right-hand side vector is the right-hand side vector $b$ of the system  $(S)$ while its matrix is of the form $A^{(\vec{i}, \vec{j})}$ (Definition \ref{def:vecivecj}) and is obtained  by minimal modifications of the matrix $A$ of the system $(S)$, see Theorem \ref{theorem:theo1consSij} and Theorem \ref{theorem:theo2Aijnorm}. 

From these results, we compute the distance $\mathring{\Delta}_{p}$, see (\ref{eq:deltaAp}),  in a simpler way, see Corollary \ref{corollary:normminJh}. Furthermore, when using the $L_\infty$ norm, i.e., $p = \infty$, we give an explicit analytical formula (Corollary \ref{corollary:formulaforLinftydist}) for computing the distance $\mathring{\Delta}_{\infty}$, see (\ref{eq:deltaAp}). Finally,  we provide a  method to construct 
a non-empty and finite subset of matrices  $A^{(\vec{i}, \vec{j})}$ 
  whose distance to the matrix  $A$ of the inconsistent system is equal to the scalar $\mathring{\Delta}_{\infty}$, see (\ref{eq:lAinfty}).

\subsection{Constructing consistent systems close to the initial inconsistent system \texorpdfstring{$(S)$}{(S)}}

We show  that for some pairs  $(\vec{i} , \vec{j})$,  the matrix $A^{(\vec{i},\vec{j})}$ constructed from $A$ and $b$ with (\ref{eq:constructionAijvec}),  yields a consistent max-min system whose matrix is $A^{(\vec{i},\vec{j})}$ and its right-hand side vector is still the vector $b$ of the system $(S)$:

\begin{restatable}{theorem}{theoremtheoOneconsSij}
\label{theorem:theo1consSij}
We use the non-empty set $N_{\text{inc}}$ with  $h := \operatorname{card}(N_{\text{inc}})$, see (\ref{eq:nconsninc}). Choose an arbitrary enumeration of $N_{\text{inc}}$ as $\{i_1, i_2, \dots, i_h\}$ and set $\vec{i} = (i_1, i_2, \dots, i_h)$. Then, for every $h$-tuple $\vec{j} = (j_1, j_2, \dots, j_h) \in J^h$, the matrix $A^{(\vec{i},\vec{j})}$ constructed with (\ref{eq:constructionAijvec}) yields a consistent system, which is: 
$$
(S^{(\vec{i},\vec{j})}): A^{(\vec{i},\vec{j})} \Box_{\min}^{\max} x = b.
$$ 
 Therefore, the matrix $A^{(\vec{i},\vec{j})}$ belongs to the set $\mathcal{T}$ see (\ref{eq:setT}). \\
 For $p\in\{1 , 2 , \infty\}$, we have: $$\Vert A^{(\vec{i},\vec{j})} - A \Vert_p \geq \mathring{\Delta}_{p},$$ see (\ref{eq:deltaAp}) for the definition of $\mathring{\Delta}_{p}$.
\end{restatable}
The following result highlights the minimality of the modifications made to obtain the  matrices   $A^{(\vec{i},\vec{j})}$  from $A$ using (\ref{eq:constructionAijvec}) and an $h$-tuple $\vec i$ as in Theorem \ref{theorem:theo1consSij}:

\begin{restatable}{theorem}{theoremtheoTwoAijnorm}
\label{theorem:theo2Aijnorm}
Retain the notations from Theorem~\ref{theorem:theo1consSij} and let 
\[
A = [a_{kl}]_{1 \le k \le n,\; 1 \le l \le m} \quad \text{and} \quad 
 \vec{i} = (i_1, i_2, \dots, i_h) \quad \text{such that } \quad N_{\text{inc}}=\{i_1, i_2, \dots, i_h\}\]
Then, for every consistent system whose right-hand side vector is $b$:
\[
T \Box_{\min}^{\max} x = b,\quad \text{with } T = [t_{kl}]_{1 \le k \le n,\; 1 \le l \le m} \in \mathcal{T},
\]
there exists an $h$-tuple $\vec{j} = (j_1, j_2, \dots, j_h) \in J^h$ such that the system 
\[
A^{(\vec{i},\vec{j})} \Box_{\min}^{\max} x = b
\]
is consistent and, for every $(k,l) \in I \times J$, 
\[
|a^{(\vec{i},\vec{j})}_{kl} - a_{kl}| \le |t_{kl} - a_{kl}| \quad \text{where} \quad 
A^{(\vec{i},\vec{j})} = [a^{(\vec{i},\vec{j})}_{kl}]_{1 \le k \le n,\; 1 \le l \le m}.
\]
We have:  $$\Vert A^{(\vec{i},\vec{j})} - A \Vert_p \leq \Vert T - A \Vert_p\, \quad \text{ for } \quad p\in\{1 , 2 , \infty\}.$$ 
\end{restatable}
The above theorem  allows us to compute the distance $\mathring{\Delta}_{p}$, see (\ref{eq:deltaAp}), using the finite set $J^h$ instead of $\mathcal T$:
\begin{restatable}{corollary}{corollarynormminJh}
\label{corollary:normminJh}
Recall that for $p \in \{1,2,\infty\}$ we defined in (\ref{eq:deltaAp}):
\[
\mathring{\Delta}_{p} := \inf_{T \in \mathcal{T}} \|T - A\|_p.
\]
Then, it follows from Theorems~\ref{theorem:theo1consSij} and \ref{theorem:theo2Aijnorm} that
\begin{equation}\label{eq:DeltaAinfty}
\mathring{\Delta}_{p} = \min_{\vec{j} \in J^h} \|A^{(\vec{i},\vec{j})} - A\|_p.    
\end{equation}
Moreover, this minimum is independent of the particular enumeration $\{i_1, i_2, \dots, i_h\}$ of the set $N_{\text{inc}}$ chosen to define the $h$-tuple $\vec{i} = (i_1, i_2, \dots, i_h)$.
 \end{restatable}

The proofs of Theorem \ref{theorem:theo1consSij} and  \ref{theorem:theo2Aijnorm} and Corollary \ref{corollary:normminJh}  are given in Subsection \ref{sec:proofs:mainresults}.

\subsection{Computing the distance \texorpdfstring{$\mathring{\Delta}_{\infty}$}{Delta A L infinity norm} by an explicit analytical formula}
\label{subsec:computingdistance}

In the following, using the $L_\infty$ norm i.e, $p = \infty$, we aim to construct a matrix 
$
A^{(\vec{i},\vec{j})}
$, see (\ref{eq:constructionAijvec}), 
such that
\begin{equation}\label{eq:satisfylinfty}
  \mathring{\Delta}_{\infty}=   \|A^{(\vec{i},\vec{j})} - A\|_\infty  ,
\end{equation}

\noindent where $ \mathring{\Delta}_{\infty}$ is defined in (\ref{eq:deltaAp}) and $A$ is the matrix of the inconsistent system $(S): A \Box_{\min}^{\max} x = b$. This construction will allow us to introduce an explicit analytical formula for computing  the distance $\mathring{\Delta}_{\infty}$.

Our first objective is to find a pair $(\vec{i},\vec{j})$ satisfying the equality (\ref{eq:satisfylinfty}).
We process as follows:
\begin{itemize}
    \item We choose an arbitrary enumeration of $N_{\text{inc}}$ as $\{i_1, i_2, \dots, i_h\}$ , see (\ref{eq:nconsninc}), and set $\vec{i} = (i_1, i_2, \dots, i_h)$.

    \item To the $h$-uplet $\vec{i} := (i_1, \dots, i_h)$, we associate an $h$-uplet $\vec{j}  = (j_1, j_2, \dots, j_h)$ whose components satisfy the following requirement: for each $\lambda \in \{1,2,\dots,h\}$, set 
    \begin{equation}\label{eq:levecj}
    j_\lambda := \operatorname{argmin}_{j \in J} \| A^{(i_\lambda, j)} - A \|_\infty, \text{ which means that } \| A^{(i_\lambda, j_\lambda)} - A \|_\infty = \min_{j \in J} \| A^{(i_\lambda, j)} - A \|_\infty.
    \end{equation}
\end{itemize}
The above method allows us to  introduce the following set of pairs $(\vec{i}, \vec{j})$: 

\begin{equation}\label{eq:lAinfty}
L_{A,\infty} = \Bigl\{ (\vec{i},\vec{j}) \,\Bigm|\, 
\begin{array}{l}
\vec{i} = (i_1,\dots,i_h) \text{ is any ordering of } N_{\text{inc}},\\[1mm]
\vec{j} = (j_1,\dots,j_h) \text{ with } j_\lambda \in \mathcal{J}(i_\lambda) \text{ for each } \lambda=1,\dots,h
\end{array}
\Bigr\},    
\end{equation}

\noindent where $\mathcal{J}: I \to 2^J: i \rightarrow  \{ j \in J \mid \|A^{(i,j)} - A\|_\infty = \min\limits_{j' \in J}\|A^{(i,j')} - A\|_\infty \}$. \\

From Theorem \ref{theorem:theo1consSij}, we know that  for any $(\vec{i},\vec{j}) \in L_{A,\infty}$, we have $A^{(\vec{i},\vec{j})} \in {\cal T}$, i.e., 
the system $A^{(\vec{i},\vec{j})} \Box_{\min}^{\max} x = b$ is consistent and therefore $\mathring{\Delta}_{\infty} \leq   \|A^{(\vec{i},\vec{j})} - A\|_\infty$. The following theorem is the main step  to compute $\mathring{\Delta}_{\infty}$ by an explicit analytical formula:   

\begin{restatable}{theorem}{theoremformulaforLinftydist}
\label{theorem:formulaforLinftydistance}
Let $(S) = A \Box_{\min}^{\max} x = b$ be an inconsistent system. Then, 
for any $(\vec i , \vec j)\in L_{A,\infty}$ (see (\ref{eq:lAinfty})), we have: 
\begin{enumerate}
    \item $\mathring{\Delta}_{\infty} = \Vert A^{(\vec i , \vec j)}  - A \Vert_\infty.$
\item $\mathring{\Delta}_{\infty} = \max_{1 \leq \lambda \leq h}\, \Vert A^{(i_\lambda , j_\lambda)}  - A \Vert_\infty $ where $\vec i = (i_1 , \dots , i_h)$ and $\vec j = (j_1 , \dots , j_h).$
\end{enumerate}
\end{restatable}

We give an explicit analytical formula for computing the distance $\mathring{\Delta}_{\infty}$:
\begin{restatable}{corollary}{corollaryformulaforLinftydist}
\label{corollary:formulaforLinftydist}
For any system $(S) : A \Box_{\min}^{\max} x = b$, 
 we have:
\begin{equation}\label{eq:formulaforLinftydistance}
    \mathring{\Delta}_{\infty} = \max_{i\in I} \, \min_{j\in J} \, \max\Bigl[(b_i - a_{ij})^+, \, \max_{k \in I, k\neq i}\Bigl(\theta(i,k) \cdot (a_{kj} - b_k)^+\Bigr)\Bigr]
\end{equation}

where $\theta: I \times I \mapsto \{0,1\}: \theta(i,k) = \begin{cases}
1 & \text{ if } b_i > b_k\\
0 & \text{ otherwise}
\end{cases}$ and was defined in (\ref{eq:UDij1}).

The formula (\ref{eq:formulaforLinftydistance}) can be reformulated using the set $N_{\text{inc}}$, see (\ref{eq:nconsninc}):
\[
\mathring{\Delta}_{\infty} = \max_{i\in N_{\text{inc}}} \, \min_{j\in J} \, \max\Bigl[(b_i - a_{ij})^+, \, \max_{k \in I, k\neq i}\Bigl(\theta(i,k) \cdot (a_{kj} - b_k)^+\Bigr)\Bigr].
\]
\end{restatable}

The formula (\ref{eq:formulaforLinftydistance}) of $\mathring{\Delta}_{\infty}$ is useful for estimating to which extent the system $(S)$ is inconsistent, with respect to its matrix $A$. Analogously, using $\Delta$, see (\ref{eq:Delta}), we can measure how inconsistent the system $(S)$ is, with respect to its right-hand side vector $b$. It is easy to see that we have $\Delta \leq \mathring{\Delta}_{\infty}$ for any max-min system $(S)$.

\begin{example}
(continued)
We reuse the inconsistent max-min system $(S) : A \Box_{\min}^{\max} x = b$ of Example \ref{ex:incons} and the matrices of Example \ref{ex:matAij} where $J = \{1,2,3\}$ and:
\[
A = \begin{bmatrix}
0.04 & 0.73 & 0.5 \\
0.33 & 0.35 & 0.94 \\
0.55 & 0.9  & 0.35
\end{bmatrix}, \quad
b = \begin{bmatrix}
0.3\\[0.5em]
0.42\\[0.5em]
0.76
\end{bmatrix}.
\]
We obtained
\[
N_{\text{cons}} = \{1\}, \quad N_{\text{inc}} = \{2, 3\}.
\]
We define two vectors: 
\[
\vec{i} = (i_1,i_2) = (2,3)
\]
(from the elements of $N_{\text{inc}}$) and 
\[
\vec{j} = (j_1,j_2) \in J^2.
\]
Since $|J|=3$ and the length of $\vec{i}$ is $2$, there are $3^2=9$ possible choices for $\vec{j}$. For each such pair $(\vec{i},\vec{j})$, we update $A$ in two steps:
\begin{itemize}
  \item Step $\lambda=1$: Update using the pair $(i_1,j_1)$.
  \item Step $\lambda=2$: Update using the pair $(i_2,j_2)$.
\end{itemize}
Below, we detail the nine cases.

\noindent {Case 1:} $\vec{j} = (1,1)$
\begin{itemize}
  \item $\lambda=1$: For $(i_1,j_1) = (2,1)$, update the (2,1)-entry from $0.33$ to $0.42$.
  \item $\lambda=2$: For $(i_2,j_2) = (3,1)$, update the (3,1)-entry from $0.55$ to $0.76$.
\end{itemize}
The final matrix is
\[
A^{(\vec{i},\vec{j})} = \begin{bmatrix}
0.04 & 0.73 & 0.5 \\
0.42 & 0.35 & 0.94 \\
0.76 & 0.9  & 0.35
\end{bmatrix},
\]
and we have $\| A^{(\vec{i},\vec{j})} - A \|_{\infty} = 0.21.$

\noindent {Case 2:} $\vec{j} = (1,2)$
\begin{itemize}
  \item $\lambda=1$: For $(2,1)$, update the (2,1)-entry from $0.33$ to $0.42$.
  \item $\lambda=2$: For $(3,2)$, update (1,2)-entry from $0.73$ to $0.3$.
\end{itemize}
The final matrix is
\[
A^{(\vec{i},\vec{j})} = \begin{bmatrix}
0.04 & 0.3 & 0.5 \\
0.42 & 0.35 & 0.94 \\
0.55 & 0.9  & 0.35
\end{bmatrix},
\]
and we have $\| A^{(\vec{i},\vec{j})} - A \|_{\infty} = 0.43.$

\noindent {Case 3:} $\vec{j} = (1,3)$
\begin{itemize}
  \item $\lambda=1$: For $(2,1)$, update the $(2,1)$-entry from $0.33$ to $0.42$.
  \item $\lambda=2$: For $(3,3)$, update:
  \[
  \text{the (3,3)-entry from } 0.35 \text{ to } 0.76,\quad
  \text{the (1,3)-entry from } 0.5 \text{ to } 0.3,\quad
   \text{the (2,3)-entry from } 0.94 \text{ to } 0.42.
  \]
\end{itemize}
The final matrix is
\[
A^{(\vec{i},\vec{j})} = \begin{bmatrix}
0.04 & 0.73 & 0.3 \\
0.42 & 0.35 & 0.42 \\
0.55 & 0.9  & 0.76
\end{bmatrix},
\]
and we have $
\| A^{(\vec{i},\vec{j})} - A \|_{\infty} = 0.52.
$

\noindent {Case 4:} $\vec{j} = (2,1)$
\begin{itemize}
  \item $\lambda=1$: For $(2,2)$, update:
  \[
  \text{the (2,2)-entry from } 0.35 \text{ to } 0.42,\quad
   \text{the (1,2)-entry from } 0.73 \text{ to } 0.3.
  \]
  \item $\lambda=2$: For $(3,1)$, update the (3,1)-entry from $0.55$ to $0.76$.
\end{itemize}
The final matrix is
\[
A^{(\vec{i},\vec{j})} = \begin{bmatrix}
0.04 & 0.3 & 0.5 \\
0.33 & 0.42 & 0.94 \\
0.76 & 0.9  & 0.35
\end{bmatrix},
\]
and we have $
\| A^{(\vec{i},\vec{j})} - A \|_{\infty} = 0.43.
$

\noindent {Case 5:} $\vec{j} = (2,2)$
\begin{itemize}
  \item $\lambda=1$: For $(2,2)$, update:
  \[
  \text{the (2,2)-entry from } 0.35 \text{ to } 0.42,\quad
 \text{the (1,2)-entry from } 0.73 \text{ to } 0.3.
  \]
  \item $\lambda=2$: For $(3,2)$, no change is made.
\end{itemize}
The final matrix is
\[
A^{(\vec{i},\vec{j})} = \begin{bmatrix}
0.04 & 0.3 & 0.5 \\
0.33 & 0.42 & 0.94 \\
0.55 & 0.9  & 0.35
\end{bmatrix},
\]
and we have $\| A^{(\vec{i},\vec{j})} - A \|_{\infty} = 0.43.$

\noindent {Case 6:} $\vec{j} = (2,3)$
\begin{itemize}
  \item $\lambda=1$: For $(2,2)$, update:
  \[
   \text{the (2,2)-entry from } 0.35 \text{ to } 0.42,\quad
   \text{the (1,2)-entry from } 0.73 \text{ to } 0.3.
  \]
  \item $\lambda=2$: For $(3,3)$, update:
  \[
  \text{the (3,3)-entry from } 0.35 \text{ to } 0.76,\quad
  \text{the (1,3)-entry from } 0.5 \text{ to } 0.3,\quad
\text{the (2,3)-entry from } 0.94 \text{ to } 0.42.
  \]
\end{itemize}
The final matrix is
\[
A^{(\vec{i},\vec{j})} = \begin{bmatrix}
0.04 & 0.3 & 0.3 \\
0.33 & 0.42 & 0.42 \\
0.55 & 0.9  & 0.76
\end{bmatrix},
\]
and we have $\| A^{(\vec{i},\vec{j})} - A \|_{\infty} = 0.52.$

\noindent {Case 7:} $\vec{j} = (3,1)$
\begin{itemize}
  \item $\lambda=1$: For $(2,3)$, update the (1,3)-entry  from $0.5$ to $0.3$.
  \item $\lambda=2$: For $(3,1)$, update the (3,1)-entry from $0.55$ to $0.76$.
\end{itemize}
The final matrix is
\[
A^{(\vec{i},\vec{j})} = \begin{bmatrix}
0.04 & 0.73 & 0.3 \\
0.33 & 0.35 & 0.94 \\
0.76 & 0.9  & 0.35
\end{bmatrix},
\]
and we have $\| A^{(\vec{i},\vec{j})} - A \|_{\infty} = 0.21.$

\noindent {Case 8:} $\vec{j} = (3,2)$
\begin{itemize}
  \item $\lambda=1$: For $(2,3)$, update the (1,3)-entry from $0.5$ to $0.3$.
  \item $\lambda=2$: For $(3,2)$, update (1,2)-entry from $0.73$ to $0.3$.
\end{itemize}
The final matrix is
\[
A^{(\vec{i},\vec{j})} = \begin{bmatrix}
0.04 & 0.3 & 0.3 \\
0.33 & 0.35 & 0.94 \\
0.55 & 0.9  & 0.35
\end{bmatrix},
\]
and we have $\| A^{(\vec{i},\vec{j})} - A \|_{\infty} = 0.43.$\\
\noindent {Case 9:}  $\vec{j} = (3,3)$
\begin{itemize}
  \item $\lambda=1$: For $(2,3)$, update the (1,3)-entry from $0.5$ to $0.3$.
  \item $\lambda=2$: For $(3,3)$, update:
  \[
   \text{the (3,3)-entry from } 0.35 \text{ to } 0.76,\quad
   \text{the (2,3)-entry from } 0.94 \text{ to } 0.42.
  \]
\end{itemize}
The final matrix is
\[
A^{(\vec{i},\vec{j})} = \begin{bmatrix}
0.04 & 0.73 & 0.3 \\
0.33 & 0.35 & 0.42 \\
0.55 & 0.9  & 0.76
\end{bmatrix},
\]
and we have $\| A^{(\vec{i},\vec{j})} - A \|_{\infty} = 0.52.$ 

Therefore, we constructed nine consistent systems of the form $A^{(\vec{i}, \vec{j})} \Box_{\min}^{\max}x =b$, which use the same right-hand side vector: the right-hand side vector $b$ of the original inconsistent system $A \Box_{\min}^{\max}x =b$.

    The best vector pairs (minimizing $\| A^{(\vec{i},\vec{j})} - A \|_{\infty}$) yield a value of $0.21$. In our cases, these are:
\[
\big((2,3),\,(1,1)\big) \quad \text{and} \quad \big((2,3),\,(3,1)\big).
\]
Moreover, by choosing
\[
j_\lambda = \operatorname*{argmin}_{j \in J^h} \| A^{(i_\lambda,j)} - A \|_{\infty},
\]
we obtain these vector pairs which results in the best norm $\| A^{(\vec{i},\vec{j})} - A \|_{\infty} = 0.21$.
\end{example}

\section{Analogous results for a system of min-max fuzzy relational equations}
\label{sec:minmaxres}
In \cite{baaj2024maxmin}, we showed how to switch from a system of min-max fuzzy relational equations  to a system of max-min fuzzy relational equations. This equivalence is based on the transformation $t \rightarrow 1 - t$, which allowed us to show that studying a system of max-min fuzzy relational equations is equivalent to studying a system of min-max fuzzy relational equations, see \cite{baaj2024maxmin} for more details.

In this section, using the same notations as in \cite{baaj2024maxmin}, we begin by reminding how to solve a system of min-max fuzzy relational equations. 
Given an inconsistent system of  min-max fuzzy relational equations, we study how to obtain  consistent min-max systems close to the inconsistent system by  minimally modifying the matrix of the inconsistent system. We introduce new tools for a system of min-max fuzzy relational equations corresponding to those already introduced for a system of max-min fuzzy relational equations and we show their  correspondences in Table \ref{tab:correspondences}.

\subsection{Solving systems of \texorpdfstring{min-max}{min-max} fuzzy relational equations}
Let us use:
\begin{notation}\mbox{}
\label{not:minmaxbasic}
\begin{itemize}
    \item To any matrix $A = [a_{ij}]$, we associate the matrix  $A^{\circ} := [a_{ij}^\circ]$ where $a_{ij}^\circ:= 1 - a_{ij}$.  Clearly, 
      we have $(A^{\circ})^{\circ} = A$.
    \item The min-max matrix product $\Box^{\min}_{\max}$ uses the function $\max$ as the product and the function $\min$ as the addition and for  any matrices $A$ and $B$ of size $(n , p)$ and $(p ,m)$ respectively, we have:
    \begin{equation}\label{eq:cinq1}
     (A \Box_{\min}^{\max} B)^\circ = A^\circ \Box_{\max}^{\min} B^\circ.   
    \end{equation}
    \item The $\max-\epsilon$ matrix product $\Box_{\epsilon}^{\max}$ uses the $\epsilon$-product  defined by: \[ x \epsilon  y = \left\{\begin{array}{rrl}
y & \text{if}  & x < y\\
0 & \text{if}  & x \geq  y \\\end{array}\right. \text{ in } [0, 1].\]  as the product and the function $\max$ as the addition.
\end{itemize}
\end{notation}

Let us use the following system of min-max fuzzy relational equations:
\begin{equation}\label{eq:minmaxsys}
    G \Box_{\max}^{\min} x = d
\end{equation}
\noindent where
\[
G = [g_{ij}]_{1 \leq i \leq n,\,1 \leq j \leq m} \in [0,1]^{n \times m}, \quad d = [d_i]_{1 \leq i \leq n}, \quad \text{and} \quad x = [x_j]_{1 \leq j \leq m} \in [0,1]^{m} \text{ is the unknown vector}.
\]

\noindent To check if the min-max system $G \Box_{\max}^{\min} x = d$  is consistent, we compute the following vector: \begin{equation}\label{eq:rlowtestsol}
r = G^t \Box_{\epsilon}^{\max} d,\end{equation}
\noindent where $G^t$ is the transpose of $G$  and we use the $\max-\epsilon$ matrix product $\Box_{\epsilon}^{\max}$ (Notation \ref{not:minmaxbasic}). 
\noindent The vector $r$ is the potential lowest solution of the system $G \Box_{\max}^{\min} x = d$. 

\noindent Sanchez's result \cite{sanchez1976resolution} ((reminded in (\ref{eq:consiste1}))  for  solving of a system max-min fuzzy relational equations becomes  the following one for solving a system of min-max fuzzy relational equations:
\begin{equation}\label{eq:consiste1Minmax}
    \text{The system } G \Box_{\max}^{\min} x = d \text{ is consistent}\Longleftrightarrow G \Box_{\max}^{\min} r = d. 
\end{equation}
\noindent The structure of the solution set  of a  system of min-max fuzzy relational equations is as follows. If the system $G \Box_{\max}^{\min} x = d$ is consistent,  the vector $r = G^t \Box_{\epsilon}^{\max} d$  is the lowest solution of the system, and the system  has a finite number of maximal solutions.

To handle the inconsistency of a min-max system, the author of \cite{baaj2024maxmin} studied the Chebyshev distance associated with the right-hand side vector $d$ of the system $G \Box_{\max}^{\min} x = d$:
\begin{equation}
     \nabla = \nabla(G,d) = \min_{w \in{\cal W} }\Vert d - w\Vert_{\infty} \text{ where }   {\cal W} = \{w\in [0 , 1]^{n \times 1}\mid \text{the system }  G \Box_{\max}^{\min} x =  w \text{ is consistent}\}.
\end{equation}
 \cite{baaj2024maxmin} computes by an explicit analytical formula the Chebyshev distance $\nabla$:
\begin{equation}\label{eq:nablapremier}
    \nabla  =  \max_{1 \leq i \leq n}\,\nabla_i^G
\end{equation}
where for $i = 1, 2, \dots n$: 
\begin{equation} \label{eq:Nablamini} 
\nabla_i^G =  \min_{1 \leq j \leq m}\,\max[ ( g_{ij}-d_i)^+,  \max_{1 \leq k \leq n}\,  \,\sigma_\epsilon\,(d_i, g_{kj}, d_k)]
\end{equation}
\noindent and
\begin{equation}\label{eq:sigmaespi}
    \sigma_\epsilon\,(u,v,w) = \min(\frac{(w-u)^+}{2}, (w - v)^+) = \sigma_G(u^\circ , v^\circ , w^\circ).  
\end{equation}
where $x^\circ := 1 - x$ for any $x\in[0 , 1]$ and the function $\sigma_G$ is defined in (\ref{eq:sigmaG}).

\noindent
Similarly to the case of an inconsistent max-min system,  the author of  \cite{baaj2024maxmin}  studies how to obtain approximate solutions of an inconsistent min-max system $G \Box_{\max}^{\min} x = d$ and  Chebyshev approximations of the right-hand side vector $d$ of this system. 

\subsection{Constructing consistent systems by minimally modifying the matrix of an inconsistent min-max system}
\label{subsec:minmaxsysresults}
In this section, we study how to minimally modify the matrix $G$ of size $(n,m)$ of an inconsistent min-max  system $G \Box_{\max}^{\min} x = d$, see (\ref{eq:minmaxsys}), in order to obtain a consistent min-max system, while keeping the right-hand side vector $d$ unchanged. This study is analogous to that carried out for a max-min system $A \Box_{\min}^{\max} x = b$, see Section \ref{sec:minimallymod} and Section \ref{sec:mainresults}.

\noindent Table \ref{tab:correspondences} summarizes the main tools introduced for a max-min system $A \Box_{\min}^{\max} x = b$ and those for a min-max system $G \Box_{\min}^{\min} x = d$, introduced in what follows. In the case where $G:= A^\circ$ and $d:= b^\circ$ (Notation \ref{not:minmaxbasic}), we establish explicit correspondences between the tools of these  systems.  The rest of this section introduces the tools for a min-max system $G \Box_{\min}^{\min} x = d$ and we justify the correspondences with the tools for a max-min system.

\begin{table}[H]
\footnotesize
\centering
\begin{tabular}{c|c|c|c|}
\hhline{~---}
& \multicolumn{1}{c|}{System: $A \Box_{\min}^{\max} x = b$} & \multicolumn{1}{c|}{System: $G \Box_{\max}^{\min} x = d$} & 
\multicolumn{1}{c|}{
\cellcolor[HTML]{C0C0C0}
\begin{tabular}[c]{@{}l@{}}
Relation \\
 iff $G = A^\circ$ and $d= b^\circ$
\end{tabular}}\\\hline

\multicolumn{1}{|l|}{
\begin{tabular}[c]{@{}l@{}}Potential greatest/lowest\\ solution\end{tabular}}   &   \begin{tabular}[c]{@{}l@{}} $e=A^t\Box_{\rightarrow_G}^{\min} b$ \\ (greatest solution) \end{tabular}     & \begin{tabular}[c]{@{}l@{}} $ r=G^t\Box_{\epsilon}^{\max} d $\\ (lowest solution) \end{tabular}        &  \cellcolor[HTML]{C0C0C0}$r= e^\circ$  \\ \hline

\multicolumn{1}{|l|}{\begin{tabular}[c]{@{}l@{}}Chebyshev distance\\ associated  with \\the right-hand side vector\end{tabular}}                                                                & $\Delta=\Delta(A,b) = \max\limits_{1\leq i\leq n}\delta_i^A $        & $\nabla=\nabla(G,d) = \max\limits_{1 \leq i \leq n} \nabla_i^G$        &                    \cellcolor[HTML]{C0C0C0}            $\nabla(G,d)=\Delta(A,b)$        \\ \hline

\multicolumn{1}{|l|}{\begin{tabular}[c]{@{}l@{}}Subsets of equation indexes\\ $\quad$\end{tabular}}  &{$N_{\text{cons}}\text{ and }N_{\text{inc}}$}& {$M_{\text{cons}}\text{ and }M_{\text{inc}}$}& {\cellcolor[HTML]{C0C0C0}
\begin{tabular}[c]{@{}l@{}}$M_{\text{cons}} = N_{\text{cons}} $\\ $M_{\text{inc}} = N_{\text{inc}} $
\end{tabular}} \\ \hline

\multicolumn{1}{|l|}{\begin{tabular}[c]{@{}l@{}}Set of matrices \\of consistent systems\\
with fixed right-hand side vector \end{tabular}}                                                                & $\mathcal{T}$        & $\mathcal{Q}$        &                    \cellcolor[HTML]{C0C0C0}            $\mathcal{Q} = \mathcal{T}^\circ$        \\ \hline

\multicolumn{1}{|l|}{\begin{tabular}[c]{@{}l@{}}Distance associated  with \\the matrix of the system\end{tabular}}                                                                & $\mathring{\Delta}_{p}$        & $\mathring{\nabla}_{p}$        &                    \cellcolor[HTML]{C0C0C0}            $\mathring{\Delta}_{p} = \mathring{\nabla}_{p}$        \\ \hline

\multicolumn{1}{|l|}{\begin{tabular}[c]{@{}l@{}}Auxiliary matrices\\ $\quad$\end{tabular}}  &{$A^{(i , j)}$}& {$G^{(i , j)}$}& {\cellcolor[HTML]{C0C0C0}$G^{(i , j)} = [A^{(i , j)}]^{\circ}$} \\ \hline

\multicolumn{1}{|l|}{\begin{tabular}[c]{@{}l@{}}Modified matrices\\ $\quad$\end{tabular}}  &{$A^{(\vec{i} , \vec{j})}$}& {$G^{(\vec{i} , \vec{j})}$}& {\cellcolor[HTML]{C0C0C0}$G^{(\vec{i} , \vec{j})} = [A^{(\vec{i} , \vec{j})}]^{\circ}$} \\ \hline

\multicolumn{1}{|l|}{\begin{tabular}[c]{@{}l@{}}Set of best pairs of vectors \\(associated with $L_\infty$ norm)\\ $\quad$\end{tabular}}  &{$L_{A,\infty}$}& {$P_{G,\infty}$}& {\cellcolor[HTML]{C0C0C0}$P_{G,\infty} = L_{A,\infty}$} \\ \hline
\end{tabular}
\vspace{1em}
\caption{Tools of the systems $A \Box_{\min}^{\max} x = b$ and $G \Box_{\max}^{\min} x = d$ and their relations iff $G = A^\circ$ and $d= b^\circ$.}
\label{tab:correspondences}
\end{table}

We introduce the following set, which is formed by  the matrices of size $(n,m)$ governing the  consistent min-max  systems which use the same right-hand side vector: the right-hand side vector $d$ of the min-max system $G \Box_{\max}^{\min} x = d$:

\begin{equation}\label{eq:setQ}
    {\cal Q} := \{Q\in [0 , 1]^{n \times m}\mid \text{the system }  Q \Box_{\max}^{\min} x = d  \text{ is consistent}\}.
\end{equation}

\noindent For each $p \in \{1,2,\infty\}$, we define the distance from the matrix $G$ to the set $\mathcal{Q}$ as:
\begin{equation}\label{eq:nablaAp}
    \mathring{\nabla}_{p} := \inf_{Q \in \mathcal{Q}} \|Q - G\|_p.
\end{equation}
Here, $\|\cdot\|_p$ denotes the matrix norm with exponent $p$. The set ${\cal Q}$ and the distance $\mathring{\nabla}_{p}$ are defined  for the min-max system $G \Box_{\max}^{\min} x = d$  analogously to the set $\mathcal{T}$ and the distance $\mathring{\Delta}_{p}$ for the max-min system $A \Box_{\min}^{\max} x = b$, see (\ref{eq:setT}) and (\ref{eq:deltaAp}).

\noindent
We rely on the max-min system $A \Box_{\min}^{ \max} x = b$ equivalent to the min-max system $G \Box_{\max}^{\min} x = d$ in the following sense:
\begin{definition}\label{def:equivsys}
    To the min-max  system $G \Box_{\max}^{\min} x = d$, see (\ref{eq:minmaxsys}),  we associate the following $\max$-$\min$  system:
\begin{equation}\label{eq:min-max/max-min}
 A \Box_{\min}^{ \max} x = b \quad \text{where} \quad A:= G^\circ  \quad \text{and} \quad b:= d^\circ.
\end{equation}
\end{definition}

\noindent
We notice that we have $\Vert A  - B \Vert_p = \Vert A^\circ - B^\circ \Vert_p$ for any matrices $A$ and $B$ of size $(n , m)$ (see Notation \ref{not:minmaxbasic}). Therefore, we    easily deduce from  (\ref{eq:cinq1}) and   Lemma \ref{lemma:DistanceDeltaApmin}   that: 

\begin{restatable}{lemma}{lemmanablaApmin}
\label{lemma:nablaApmin}
Set $A:= G^\circ$ and $b:= d^\circ$, then we have:
\begin{enumerate}
    \item $\mathcal{Q} = \mathcal{T}^\circ$  where $\mathcal{T}^\circ := \{T^\circ\,\mid\, T \in {\cal T}\}$
    , see (\ref{eq:setT}) and (\ref{eq:setQ}).
    \item $\mathring{\nabla}_{p} = \mathring{\Delta}_{p} = \min_{Q \in \mathcal{Q}} \|Q - G\|_p$ 
    , see (\ref{eq:deltaAp}) and  (\ref{eq:nablaAp}).
\end{enumerate}
\end{restatable}
\qed

\noindent
We reuse the sets of indices $I = \{1,2,\cdots,n\}$ and $J = \{1,2,\dots,m\}$ to put some notations:
\begin{notation}\label{not:inconsminmax}

  For all $(i ,j)\in I \times J$, using the matrix $G$ and the right-hand side vector $d$  of the system $G \Box_{\max}^{\min} x = d$ we put: 
  \begin{equation}\label{eq:vgij}
      V^G_{ij} := \{k\in I\,\mid \, \sigma_\epsilon(d_i , g_{kj} , d_k) > 0\},
  \end{equation}
  \noindent where $\sigma_\epsilon$ is defined in (\ref{eq:sigmaespi}). Thus, we have: 
  \begin{equation}\label{eq:VequalU}
   V^G_{ij} = U^A_{ij}   
  \end{equation}
  where   $U^A_{ij}$ is the set associated with the $\max$-$\min$ system (\ref{eq:min-max/max-min}) as in  (\ref{eq:UDij}).

From the formula for computing the Chebyshev distance $\nabla$ associated with the right-hand side vector $d$ of the system $G \Box_{\max}^{\min} x = d$, see (\ref{eq:nablapremier}), we set:
\begin{equation}\label{eq:mconsninc}
    M_{\text{cons}} := \{ i \in I \mid \nabla_i = 0\} \quad \text{ and } \quad  M_{\text{inc}} := \{ i \in I \mid \nabla_i > 0\} \text{ and } h := \text{card}(M_{\text{inc}}).
\end{equation}

Thus, we have:
\begin{equation}\label{eq:M=N}
M_{\text{cons}} = N_{\text{cons}} \quad \text{and} \quad  M_{\text{inc}}=N_{\text{inc}}  
\end{equation}
where $N_{\text{cons}}$ and $N_{\text{inc}}$ are   the sets associated with the max-min system (\ref{eq:min-max/max-min}) as in (\ref{eq:nconsninc}).
\end{notation}

For any pair $(i,j) \in I \times J$, we define an  auxiliary matrix $G^{(i,j)}$ using the set $V^G_{ij}$:
\begin{definition}
 \noindent For any pair $(i,j) \in I \times J$, we construct an auxiliary matrix $G^{(i,j)} = [g^{(i,j)}_{kl}]_{1 \leq k \leq n , 1 \leq l \leq m}$ by:
\begin{subequations}\label{eq:Gijklgroup}
\begin{equation}\label{eq:Gijkl1}
g^{(i,j)}_{ij} = \begin{cases}
    d_i \quad & \text{if } \quad      d_i < g_{ij}    \\
 g_{ij}   \quad & \text{if } \quad d_i \geq g_{ij}
\end{cases} \quad  \text{and for } \quad \forall k \in I\backslash\{i\}, \quad         g^{(i , j)}_{kj} = \begin{cases}
d_k  \quad &  \text{if } \quad k \in V^G_{ij}  \\
 g_{kj}   \quad & \text{if } \quad k \notin V^G_{ij}  
 \end{cases}.     
\end{equation}
\begin{equation}\label{eq:Gijkl2}
\forall k \in I \, \forall l\in J \backslash\{j\}, 
  \quad  g^{(i , j)}_{kl} = g_{kl}.   
\end{equation}
\end{subequations}   
\end{definition}

The above definitions of the coefficients of  the auxiliary matrix $G^{(i,j)}$  correspond to the definitions (\ref{eq:Aijkl1}) and (\ref{eq:Aijkl2}) of the coefficients of the auxiliary matrix $A^{(i,j)}$ associated with  the  max-min system   $A \Box_{\min}^{ \max} x = b$ (\ref{eq:min-max/max-min}): 
\begin{equation}\label{eq:auxGA}
\text{for all $(i , j)\in I \times J$ , we have } \quad   G^{(i,j)} = [A^{(i,j)}]^\circ. 
\end{equation}

We generalize the construction of auxiliary matrices $G^{(i , j)}$ associated with pairs $(i , j)\in I \times J$ to pairs of index uplets $(\vec i , \vec j)\in I^h \times J^h$  where $1 \leq h \leq n$:
\begin{definition}\label{def:minmaxvecivecj}
For any  $h$-tuple  $\vec{i} = (i_1, i_2, \dots, i_h)$ of two-by-two distinct indices of the set $I$ and an  $h$-tuple  $\vec{j} = (j_1, j_2, \dots, j_h)$ of   indices of the set $J$, we associate  the  matrix  denoted $G^{(\vec{i},\vec{j})}$ which is constructed by successively applying the modifications indicated by the pairs $(i_1,j_1), (i_2,j_2), \dots, (i_h,j_h)$, as follows:
\begin{equation}\label{eq:constructionGijvec}
G^{(\vec{i} , \vec{j})} =  \begin{cases}
 G^{(i_1 , j_1)}    &   \text{if }       h = 1,\hfill     \\
 Y^{(i_h , j_h)}   & \text{if } h > 1 \,\text{and} \, Y := G^{(i_1 , j_1) \ast (i_2 , j_2) \ast \dots \ast(i_{h-1} , j_{h-1})}.
 \end{cases}
 \end{equation}
\end{definition}
  This construction can also be written as
$$
G^{(\vec{i},\vec{j})} = (((G^{(i_1,j_1)})^{(i_2,j_2)})^{(i_3,j_3)} \cdots )^{(i_h,j_h)}.
$$ 

By induction on $h$, we easily generalize  (\ref{eq:auxGA}):
\begin{equation}\label{eq:AuxGA}
  G^{(\vec i,\vec j)} = [A^{(\vec i,\vec j)}]^\circ
\end{equation}
where $A^{(\vec i,\vec j)}$  is introduced  in Definition \ref{def:vecivecj}.

We give a simpler construction of the matrices $G^{(\vec{i}, \vec{j})}$. For any  $h$-tuple  $\vec{i} = (i_1, i_2, \dots, i_h)$ of two-by-two distinct indices of the set $I$ and an  $h$-tuple  $\vec{j} = (j_1, j_2, \dots, j_h)$ of   indices of the set $J$, using the sets $V^G_{ij}$ defined in  (\ref{eq:vgij}),  we construct the following sets: 
\begin{subequations}\label{eq:fhminusonesets}
\begin{flalign}
 {\cal F}_1 &:= \{ (i_\lambda, j_\lambda) \in I \times J \mid \lambda \in \{1,\dots,h\} \text{ and }  g_{i_\lambda j_\lambda} > d_{i_\lambda} \},\\[1mm]
 {\cal F}_2 &:= \bigcup_{\lambda=1}^h \{ (k, j_\lambda) \mid k \in V^G_{i_\lambda j_\lambda} \} \subseteq I \times J,\\[1mm]
  {\cal F} &:=  {\cal F}_1 \cup  {\cal F}_2 \subseteq I \times J.
\end{flalign}
\end{subequations}
\noindent Using the equalities $V^G_{ij} = U^A_{ij}$ for all $(i,j) \in I \times J$, see (\ref{eq:VequalU}), we deduce:
\begin{equation}\label{eq:F=E}
{\cal F}_1 = {\cal E}_1, \quad  {\cal F}_2 = {\cal E}_2  \quad \text{and} \quad 
{\cal F} = {\cal E}
 \end{equation}

 where the sets $\mathcal{E}_1, \mathcal{E}_2$ and $\mathcal{E}$, see (\ref{eq:combinedE1E2Eh}), are associated with the  max-min system   $A \Box_{\min}^{ \max} x = b$ equivalent to the min-max  $G \Box_{\max}^{\min} x = d$ (see Definition \ref{def:equivsys}). Thus, using  (\ref{eq:AuxGA}) and  Proposition \ref{proposition:SimpleDefAijvec}, we easily get:
\begin{proposition}\label{eq:AuxGF}
For any  $h$-tuple  $\vec{i} = (i_1, i_2, \dots, i_h)$ of two-by-two distinct indices of the set $I$ and an  $h$-tuple  $\vec{j} = (j_1, j_2, \dots, j_h)$ of   indices of the set $J$, we have:    \begin{equation}\label{eq:constructionGijvecWithF}
    \forall (k , l)\in I \times J, \, g^{(\vec i , \vec j)}_{kl} = 
\begin{cases}
 d_{k} \quad & \text{if } \quad      (k , l) \in {\cal F}    \\
 g_{kl}   \quad & \text{if } \quad  (k , l) \notin {\cal F}
 \end{cases}.
\end{equation} 
\end{proposition}
\qed
\noindent

\noindent Using the  max-min system   $A \Box_{\min}^{ \max} x = b$ equivalent to the min-max system  $G \Box_{\max}^{\min} x = d$ (see Definition \ref{def:equivsys}), we deduce the main results for the min-max system  $G \Box_{\max}^{\min} x = d$:
\begin{restatable}{theorem}{theoremminmaxtheoOneConsistentGij}
\label{theorem:minmaxtheoOneConsistentGij}
We  use the set $M_{\text{inc}}$ where  $h := \operatorname{card}(M_{\text{inc}})$, see (\ref{eq:mconsninc}). Choose an arbitrary enumeration of $M_{\text{inc}}$ as $\{i_1, i_2, \dots, i_h\}$ and set $\vec{i} = (i_1, i_2, \dots, i_h)$. Then, for every $h$-tuple $\vec{j} = (j_1, j_2, \dots, j_h) \in J^h$, the matrix $G^{(\vec{i},\vec{j})}$ constructed with (\ref{eq:constructionGijvec}) yields a consistent system, which is: 
$$
G^{(\vec{i},\vec{j})} \Box_{\max}^{\min} x = d.
$$ 
 Therefore, the matrix $G^{(\vec{i},\vec{j})}$ belongs to the set $\mathcal{Q}$, see (\ref{eq:setQ}), i.e., $G^{(\vec{i},\vec{j})} \in \mathcal{Q}$.\\
For $p\in\{1 , 2 , \infty\}$, we have: $$\Vert G^{(\vec{i},\vec{j})} - G \Vert_p \geq \mathring{\nabla}_{p},$$ see (\ref{eq:nablaAp}) for the definition of $\mathring{\nabla}_{p}$. 
\end{restatable}
The proof relies on  (\ref{eq:M=N}),  (\ref{eq:AuxGA}) and  Theorem \ref{theorem:theo1consSij}.
\qed

\noindent
The following result highlights the minimality of the modifications made to obtain the  matrices   $G^{(\vec{i},\vec{j})}$  from $G$ using (\ref{eq:constructionGijvec}) and an $h$-tuple $\vec i$ as in Theorem \ref{theorem:minmaxtheoOneConsistentGij}:
\begin{restatable}{theorem}{theoremtheoTwoGijnorm}
\label{theorem:theo2Gijnorm}
Retain the notations from Theorem~\ref{theorem:minmaxtheoOneConsistentGij} and let 
\[
G = [g_{kl}]_{1 \le k \le n,\; 1 \le l \le m} \quad \text{and} \quad 
 \vec{i} = (i_1, i_2, \dots, i_h) \quad \text{such that } \quad M_{\text{inc}}=\{i_1, i_2, \dots, i_h\}.\]
Then, for every consistent min-max system whose right-hand side vector is $d$:
\[
T' \Box_{\max}^{\min} x = d,\quad \text{with } T' = [t'_{kl}]_{1 \le k \le n,\; 1 \le l \le m} \in \mathcal{Q},
\]
there exists an $h$-tuple $\vec{j} = (j_1, j_2, \dots, j_h) \in J^h$ such that the system 
\[
G^{(\vec{i},\vec{j})} \Box_{\max}^{\min} x = d
\]
is consistent and, for every $(k,l) \in I \times J$, 
\[
|g^{(\vec{i},\vec{j})}_{kl} - g_{kl}| \le |t'_{kl} - g_{kl}| \quad \text{where} \quad 
G^{(\vec{i},\vec{j})} = [g^{(\vec{i},\vec{j})}_{kl}]_{1 \le k \le n,\; 1 \le l \le m}.
\]
We have:  $$\Vert G^{(\vec{i},\vec{j})} - G \Vert_p \leq \Vert T' - G \Vert_p\, \quad \text{ for } \quad p\in\{1 , 2 , \infty\}.$$ 
\end{restatable}
The proof relies on  (\ref{eq:M=N}),  (\ref{eq:AuxGA}) and  Theorem \ref{theorem:theo2Aijnorm}.
\qed

\noindent
The above theorem  allows us to compute the distance $\mathring{\nabla}_{p}$, see (\ref{eq:nablaAp}), using the finite set $J^h$ instead of $\mathcal Q$:
\begin{restatable}{corollary}{corollarynormminJhG}
\label{corollary:normminJhG}
Recall that for $p \in \{1,2,\infty\}$ we defined in (\ref{eq:nablaAp}):
\[
\mathring{\nabla}_{p} := \inf_{Q \in \mathcal{Q}} \|Q - G\|_p.
\]
Then, it follows from Theorems~\ref{theorem:minmaxtheoOneConsistentGij} and \ref{theorem:theo2Gijnorm} that
\begin{equation}\label{eq:nablaAinfty}
\mathring{\nabla}_{p} = \min_{\vec{j} \in J^h} \|G^{(\vec{i},\vec{j})} - G\|_p.    
\end{equation}
Moreover, this minimum is independent of the particular enumeration $\{i_1, i_2, \dots, i_h\}$ of the set $M_{\text{inc}}$ chosen to define the $h$-tuple $\vec{i} = (i_1, i_2, \dots, i_h)$.
 \end{restatable}
\qed

\noindent In the rest of this section, using the $L_\infty$ norm i.e, $p = \infty$, we introduce a method for constructing a matrix 
$G^{(\vec{i},\vec{j})}
$, see (\ref{eq:constructionGijvec}), 
such that
\begin{equation}\label{eq:satisfylinftyG}
  \mathring{\nabla}_{\infty}=   \|G^{(\vec{i},\vec{j})} - G\|_\infty  ,
\end{equation}

\noindent  This construction will allow us to introduce an explicit analytical formula (Corollary \ref{corollary:formulaforLinftydistG}) for computing  the distance $\mathring{\nabla}_{\infty}$, see (\ref{eq:nablaAp}).

We rely on the  max-min system   $A \Box_{\min}^{ \max} x = b$ equivalent to the min-max  $G \Box_{\max}^{\min} x = d$ (see Definition \ref{def:equivsys}). We note that we have  $\mathring{\nabla}_{\infty} = \mathring{\Delta}_{\infty}$  
 where $\mathring{\Delta}_{\infty}$    is the distance associated with the system $A \Box_{\min}^{\max} x = b$, see (\ref{eq:deltaAp}).

\noindent
To find a pair $(\vec{i},\vec{j})$ satisfying the equality (\ref{eq:satisfylinftyG}), 
we process as follows:
\begin{itemize}
    \item We choose an arbitrary enumeration of $M_{\text{inc}}$ as $\{i_1, i_2, \dots, i_h\}$ , see (\ref{eq:mconsninc}), and set $\vec{i} = (i_1, i_2, \dots, i_h)$.

    \item To the $h$-uplet $\vec{i} := (i_1, \dots, i_h)$, we associate an $h$-uplet $\vec{j}  = (j_1, j_2, \dots, j_h)$ whose components satisfy the following requirement: for each $\lambda \in \{1,2,\dots,h\}$, set 
    \begin{equation}\label{eq:levecjG}
    j_\lambda := \operatorname{argmin}_{j \in J} \| G^{(i_\lambda, j)} - G \|_\infty, \text{ which means that } \| G^{(i_\lambda, j_\lambda)} - G \|_\infty = \min_{j \in J} \| G^{(i_\lambda, j)} - G \|_\infty.
    \end{equation}
\end{itemize}
The above method allows us to  introduce the following set of pairs $(\vec{i}, \vec{j})$ : 

\begin{equation}\label{eq:lAinftyG}
P_{G,\infty} = \Bigl\{ (\vec{i},\vec{j}) \,\Bigm|\, 
\begin{array}{l}
\vec{i} = (i_1,\dots,i_h) \text{ is any ordering of } M_{\text{inc}},\\[1mm]
\vec{j} = (j_1,\dots,j_h) \text{ with } j_\lambda \in \mathcal{J}'(i_\lambda) \text{ for each } \lambda=1,\dots,h
\end{array}
\Bigr\},    
\end{equation}

\noindent where $\mathcal{J}': I \to 2^J: i \rightarrow  \{ j \in J \mid \|G^{(i,j)} - G\|_\infty = \min_{j' \in J}\|G^{(i,j')} - G\|_\infty \}$. \\

\noindent
From (\ref{eq:M=N}) and (\ref{eq:auxGA}), we easily deduce the equality 
$P_{G,\infty} = L_{A,\infty}$ where $L_{A,\infty}$ is define in (\ref{eq:lAinfty}).

From Theorem \ref{theorem:minmaxtheoOneConsistentGij}, we know that  for any $(\vec{i},\vec{j}) \in P_{G,\infty}$, we have $G^{(\vec{i},\vec{j})} \in {\cal Q}$, i.e., 
the system $G^{(\vec{i},\vec{j})} \Box_{\min}^{\max} x = d$ is consistent and therefore $\mathring{\nabla}_{\infty} \leq   \|G^{(\vec{i},\vec{j})} - G\|_\infty$. We easily deduce
from  the second statment of Lemma \ref{lemma:nablaApmin} , (\ref{eq:AuxGA}) and  Theorem \ref{theorem:formulaforLinftydistance},    
the following theorem which is the main step  to compute $\mathring{\nabla}_{\infty}$ by an explicit analytical formula:   

\begin{restatable}{theorem}{theoremformulaforLinftydistG}
\label{theorem:formulaforLinftydistanceG}
Let $G \Box_{\max}^{\min} x = d$ be an inconsistent system. Then, 
for any $(\vec i , \vec j)\in P_{G,\infty}$, we have: 
\begin{enumerate}
    \item $\mathring{\nabla}_{\infty} = \Vert G^{(\vec i , \vec j)}  - G \Vert_\infty.$
\item $\mathring{\nabla}_{\infty} = \max_{1 \leq \lambda \leq h}\, \Vert G^{(i_\lambda , j_\lambda)}  - G \Vert_\infty $ where $\vec i = (i_1 , \dots , i_h)$ and $\vec j = (j_1 , \dots , j_h).$
\end{enumerate}
\end{restatable}
\qed

\noindent We give an explicit analytical formula for computing the distance $\mathring{\nabla}_{\infty}$ using Corollary \ref{corollary:formulaforLinftydist}:
\begin{restatable}{corollary}{corollaryformulaforLinftydistG}
\label{corollary:formulaforLinftydistG}
For any system $G \Box_{\max}^{\min} x = d$, 
 we have:
\begin{equation}\label{eq:formulaforLinftydistanceG}
    \mathring{\nabla}_{\infty} = \max_{i\in I} \, \min_{j\in J} \, \max\Bigl[(g_{ij} - d_i)^+, \, \max_{k \in I, k\neq i}\Bigl(\theta'(i,k) \cdot (d_k - g_{kj})^+\Bigr)\Bigr]
\end{equation}

where $\theta': I \times I \mapsto \{0,1\}: \theta'(i,k) = \begin{cases}
1 & \text{ if } d_i < d_k\\
0 & \text{ otherwise}
\end{cases}.$

The formula (\ref{eq:formulaforLinftydistanceG}) can be reformulated using the set $M_{\text{inc}}$, see (\ref{eq:mconsninc}):
\[
\mathring{\nabla}_{\infty} = \max_{i\in M_{\text{inc}}} \, \min_{j\in J} \, \max\Bigl[(g_{ij} - d_i)^+, \, \max_{k \in I, k\neq i}\Bigl(\theta'(i,k) \cdot (d_k - g_{kj})^+\Bigr)\Bigr].
\]
\end{restatable}
\qed

The distance  $\mathring{\nabla}_{\infty}$ is useful for estimating to which extent the min-max system $G \Box_{\max}^{\min} x = d$ is inconsistent, with respect to its matrix $G$. Analogously, using $\nabla$, see (\ref{eq:nablapremier}), we can measure how inconsistent the min-max system $G \Box_{\max}^{\min} x = d$ is with respect to its right-hand side vector $d$. It is easy to see that we have $\nabla \leq \mathring{\nabla}_{\infty}$ for the min-max system $G \Box_{\max}^{\min} x = d$.

\section{Conclusion}
In this article, we handled the inconsistency of a system of max-min fuzzy relational equations by modifying the matrix governing the system.  Our main results (Theorem \ref{theorem:theo1consSij} and Theorem \ref{theorem:theo2Aijnorm}) allows us to obtain 
consistent systems that closely approximate the original inconsistent system in the following sense: the right-hand side vector of each consistent system is that of the inconsistent system, and the coefficients of the matrix governing each consistent system are obtained by modifying, exactly and minimally, the entries of the original matrix that must be corrected to achieve consistency, while leaving all other entries unchanged.  The solutions of the obtained consistent systems can be considered as approximate solutions of the  inconsistent system. We studied the distance  $\mathring{\Delta}_{p}$ (in terms of a norm among $L_1$, $L_2$ or $L_\infty$) between the matrix of the inconsistent system and the set formed by the matrices of consistent systems that use the same right-hand side vector as the inconsistent system, see (\ref{eq:deltaAp}). We showed in Theorem \ref{theorem:formulaforLinftydistance}  that our method allows us to directly compute  matrices of consistent systems that use the same right-hand side vector as the inconsistent system whose distance in terms of $L_\infty$ norm  to the matrix of the inconsistent system is minimal (the computational costs are higher when using $L_1$ norm or $L_2$ norm, see Corollary \ref{corollary:normminJh}).   We also give an explicit analytical formula for computing  $\mathring{\Delta}_{\infty}$, see Corollary \ref{corollary:formulaforLinftydist}. Finally, we translated our results for min-max systems.

This work may be useful for solving inconsistency issues in max-min systems involved in max-min learning methods, such as learning the rule parameters of possibilistic rule-based systems \cite{baaj2022learning}, learning associate memories \cite{sussner2006implicative}, or learning  capacities of Sugeno integrals according to training data  \cite{baaj2024Sug}. It can also be useful for applications based on max-min systems, e.g. spatial analysis \cite{di2011spatial}, diagnostic problems \cite{dubois1995fuzzy}.

\clearpage
\bibliographystyle{abbrvnat} 
\bibliography{ref} 

\begin{thebibliography}{35}
\providecommand{\natexlab}[1]{#1}
\providecommand{\url}[1]{\texttt{#1}}
\expandafter\ifx\csname urlstyle\endcsname\relax
  \providecommand{\doi}[1]{doi: #1}\else
  \providecommand{\doi}{doi: \begingroup \urlstyle{rm}\Url}\fi

\bibitem[Adlassnig(1980)]{adlassnig1980fuzzy}
K.-P. Adlassnig.
\newblock A fuzzy logical model of computer-assisted medical diagnosis.
\newblock \emph{Methods of Information in Medicine}, 19:\penalty0 141--148,
  1980.
\newblock ISSN 0026-1270.

\bibitem[Baaj(2022)]{baaj2022learning}
I.~Baaj.
\newblock Learning rule parameters of possibilistic rule-based system.
\newblock In \emph{2022 IEEE International Conference on Fuzzy Systems
  (FUZZ-IEEE)}, pages 1--8. IEEE, 2022.

\bibitem[{Baaj}(2024)]{baaj2024Sug}
I.~{Baaj}.
\newblock {On learning capacities of Sugeno integrals with systems of fuzzy
  relational equations}.
\newblock \emph{arXiv e-prints}, art. arXiv:2408.07768, Aug. 2024.

\bibitem[Baaj(2024{\natexlab{a}})]{baaj2024maxProdMaxLuka}
I.~Baaj.
\newblock On the handling of inconsistent systems based on max-product and
  max-lukasiewicz compositions.
\newblock \emph{Fuzzy Sets and Systems}, page 109142, 2024{\natexlab{a}}.
\newblock ISSN 0165-0114.

\bibitem[Baaj(2024{\natexlab{b}})]{baaj2024maxmin}
I.~Baaj.
\newblock On the handling of inconsistent systems of max-min fuzzy relational
  equations.
\newblock \emph{Fuzzy sets and systems}, 482:\penalty0 108912,
  2024{\natexlab{b}}.
\newblock ISSN 0165-0114.

\bibitem[Bandler and Kohout(1980)]{bandler1980semantics}
W.~Bandler and L.~J. Kohout.
\newblock Semantics of implication operators and fuzzy relational products.
\newblock \emph{International Journal of Man-Machine Studies}, 12\penalty0
  (1):\penalty0 89--116, 1980.
\newblock ISSN 0020-7373.

\bibitem[Bandler and Kohout(1986)]{bandler1986survey}
W.~Bandler and L.~J. Kohout.
\newblock \emph{A Survey of Fuzzy Relational Products in Their Applicability to
  Medicine and Clinical Psychology}, pages 107--118.
\newblock Abacus Books. Gordon and Breach, London, 1986.
\newblock ISBN 9780856264429.

\bibitem[Cechlárová(2000)]{CECHLAROVA2000123}
K.~Cechlárová.
\newblock A note on unsolvable systems of max–min (fuzzy) equations.
\newblock \emph{Linear Algebra and its Applications}, 310\penalty0
  (1):\penalty0 123--128, 2000.
\newblock ISSN 0024-3795.

\bibitem[Cimler et~al.(2018)Cimler, Gavalec, and
  Zimmermann]{cimler2018optimization}
R.~Cimler, M.~Gavalec, and K.~Zimmermann.
\newblock An optimization problem on the image set of a (max, min) fuzzy
  operator.
\newblock \emph{Fuzzy Sets and Systems}, 341:\penalty0 113--122, 2018.
\newblock ISSN 0165-0114.

\bibitem[Cuninghame{-}Green and
  Cechl{\'a}rov{\'a}(1995)]{cuninghame1995residuation}
R.~A. Cuninghame{-}Green and K.~Cechl{\'a}rov{\'a}.
\newblock Residuation in fuzzy algebra and some applications.
\newblock \emph{Fuzzy Sets and Systems}, 71\penalty0 (2):\penalty0 227--239,
  1995.
\newblock ISSN 0165-0114.

\bibitem[De~Baets(2000)]{baets2000analytical}
B.~De~Baets.
\newblock \emph{Analytical Solution Methods for Fuzzy Relational Equations},
  pages 291--340.
\newblock Springer US, Boston, MA, 2000.
\newblock ISBN 978-1-4615-4429-6.

\bibitem[Di~Martino and Sessa(2011)]{di2011spatial}
F.~Di~Martino and S.~Sessa.
\newblock Spatial analysis and fuzzy relation equations.
\newblock \emph{Advances in Fuzzy Systems}, 2011:\penalty0 6--6, 2011.

\bibitem[Di~Nola et~al.(1982)Di~Nola, Pedrycz, and Sessa]{di1982solution}
A.~Di~Nola, W.~Pedrycz, and S.~Sessa.
\newblock On solution of fuzzy relational equations and their characterization.
\newblock \emph{Busefal}, 12:\penalty0 60--71, 1982.

\bibitem[Di~Nola et~al.(1984)Di~Nola, Pedrycz, Sessa, and Zhuang]{di1984fuzzy}
A.~Di~Nola, W.~Pedrycz, S.~Sessa, and W.~P. Zhuang.
\newblock Fuzzy relation equation under a class of triangular norms: A survey
  and new results.
\newblock \emph{Stochastica}, 8\penalty0 (2):\penalty0 99--145, 1984.
\newblock ISSN 0210-2930.

\bibitem[Di~Nola et~al.(1991)Di~Nola, Pedrycz, Sessa, and Sanchez]{di1991fuzzy}
A.~Di~Nola, W.~Pedrycz, S.~Sessa, and E.~Sanchez.
\newblock Fuzzy relation equations theory as a basis of fuzzy modelling: An
  overview.
\newblock \emph{Fuzzy sets and systems}, 40\penalty0 (3):\penalty0 415--429,
  1991.
\newblock ISSN 0165-0114.

\bibitem[Dubois and Prade(1995)]{dubois1995fuzzy}
D.~Dubois and H.~Prade.
\newblock Fuzzy relation equations and causal reasoning.
\newblock \emph{Fuzzy sets and systems}, 75\penalty0 (2):\penalty0 119--134,
  1995.
\newblock ISSN 0165-0114.

\bibitem[Li(2009)]{li2009fuzzy}
P.~Li.
\newblock \emph{Fuzzy Relational Equations: Resolution and Optimization}.
\newblock PhD thesis, North Carolina State University, 2009.

\bibitem[Li(2024)]{li2024linear}
P.~Li.
\newblock Linear optimization over the approximate solutions of a system of
  max-min equations.
\newblock \emph{Fuzzy Sets and Systems}, 484:\penalty0 108946, 2024.
\newblock ISSN 0165-0114.

\bibitem[Li and Fang(2008)]{li2008resolution}
P.~Li and S.-C. Fang.
\newblock On the resolution and optimization of a system of fuzzy relational
  equations with sup-t composition.
\newblock \emph{Fuzzy Optimization and Decision Making}, 7\penalty0
  (2):\penalty0 169--214, 2008.

\bibitem[Li and Fang(2010)]{li2010chebyshev}
P.~Li and S.-C. Fang.
\newblock Chebyshev approximation of inconsistent fuzzy relational equations
  with max-t composition.
\newblock In \emph{Fuzzy Optimization}, pages 109--124. Springer, 2010.

\bibitem[Markovskii(2005)]{markovskii2005relation}
A.~Markovskii.
\newblock On the relation between equations with max-product composition and
  the covering problem.
\newblock \emph{Fuzzy Sets and Systems}, 153\penalty0 (2):\penalty0 261--273,
  2005.

\bibitem[Matusiewicz and Drewniak(2013)]{matusiewicz2013increasing}
Z.~Matusiewicz and J.~Drewniak.
\newblock Increasing continuous operations in fuzzy max-* equations and
  inequalities.
\newblock \emph{Fuzzy Sets and Systems}, 232:\penalty0 120--133, 2013.

\bibitem[Miyakoshi and Shimbo(1985)]{MIYAKOSHI198553}
M.~Miyakoshi and M.~Shimbo.
\newblock Solutions of composite fuzzy relational equations with triangular
  norms.
\newblock \emph{Fuzzy Sets and Systems}, 16\penalty0 (1):\penalty0 53--63,
  1985.
\newblock ISSN 0165-0114.

\bibitem[Pedrycz(1982)]{pedrycz1982fuzzy}
W.~Pedrycz.
\newblock Fuzzy relational equations with triangular norms and their
  resolutions.
\newblock \emph{Busefal}, 11:\penalty0 24--32, 1982.

\bibitem[Pedrycz(1985)]{pedrycz1985applications}
W.~Pedrycz.
\newblock Applications of fuzzy relational equations for methods of reasoning
  in presence of fuzzy data.
\newblock \emph{Fuzzy sets and systems}, 16\penalty0 (2):\penalty0 163--175,
  1985.

\bibitem[Pedrycz(1990)]{pedrycz1990inverse}
W.~Pedrycz.
\newblock Inverse problem in fuzzy relational equations.
\newblock \emph{Fuzzy Sets and systems}, 36\penalty0 (2):\penalty0 277--291,
  1990.

\bibitem[Perfilieva and Noskov{\'a}(2008)]{perfilieva2008system}
I.~Perfilieva and L.~Noskov{\'a}.
\newblock System of fuzzy relation equations with inf-→ composition: Complete
  set of solutions.
\newblock \emph{Fuzzy Sets and Systems}, 159\penalty0 (17):\penalty0
  2256--2271, 2008.

\bibitem[Sagan(2013)]{sagan2013symmetric}
B.~E. Sagan.
\newblock \emph{The symmetric group: representations, combinatorial algorithms,
  and symmetric functions}, volume 203.
\newblock Springer Science \& Business Media, 2013.
\newblock Chapter 1, page 4.

\bibitem[Sanchez(1976)]{sanchez1976resolution}
E.~Sanchez.
\newblock Resolution of composite fuzzy relation equations.
\newblock \emph{Information and control}, 30\penalty0 (1):\penalty0 38--48,
  1976.

\bibitem[Sanchez(1993)]{sanchez1977}
E.~Sanchez.
\newblock Solutions in composite fuzzy relation equations: Application to
  medical diagnosis in brouwerian logic.
\newblock In \emph{Readings in Fuzzy Sets for Intelligent Systems}, pages
  159--165. Morgan Kaufmann, 1993.
\newblock ISBN 978-1-4832-1450-4.

\bibitem[Shieh(2007)]{shieh2007solutions}
B.-S. Shieh.
\newblock Solutions of fuzzy relation equations based on continuous t-norms.
\newblock \emph{Information Sciences}, 177\penalty0 (19):\penalty0 4208--4215,
  2007.

\bibitem[Sussner and Valle(2006)]{sussner2006implicative}
P.~Sussner and M.~E. Valle.
\newblock Implicative fuzzy associative memories.
\newblock \emph{IEEE Transactions on Fuzzy Systems}, 14\penalty0 (6):\penalty0
  793--807, 2006.

\bibitem[Wen et~al.(2023)Wen, Wu, and Li]{WEN2022}
C.-F. Wen, Y.-K. Wu, and Z.~Li.
\newblock Algebraic formulae for solving systems of max-min inverse fuzzy
  relational equations.
\newblock \emph{Information Sciences}, 622:\penalty0 1162--1183, 2023.
\newblock ISSN 0020-0255.

\bibitem[Wu et~al.(2022)Wu, Lur, Wen, and Lee]{wu2022analytical}
Y.-K. Wu, Y.-Y. Lur, C.-F. Wen, and S.-J. Lee.
\newblock Analytical method for solving max-min inverse fuzzy relation.
\newblock \emph{Fuzzy Sets and Systems}, 440:\penalty0 21--41, 2022.
\newblock ISSN 0165-0114.

\bibitem[Zadeh(1979)]{ZadehTH}
L.~A. Zadeh.
\newblock A theory of approximate reasoning.
\newblock \emph{Machine intelligence}, 9:\penalty0 149--194, 1979.

\end{thebibliography}
\appendix

\clearpage

\section{Proofs}
\label{sec:proofs}

\subsection{Proof of the result in Section \ref{sec:minimallymod}}
\label{sec:proof:minimallymod}
\lemmaDistanceDeltaApmin*
\begin{proof}
By applying (\ref{eq:Delta}) and (\ref{eq:equivDelta}) to any matrix $T$, we deduce the equality:   $$\mathcal{T} = \{T \in [0,1]^{n \times m} \mid \text{ for all } i \in \{1,2,\dots,n\} \text{ we have }  \delta^T_i = 0 \}.$$  
Using Lemma \ref{lemma:cont}, we conclude that  $\mathcal{T}$ 
is a non-empty closed   subset of $[0,1]^{n \times m}$. Since $[0,1]^{n \times m}$ is compact and the norm $\|\cdot\|_p$ is continuous, the function $T \mapsto \|T-A\|_p$ attains its minimum on $\mathcal{T}$. Therefore,
$$
\mathring{\Delta}_{p} = \min_{T \in \mathcal{T}} \|T-A\|_p.
$$
\end{proof}

\subsection{Proof of the results in Section \ref{sec:auxiliarymatrix}}
\label{sec:proofs:auxiliarymatrix}

\lemmadeltaaijijequalzero*
\begin{proof}
Let $A^{(i,j)} = [a^{(i,j)}_{kl}]$. We remind that: 
$$\delta^{A^{(i,j)}}(s , l) = \max[ (b_s - a^{(i,j)}_{sl})^+,  \max_{1 \leq k \leq n}\,  \sigma_G\,(b_s, a^{(i,j)}_{kl}, b_k)]$$ 
and for every $k$ we have  $a^{(i,j)}_{kj} \in \{ b_k, a_{kj} \}$, see (\ref{eq:aijklisbkorakl}). 

\begin{enumerate}

\item First, we show that $b_i \le a^{(i,j)}_{ij}$. This inequality clearly follows from 
$$
a^{(i,j)}_{ij} = \begin{cases}
b_i & \text{if } b_i > a_{ij} \\[1mm]
a_{ij} & \text{if } b_i \le a_{ij}
\end{cases},
  $$
see (\ref{eq:Aijkl1}).
Moreover, for every $k$ we have: 
$$
\sigma_G(b_i, a^{(i,j)}_{kj}, b_k) = 0.
$$
This is evident (\ref{eq:sigmaG}) for $k = i$. For $k \neq i$, by (\ref{eq:Aijkl1}), we have:
$$
a^{(i,j)}_{kj} = \begin{cases}
b_k & \text{if } k \in U^A_{ij} \\[1mm]
a_{kj} & \text{if } k \notin U^A_{ij}
\end{cases}.
$$
If $k \in U^A_{ij}$ then 
$$
\sigma_G(b_i, a^{(i,j)}_{kj}, b_k) = \sigma_G(b_i, b_k, b_k) = 0.
$$
and if $k \notin U^A_{ij}$ then 
$$
\sigma_G(b_i, a^{(i,j)}_{kj}, b_k) = \sigma_G(b_i, a_{kj}, b_k),
$$
and by the definition  (\ref{eq:UDij}) of $U^A_{ij}$ we have 
$$
\sigma_G(b_i, a_{kj}, b_k) = 0.
$$
Thus, $\delta^{A^{(i,j)}}(i,j) = 0$.

\item Next, suppose that for $1 \le s \le n$ we have: 
$$
\delta^A(s,j) :=  \max[ (b_s - a_{sj})^+,  \max_{1 \leq k \leq n}\,  \sigma_G\,(b_s, a_{kj}, b_k)] = 0.
$$
We show that $\delta^{A^{(i,j)}}(s,j) = 0$. We may assume that $s \neq i$.

\noindent
As we have   $b_s \le a_{sj}$ and $\sigma_G(b_s, a_{kj}, b_k) = 0 $ for all $k\in I$, 
 to establish the inequality $b_s \le a^{(i,j)}_{sj}$ and that we have $\sigma_G(b_s, a^{(i,j)}_{kj}, b_k) = 0$ for all $k\in I$,  
it suffices to use (\ref{eq:sigmaG}) and  note   that $
a^{(i,j)}_{kj} \in \{ b_k, a_{kj} \}$, see (\ref{eq:aijklisbkorakl}).

\item By (\ref{eq:Aijkl2}),   for every $l \neq j$,    we have 
$a^{(i,j)}_{kl} = a_{kl} \text{ for all } k\in I.$ Thus by definition (\ref{eq:deltaTij}), we have: 
$$\delta^{A^{(i,j)}}(s,l) = \delta^A(s,l).$$

\item This equality is a consequence of the definition     of the coefficients of the matrix $A^{(i,j)}$, see (\ref{eq:Aijkl1}) and (\ref{eq:Aijkl2}),   and of the following relationship based on (\ref{eq:UDij1}):
\[ \max_{k\in I , k \not= i} \theta(i , k)\cdot (a_{kj} - b_k)^+  =
\max_{k\in U^A_{ij}} \mid a^{(i,j)}_{kj} - a_{kj}\mid \quad \text{with} \quad \max_\emptyset = 0.\]
 \end{enumerate}
\end{proof}

\lemmamatrixTwithdetaTijzero*
\begin{proof}
It is clear that the first statement implies the second one.

To prove the first statement, 
It suffices to verify the inequality for $l = j$ (see (\ref{eq:Aijklgroup}b)) and for $k \in I$ such that $a^{(i,j)}_{kj} \neq a_{kj}$.

\noindent
Using the hypothesis $\delta^T(i,j) = 0$, 
  i.e., $b_i \le t_{ij}$ and $\sigma_G(b_i,t_{kj},b_k) = 0$ for all $k \in I$, see (\ref{eq:deltaTij}), we observe:

\begin{itemize}
    \item If $k=i$, then from (\ref{eq:Aijklgroup}a) we deduce that $b_i > a_{ij}$ and $b_i = a^{(i,j)}_{ij}$, hence
    $$
    a_{ij} < b_i = a^{(i,j)}_{ij} \le t_{ij} \quad \text{which implies} \quad |a^{(i,j)}_{ij} - a_{ij}| \le |t_{ij} - a_{ij}|.
    $$
    
    \item If $k\neq i$, then by (\ref{eq:Aijklgroup}a) we have $k \in U^A_{ij}$, i.e.,
    $$
    \sigma_G(b_i,a_{kj},b_k) = \min\Bigl(\frac{(b_i-b_k)^+}{2},(a_{kj}-b_k)^+\Bigr) > 0,
    $$
    and $a^{(i,j)}_{kj} = b_k$. The hypothesis $\sigma_G(b_i,t_{kj},b_k) = 0$ then implies that $t_{kj} \le b_k$, so that
    $$
    t_{kj} \le b_k = a^{(i,j)}_{kj} < a_{kj} \quad \text{which implies} \quad |a^{(i,j)}_{kj} - a_{kj}| \le |t_{kj} - a_{kj}|.
    $$
\end{itemize}    
\end{proof}

\lemmacommutativeAij*
\begin{proof}
\noindent
Note that the matrix $A^{(i,j)}$ is obtained by modifying the column with index $j$ in the matrix $A$ (see (\ref{eq:Aijklgroup}a) and (\ref{eq:Aijklgroup}b)). Consequently, if $j_1 \neq j_2$, with such modification,    it is clear that the equality $D = F$ holds.

In the case where $j_1 = j_2$, we set $j := j_1 = j_2$ and show $D = F$ by computing the coefficients explicitly. We introduce the following notations for the matrices $C, D, E$, and $F$ of size $(n,m)$:
\begin{itemize}
    \item $C := A^{(i_1, j)} = [c_{kl}]$ and $D := C^{(i_2, j)} = [d_{kl}]$.
    \item $E := A^{(i_2, j)} = [e_{kl}]$ and $F := E^{(i_1, j)} = [f_{kl}]$.
\end{itemize}

To prove $D = F$, it suffices to compare the coefficients in the $j$-th column of $C, D, E$ and $F$, and show that the $j$-th column of $D$ equals the $j$-th column of $F$.

In order to establish $d_{i_2 j} = f_{i_2 j}$ and $d_{i_1 j} = f_{i_1 j}$, it is enough to prove $d_{i_2 j} = f_{i_2 j}$ and then, by swapping the roles of $i_1$ and $i_2$, we deduce $d_{i_1 j} = f_{i_1 j}$. We now show $d_{i_2 j} = f_{i_2 j}$.

According to (\ref{eq:Aijklgroup}a), we have
$$
d_{i_2 j} = 
\begin{cases}
b_{i_2} & \text{if } b_{i_2} > c_{i_2 j}\\
c_{i_2 j} & \text{if } b_{i_2} \le c_{i_2 j}
\end{cases}
$$
and, since $i_2 \neq i_1$,
$$
c_{i_2 j} = 
\begin{cases}
b_{i_2} & \text{if } i_2 \in U^A_{i_1 j}\\
a_{i_2 j} & \text{if } i_2 \notin U^A_{i_1 j}
\end{cases}.
$$

Moreover, again by (\ref{eq:Aijklgroup}a):
$$
f_{i_2 j} = 
\begin{cases}
b_{i_2} & \text{if } i_2 \in U^E_{i_1 j}\\
e_{i_2 j} & \text{if } i_2 \notin U^E_{i_1 j}
\end{cases},
\quad
e_{i_2 j} = 
\begin{cases}
b_{i_2} & \text{if } b_{i_2} > a_{i_2 j},\\
a_{i_2 j} & \text{if } b_{i_2} \le a_{i_2 j}
\end{cases}.
$$

\noindent\text{Case 1:} $b_{i_2} > a_{i_2 j}$.
\begin{itemize}
    \item We get $e_{i_2 j} = b_{i_2}$, so $f_{i_2 j} = b_{i_2}$.
    \item Since $(a_{i_2 j} - b_{i_2})^+ = 0$,
    from (\ref{eq:sigmaG}) and (\ref{eq:UDij}),  it follows that $i_2 \notin U^A_{i_1 j}$, thus $c_{i_2 j} = a_{i_2 j} < b_{i_2}$, which implies $d_{i_2 j} = b_{i_2} = f_{i_2 j}$.
\end{itemize}

\noindent\text{Case 2:} $b_{i_2} \le a_{i_2 j}$. Then $e_{i_2 j} = a_{i_2 j}$. We distinguish two subcases:
\begin{itemize}
    \item If $i_2 \in U^A_{i_1 j}$, then $c_{i_2 j} = b_{i_2}$, so $d_{i_2 j} = b_{i_2}$. Since $e_{i_2 j} = a_{i_2 j}$ and $i_2 \in U^A_{i_1 j}$, it follows that $i_2 \in U^E_{i_1 j}$, hence $f_{i_2 j} = b_{i_2} = d_{i_2 j}$.
    \item If $i_2 \notin U^A_{i_1 j}$, then $c_{i_2 j} = a_{i_2 j} \ge b_{i_2}$, so $d_{i_2 j} = a_{i_2 j}$. But $e_{i_2 j} = a_{i_2 j}$ and $i_2 \notin U^A_{i_1 j}$ imply $i_2 \notin U^E_{i_1 j}$, so $f_{i_2 j} = e_{i_2 j} = a_{i_2 j} = d_{i_2 j}$.
\end{itemize}

Thus we have shown $d_{i_2 j} = f_{i_2 j}$. By interchanging $i_1$ and $i_2$, we also obtain $d_{i_1 j} = f_{i_1 j}$.

Next, let $k \notin \{ i_1, i_2 \}$. We will prove $d_{k j} = f_{k j}$ by considering the two cases $k \in U^A_{i_1 j}$ or $k \notin U^A_{i_1 j}$. From (\ref{eq:Aijkl1}), we have
$$
d_{k j} = 
\begin{cases}
b_k  & \text{if } k \in U^C_{i_2 j}\\
c_{k j} & \text{if } k \notin U^C_{i_2 j}
\end{cases},
\quad
c_{k j} = 
\begin{cases}
b_k & \text{if } k \in U^A_{i_1 j}\\
a_{k j} & \text{if } k \notin U^A_{i_1 j}
\end{cases}.
$$
and
$$
f_{k j} = 
\begin{cases}
b_k & \text{if } k \in U^E_{i_1 j}\\
e_{k j} & \text{if } k \notin U^E_{i_1 j}
\end{cases},
\quad
e_{k j} = 
\begin{cases}
b_k & \text{if } k \in U^A_{i_2 j}\\
a_{k j} & \text{if } k \notin U^A_{i_2 j}
\end{cases}.
$$

\noindent\text{Case A:} $k \in U^A_{i_1 j}$.  
Then $\sigma_G(b_{i_1}, a_{k j}, b_k) = \min\!\Bigl(\tfrac{(b_{i_1}-b_k)^+}{2}, (a_{k j}-b_k)^+\Bigr) > 0$, which implies $b_{i_1} > b_k$ and $a_{k j} > b_k$. Therefore $c_{k j} = b_k$ and $d_{k j} = b_k$.
\begin{itemize}
    \item If $k \in U^E_{i_1 j}$, then $f_{k j} = b_k = d_{k j}$.
    \item If $k \notin U^E_{i_1 j}$, then  $f_{k j} = e_{k j}$  and since 
    $\sigma_G(b_{i_1}, e_{k j}, b_k) = 0$, we get   $ e_{k j} \le b_k$. As by (\ref{eq:aijklisbkorakl}), we have $e_{kj} \in\{b_k , a_{kj}\}$, from     $b_k < a_{k j}$, we obtain  $e_{k j} = b_k$, hence $f_{k j} =   e_{k j} = b_k = d_{k j}$.
\end{itemize}

\noindent\text{Case B:} $k \notin U^A_{i_1 j}$.  
Then $\sigma_G(b_{i_1}, a_{k j}, b_k) = \min\!\Bigl(\tfrac{(b_{i_1}-b_k)^+}{2}, (a_{k j}-b_k)^+\Bigr) = 0$ and  $c_{k j} = a_{k j}$.
\begin{itemize}
    \item If $k \in U^A_{i_2 j}$, then $e_{k j} = b_k$ and from (\ref{eq:sigmaG}) and  (\ref{eq:UDij}), we deduce $a_{kj} > b_k$ and $k \in U^C_{i_2 j}$. Hence $d_{k j} = b_k < a_{k j}$ and $b_{i_1} \le b_k$, which implies $k \notin U^E_{i_1 j}$, so finally $f_{k j} = e_{k j} = b_k = d_{k j}$.
    \item If $k \notin U^A_{i_2 j}$, then $e_{k j} = a_{k j} = c_{k j}$, so by (\ref{eq:UDij}) , we deduce $k \notin U^E_{i_1 j}$ and  $k \notin U^C_{i_2 j}$, leading to $d_{k j} = c_{k j} = e_{k j} = f_{k j}$.
\end{itemize}    
\end{proof}

\subsection{Proof of the results in Section \ref{sec:constructionAijvec}}
\label{sec:proofs:constructionAijvec}

\propositionaijklbkorakl*
\begin{proof}
Let 
$$
A^{(\vec{i},\vec{j})} = [a^{(\vec{i},\vec{j})}_{kl}]_{1\le k\le n,\; 1\le l\le m}.
$$
By induction on $h$, we easily obtain the following two results:

- From (\ref{eq:Aijklgroup}a)  , (\ref{eq:Aijklgroup}b) and (\ref{eq:constructionAijvec}), we deduce:
\begin{equation}\label{eq:Aijkl3h}
\forall (k,l)\in I\times J,\quad a^{(\vec{i},\vec{j})}_{kl} \in \{b_k,\, a_{kl}\}.
\end{equation}

- From (\ref{eq:Aijklgroup}b)  and (\ref{eq:constructionAijvec}), we deduce:
\begin{equation}\label{eq:Aijkl2h}
\forall k\in I,\; \forall l\notin\{j_1,j_2,\dots,j_h\},\quad a^{(\vec{i},\vec{j})}_{kl} = a_{kl}.
\end{equation}    
\end{proof}

\propositioncommutativePi*

\begin{proof}
The proof of (\ref{eq:mainavecipivecjpi}) is based on the case $h=2$ established in Lemma \ref{lemma:commutativeAij} and on the following two results:

\begin{itemize}
    \item Let $\vec{i} = (i_1, \dots, i_h) \in I^h$ be a tuple of pairwise distinct indices and $\vec{j} = (j_1, \dots, j_h) \in J^h$. For any pair of permutations $(\pi, \pi')$ of the set $\{1, 2, \dots, h\}$, we have
    \begin{equation}\label{eq:mainaction}
    \vec{i} \triangleleft (\pi \circ \pi') =   (\vec{i} \triangleleft \pi) \triangleleft \pi'
    ,\quad
    \vec{j} \triangleleft (\pi \circ \pi') =   (\vec{j} \triangleleft \pi) \triangleleft \pi'.
    \end{equation}
    Indeed, set $\vec{i'} := \vec{i} \triangleleft \pi = (i'_1, \dots, i'_h)$. For every $\lambda \in \{1, 2, \dots, h\}$, we have $i'_\lambda = i_{\pi(\lambda)}$. It follows that
    \[
    (\vec{i} \triangleleft \pi) \triangleleft \pi' = \vec{i'} \triangleleft \pi' = (i'_{\pi'(1)}, i'_{\pi'(2)}, \dots, i'_{\pi'(h)}) = (i_{\pi(\pi'(1))}, i_{\pi(\pi'(2))}, \dots, i_{\pi(\pi'(h))}) = \vec{i} \triangleleft (\pi \circ \pi').
    \]
    The proof of the equality $\vec{j} \triangleleft (\pi \circ \pi') =   (\vec{j} \triangleleft \pi) \triangleleft \pi'$ is identical to that of $\vec{i} \triangleleft (\pi \circ \pi') =   (\vec{i} \triangleleft \pi) \triangleleft \pi'$.
    
    \item Let $\pi, \pi'$ be two permutations of the set $\{1, 2, \dots, h\}$. Suppose that for every $\vec{i} = (i_1, \dots, i_h) \in I^h$ (with pairwise distinct indices) and every $\vec{j} = (j_1, \dots, j_h) \in J^h$, we have the equality (\ref{eq:mainavecipivecjpi}) for $\pi$ and $\pi'$:
    \[
    A^{(\vec{i}, \vec{j})} = A^{(\vec{i} \triangleleft \pi, \, \vec{j} \triangleleft \pi)}
    \quad \text{and} \quad
    A^{(\vec{i}, \vec{j})} = A^{(\vec{i} \triangleleft \pi', \, \vec{j} \triangleleft \pi')}.
    \]
    Then, for every $\vec{i} = (i_1, \dots, i_h) \in I^h$ (with pairwise distinct indices) and every $\vec{j} = (j_1, \dots, j_h) \in J^h$, we have
    \[
    A^{(\vec{i}, \vec{j})} = A^{(\vec{i} \triangleleft (\pi \circ \pi'), \, \vec{j} \triangleleft (\pi \circ \pi'))}.
    \]
    Indeed, set $\vec{i'} = \vec{i} \triangleleft \pi$ and $\vec{j'} = \vec{j} \triangleleft \pi$. By hypothesis, we have 
    \[
    A^{(\vec{i'}, \vec{j'})} = A^{(\vec{i'} \triangleleft \pi', \, \vec{j'} \triangleleft \pi')}.
    \]
    Then, by (\ref{eq:mainaction}),
    \[
    A^{(\vec{i} \triangleleft (\pi \circ \pi'), \, \vec{j} \triangleleft (\pi \circ \pi'))} = A^{(\vec{i'} \triangleleft \pi', \, \vec{j'} \triangleleft \pi')} = A^{(\vec{i'}, \vec{j'})} = A^{(\vec{i} \triangleleft \pi, \, \vec{j} \triangleleft \pi)} = A^{(\vec{i}, \vec{j})}.
    \]
\end{itemize}

By a classical result in permutation group theory, every permutation $\pi$ of the set $\{1, 2, \dots, h\}$ is a composition of transpositions of the form $\tau = (\lambda, \lambda+1)$ with $1 \le \lambda \le h-1$ (see \cite{sagan2013symmetric}). 

\noindent
For each $1 \le \lambda \le h-1$, the transposition $\tau = (\lambda, \lambda+1)$ is the permutation of the set $\{1, 2, \dots, h\}$ defined for every $\lambda' \in \{1, 2, \dots, h\}$ by
\[
\tau_{\lambda'} =
\begin{cases}
\lambda' & \text{if } \lambda' \notin \{\lambda, \lambda+1\} \\[1mm]
\lambda+1 & \text{if } \lambda' = \lambda \\[1mm]
\lambda & \text{if } \lambda' = \lambda+1.
\end{cases}
\]

We have shown that in order to obtain (\ref{eq:mainavecipivecjpi}) for every permutation $\pi$, it suffices to establish it for every transposition $\tau = (\lambda, \lambda+1)$ with $1 \le \lambda \le h-1$.

Set
\[
C := 
\begin{cases}
A, & \text{if } \lambda = 1 \\[1mm]
A^{(i_1, j_1)} \ast \dots \ast A^{(i_{\lambda-1}, j_{\lambda-1})}, & \text{if } \lambda > 1,
\end{cases}
\quad
D := C^{(i_\lambda, j_\lambda) \ast (i_{\lambda+1}, j_{\lambda+1})},
\quad
E := C^{(i_{\lambda+1}, j_{\lambda+1}) \ast (i_\lambda, j_\lambda)}.
\]
Then we have, see
(\ref{eq:constructionAijvec}) and (\ref{eq:vecpi})  :
\[
A^{(\vec{i}, \vec{j})} = 
\begin{cases}
D, & \text{if } \lambda+1 = h \\[1mm]
D^{(i_{\lambda+2}, j_{\lambda+2}) \ast \dots \ast (i_h, j_h)}, & \text{if } \lambda+1 < h,
\end{cases}
\quad
A^{(\vec{i} \triangleleft \tau, \, \vec{j} \triangleleft \tau)} = 
\begin{cases}
E, & \text{if } \lambda+1 = h \\[1mm]
E^{(i_{\lambda+2}, j_{\lambda+2}) \ast \dots \ast (i_h, j_h)}, & \text{if } \lambda+1 < h.
\end{cases}
\]
By applying Lemma \ref{lemma:commutativeAij} to the matrix $C$, we obtain the equality $D = E$, which implies that
\[
A^{(\vec{i}, \vec{j})} = A^{(\vec{i} \triangleleft \tau, \, \vec{j} \triangleleft \tau)}.
\]    
\end{proof}

\propositionmodOnegAijvec*
\begin{proof}
To prove the four statements, we use that for every $\lambda\in\{1,2,\dots,h\}$, we have
\begin{equation}\label{eq:mainilambdajlambda}
A^{(\vec{i},\vec{j})} := A^{(i_1,j_1) \ast (i_2,j_2) \ast \dots \ast (i_h,j_h)} = C^{(i_\lambda,j_\lambda)},\quad C := A^{(\vec{i'},\vec{j'})},
\end{equation}
with the $(h-1)$-tuples 
$$
\vec{i'} := (i_1,\dots,\widehat{i_\lambda},\dots,i_h)\quad \text{and} \quad \vec{j'} := (j_1,\dots,\widehat{j_\lambda},\dots,j_h)
$$
obtained from the $h$-tuples 
$$
\vec{i} = (i_1,\dots,i_\lambda,\dots,i_h)\quad \text{and} \quad \vec{j} = (j_1,\dots,j_\lambda,\dots,j_h)
$$
by removing the $\lambda$-th component (see Proposition \ref{proposition:commutativePi}).
\begin{enumerate}

\item For $h=1$, we have $A^{(\vec{i},\vec{j})} = A^{(i_1,j_1)}$, and the result follows from (\ref{eq:aijklisbkorakl}).

For $h>1$, set 
$$
C := A^{(i_1,j_1) \ast (i_2,j_2) \ast \dots \ast (i_{h-1},j_{h-1})} = [c_{kl}]_{1\le k\le n,\;1\le l\le m}.
$$
By definition, $A^{(\vec{i},\vec{j})} = C^{(i_h,j_h)}$ and 
by the induction hypothesis, for all $(k,l)\in I\times J$, we have $c_{kl}\in\{b_k,a_{kl}\}$.
Using   (\ref{eq:aijklisbkorakl}), we have  $
a^{(\vec{i},\vec{j})}_{kl}\in\{b_k,c_{kl}\} $. Then, it follows that for every $(k,l)\in I\times J$,
$$
a^{(\vec{i},\vec{j})}_{kl}\in \{b_k,a_{kl}\}.
$$

\item Let $\lambda\in\{1,2,\dots,h\}$. Set $A^{(\vec{i},\vec{j})}=C^{(i_\lambda,j_\lambda)}$ (see (\ref{eq:mainilambdajlambda})). Then, by statement  2 of Lemma \ref{lemma:deltaaijijequalzero}, we deduce that
$$
0=\delta^{C^{(i_\lambda,j_\lambda)}}(i_\lambda,j_\lambda)=\delta^{A^{(\vec{i},\vec{j})}}(i_\lambda,j_\lambda).
$$

\item For $h=1$, we have $A^{(\vec{i},\vec{j})}=A^{(i_1,j_1)}$ and $l = j_1$. By statement  2 of Lemma \ref{lemma:deltaaijijequalzero}, we deduce that
$$
\delta^{A^{(\vec{i},\vec{j})}}(s,l)= \delta^{A^{(i_1,j_1)}}(s,j_1)  = 0.
$$

For $h>1$, we distinguish the following two cases:

\begin{itemize}
    \item Suppose that $\{j_1,j_2,\dots,j_h\}=\{l\}$. In this case,
    $$
    A^{(\vec{i},\vec{j})}:=A^{(i_1,l)\ast(i_2,l)\ast\dots\ast(i_h,l)}= C^{(i_h,l)}\quad\text{with}\quad C:=A^{(i_1,l)\ast\dots\ast(i_{h-1},l)}.
    $$
    By the induction hypothesis applied to the matrix $C=A^{(\vec{i'},\vec{j'})}$ obtained (see (\ref{eq:constructionAijvec})) by $(h-1)$ compositions defined by $\vec{i'}:=(i_1,\dots,i_{h-1})$ and $\vec{j'}:=(l,l,\dots,l)$, we have $\delta^C(s,l)=0$. By applying statement 2 of Lemma \ref{lemma:deltaaijijequalzero}, we conclude that
    $$
    \delta^{A^{(\vec{i},\vec{j})}}(s,l)=\delta^{C^{(i_h,l)}}(s,l)=0.
    $$
    
    \item Suppose that $\{j_1,j_2,\dots,j_h\}\setminus\{l\}\neq\emptyset$. Using (\ref{eq:mainavecipivecjpi}), we may assume that there exists $1\le h'<h$ such that 
    $$
    \forall \lambda\in\{1,\dots,h'\},\quad j_\lambda=l\quad\text{and}\quad l\notin\{j_{h'+1},\dots,j_h\}.
    $$
    By the induction hypothesis applied to the matrix 
    $$
    C:=A^{(i_1,l)\ast(i_2,l)\ast\dots\ast(i_{h'},l)}
    $$
    obtained in  (\ref{eq:constructionAijvec}) by $h'<h$ compositions defined by $ (i_1,\dots,i_{h'})$ and $ (l,l,\dots,l)$, we have $\delta^C(s,l)=0$. From (\ref{eq:constructionAijvec}), we deduce that 
    $$
    A^{(\vec{i},\vec{j})}=C^{(\vec{i'},\vec{j'})},
    $$
    with 
    $$
    \vec{i'}:=(i_{h'+1},\dots,i_h)\quad \text{and}\quad \vec{j'}:=(j_{h'+1},\dots,j_h).
    $$
    Since $l\notin\{j_{h'+1},\dots,j_h\}$, the matrices $C$ and $C^{(\vec{i'},\vec{j'})}=A^{(\vec{i},\vec{j})}$
      have the same $l$-th column,  see the second statement of Proposition \ref{proposition:aijklbkorakl}. Hence, by  (\ref{eq:deltaTij}) we get: 
    $$
    \delta^{A^{(\vec{i},\vec{j})}}(s,l)=\delta^{C}(s,l)=0.
    $$
\end{itemize}

\item We observe that the matrices $A^{(\vec{i},\vec{j})}$ and $A$ have the same $l$-th column, see Proposition \ref{proposition:aijklbkorakl}.  Hence,
$$
\delta^{A^{(\vec{i},\vec{j})}}(s,l)=\delta^{A}(s,l).
$$    
\end{enumerate}
\end{proof}

\propositionmodtwogAijvec*
\begin{proof}
We proceed by induction on $h$. Lemma \ref{lemma:matrixTwithdetaTijzero} establishes the case $h=1$.

Suppose that $h > 1$. Let
$$
C := A^{(i_1,j_1) \ast (i_2,j_2) \ast \dots \ast (i_{h-1},j_{h-1})} = [c_{kl}]_{1\le k\le n,\; 1\le l\le m}.
$$
We have
$$
A^{(\vec{i},\vec{j})} = C^{(i_h,j_h)},
$$
(see (\ref{eq:constructionAijvec})). Since $\delta^T(i_\lambda,j_\lambda)=0$ for all $\lambda\in\{1,2,\dots,h-1\}$, by the induction hypothesis we have
\begin{equation}\label{eq:hyprecmain2}
\forall (k,l)\in I\times J,\quad |c_{kl} - a_{kl}| \le |t_{kl} - a_{kl}|.
\end{equation}
It follows from (\ref{eq:hyprecmain2}) that, in order to establish (\ref{eq:ilambdajlambdaGNew}), it suffices to consider the cases where
$$
a^{(\vec{i},\vec{j})}_{kl} \neq c_{kl}.
$$
But, by applying (\ref{eq:Aijklgroup}a) and (\ref{eq:Aijklgroup}b) to the matrix $C$ and using the equality $A^{(\vec{i},\vec{j})} = C^{(i_h,j_h)}$, we obtain
\begin{equation}\label{eq:Cijkl2n}
a^{(\vec{i},\vec{j})}_{i_h j_h} =
\begin{cases}
b_{i_h} & \text{if } b_{i_h} > c_{i_h j_h} \\[1mm]
c_{i_h j_h} & \text{if } b_{i_h} \le c_{i_h j_h},
\end{cases}
\quad\text{and for } k\neq i_h,\quad
a^{(\vec{i},\vec{j})}_{k j_h} =
\begin{cases}
b_k & \text{if } k\in U^C_{i_h j_h} \\[1mm]
c_{k j_h} & \text{if } k\notin U^C_{i_h j_h}.
\end{cases}
\end{equation}
Furthermore,
\begin{equation}\label{eq:Cijkl3n}
\forall k\in I,\;\forall l\in J\setminus\{j_h\},\quad a^{(\vec{i},\vec{j})}_{kl} = c_{kl}.
\end{equation}
This implies that
$$
\{(k,l)\in I\times J \mid a^{(\vec{i},\vec{j})}_{kl} \neq c_{kl}\} \subseteq \bigl(\{i_h\}\cup U^C_{i_h j_h}\bigr)\times\{j_h\}.
$$
It remains to show that
$$
\bigl|a^{(\vec{i},\vec{j})}_{kl} - a_{kl}\bigr| \le |t_{kl} - a_{kl}|
$$
in the case where $l=j_h$, $k\in\{i_h\}\cup U^C_{i_h j_h}$, and $a^{(\vec{i},\vec{j})}_{k j_h} \neq c_{k j_h}$.

\vspace{2mm}
\noindent\text{Step 1.} Suppose that $k=i_h$ and $a^{(\vec{i},\vec{j})}_{i_h j_h}\neq c_{i_h j_h}$. Then, from (\ref{eq:Cijkl2n}) we deduce that
$$
a^{(\vec{i},\vec{j})}_{i_h j_h} = b_{i_h} \quad \text{and} \quad b_{i_h} > c_{i_h j_h}.
$$
By applying the first statement of Proposition \ref{proposition:modOnegAijvec} to the matrix $C$, we obtain $c_{i_h j_h} = a_{i_h j_h}$. But $\delta^T(i_h,j_h)=0$ implies $b_{i_h}\le t_{i_h j_h}$, see (\ref{eq:deltaTij}),  hence
$$
a_{i_h j_h} = c_{i_h j_h} < b_{i_h} = a^{(\vec{i},\vec{j})}_{i_h j_h} \le t_{i_h j_h},
$$
which implies
$$
\bigl|a^{(\vec{i},\vec{j})}_{i_h j_h} - a_{i_h j_h}\bigr| \le |t_{i_h j_h} - a_{i_h j_h}|.
$$

\vspace{2mm}
\noindent\text{Step 2.} Suppose that $k\in U^C_{i_h j_h}$. That is, 
$$
\sigma_G(b_{i_h}, c_{k j_h}, b_k) = \min\Bigl(\frac{(b_{i_h}-b_k)^+}{2},(c_{k j_h}-b_k)^+\Bigr) > 0.
$$
Then, we have, successively:
\begin{itemize}
    \item $b_{i_h} > b_k$ and $c_{k j_h} > b_k$. Since by the first statement of Proposition \ref{proposition:modOnegAijvec},  we have $c_{k j_h}\in\{b_k,a_{k j_h}\}$, we deduce that $c_{k j_h} = a_{k j_h}$.
    \item The hypothesis $\delta^T(i_h,j_h)=0$ implies
    $$
    \sigma_G(b_{i_h}, t_{k j_h}, b_k) = \min\Bigl(\frac{(b_{i_h}-b_k)^+}{2},(t_{k j_h}-b_k)^+\Bigr) = 0.
    $$
    Since $b_{i_h} > b_k$, it follows that $t_{k j_h} \le b_k$.
    \item From (\ref{eq:Cijkl2n}), we deduce that $a^{(\vec{i},\vec{j})}_{k j_h} = b_k$.
\end{itemize}
Finally, we obtain:
$$
t_{k j_h} \le b_k = a^{(\vec{i},\vec{j})}_{k j_h} < c_{k j_h} = a_{k j_h},
$$
which implies
$$
\bigl|a^{(\vec{i},\vec{j})}_{k j_h} - a_{k j_h}\bigr| \le |t_{k j_h} - a_{k j_h}|.
$$

This completes the proof.    
\end{proof}

\propositionSimpleDefAijvec*
\begin{proof}
For every $(k,l) \in I \times J$ , set  
\begin{equation}\label{eq:atilde}
\widetilde a_{kl} = \begin{cases}
b_k & \text{if } (k,l) \in {\cal E} ,\\[1mm]
a_{kl} & \text{if } (k,l) \notin {\cal E} .
\end{cases}    
\end{equation}
To  prove that  $a^{(\vec{i},\vec{j})}_{kl}    = \widetilde a_{kl}$ for every $(k,l) \in I \times J$, we proceed by induction on $h$.

\noindent
Before performing our induction, note that for 
every $(k,l) \in I \times J$ such that $l\notin\{j_1 , j_2  , \dots j_h\}$, we  have 
$a^{(\vec{i},\vec{j})}_{kl} = a_{kl}  =\widetilde a_{kl}$ : this result  is deduced from the second statement of Proposition \ref{proposition:aijklbkorakl} and the fact that $(k , l)\notin{\cal E}$. 

\noindent
$\bullet \,$ If $h = 1$, then we have:
\[ {\cal E}_1 = \begin{cases}
\{(i_1 , j_1)\} & \text{if } b_{i_1} > a_{i_1 j_1} ,\\[1mm]
\emptyset & \text{if } b_{i_1} \leq  a_{i_1 j_1}
\end{cases}, \quad  {\cal E}_2 = U^A_{i_1 j_1} \times \{j_1\}\quad 
 \text{and}\quad {\cal E} = {\cal E}_1 \cup {\cal E}_2, \,\,  \text{see (\ref{eq:combinedE1E2Eh})}.\]
As we have $A^{(\vec{i},\vec{j})} = A^{(i_1,j_1)}$, then by using  (\ref{eq:atilde}),  (\ref{eq:Aijkl1}) and 
(\ref{eq:Aijkl2}), one can easily check that for every $k\in I$, we have 
$a^{(\vec{i},\vec{j})}_{kj_1} =  \widetilde a_{k j_1}$.

\noindent
$\bullet\,$ Suppose that for all pairs $(\vec{i},\vec{j})$ of $(h -1)$ tuples  as in Definition \ref{def:vecivecj}, 
we have $a^{(\vec{i},\vec{j})}_{kl}    =\widetilde a_{kl}$ for 
every $(k,l) \in I \times J$.

\noindent 
Let   $(\vec{i} = (i_1, i_2, \dots, i_h),\vec{j}  = (j_1 , \dots , j_h))$ be a pair of $h$-tuples   as in Definition \ref{def:vecivecj} and ${\cal E}_1$, ${\cal E}_2$, and ${\cal E}$  be the sets associated to the pair  $(\vec{i},\vec{j})$ as in (\ref{eq:combinedE1E2Eh}). We will use that for every $\lambda\in\{1 , 2 , \dots , h\}$, we have, see    (\ref{eq:constructionAijvec}) and  (\ref{eq:mainavecipivecjpi}):
\begin{equation}\label{eq:intt1}
 A^{(\vec{i},\vec{j})} = C^{(i_\lambda,j_\lambda)} \quad \text{where} \quad C = 
A^{(\vec{i'},\vec{j'})}   
\end{equation}
and $(\vec{i'},\vec{j'})$ is the pair obtained by removing the $\lambda$-th component from $(\vec{i},\vec{j})$ whose  corresponding sets ${\cal E}'_1$, ${\cal E}'_2$, and ${\cal E}'$ are defined in (\ref{eq:ehminusonesets}).

Let $(k , l)$  be a pair       in $I \times \{ j_1 , \dots , j_h \}$ and let  us prove that $a^{(\vec{i},\vec{j})}_{kl}    =\widetilde a_{kl}$:

\noindent \begin{itemize}
  
 \item If $(k , l)\in {\cal E}_1$, there is an index $\lambda\in\{1 , \dots , h\}$ such that:
\[(k , l) = (i_\lambda , j_\lambda), \quad  b_{i_\lambda} >a _{i_\lambda j_\lambda}
\quad \text{and} \quad \widetilde a_{kl} = b_{i_\lambda}, \,\,  \text{see (\ref{eq:combinedE1E2Eh}) and  (\ref{eq:atilde}).   }\]
By setting  $A^{(\vec{i},\vec{j})} = C^{(i_\lambda,j_\lambda)}$ with 
 $C = A^{(\vec{i'},\vec{j'})}$ as in (\ref{eq:intt1}), we get from (\ref{eq:Aijkl1}):
\[ a^{(\vec{i},\vec{j})}_{kl} = \begin{cases}
b_{i_\lambda}   & \text{if } b_{i_\lambda} > c_{i_\lambda j_\lambda} ,\\[1mm]
c_{i_\lambda j_\lambda}  & \text{if } b_{i_\lambda} \leq c_{i_\lambda j_\lambda} 
\end{cases}\quad \text{where} \quad C = [c_{k' l'}]_{1 \leq k' \leq n , 1 \leq l' \leq m}.\]
If the inequality  $b_{i_\lambda} \leq c_{i_\lambda j_\lambda}$ holds , then we get
$a_{i_\lambda j_\lambda} < b_{i_\lambda} \leq c_{i_\lambda j_\lambda}$. As from the first statement of Proposition \ref{proposition:aijklbkorakl}, we have 
$c_{i_\lambda j_\lambda} \in\{b_{i_\lambda} ,  a_{i_\lambda j_\lambda} \}$, we conclude that in this case $     b_{i_\lambda} = c_{i_\lambda j_\lambda} = a^{(\vec{i},\vec{j})}_{kl}$ so in both cases, we have  $a^{(\vec{i},\vec{j})}_{kl} = b_{i_\lambda} = \widetilde a_{kl}$.

\noindent
 \item If $(k , l)\in {\cal E}_2$,
 there is an index $\lambda\in\{1 , \dots , h\}$ such that:
\[l =   j_\lambda, \quad  k\in   U^A_{i_\lambda j_\lambda} 
\quad \text{and} \quad \widetilde a_{kl} = b_k, \,\,  \text{see (\ref{eq:combinedE1E2Eh}) and  (\ref{eq:atilde}).   }\]
As $k \not= i_\lambda$, see  (\ref{eq:sigmaG}) and (\ref{eq:UDij}),     
by setting  $A^{(\vec{i},\vec{j})} = C^{(i_\lambda,j_\lambda)}$ with 
 $C = A^{(\vec{i'},\vec{j'})}$ as in (\ref{eq:intt1}), we get from (\ref{eq:Aijkl2}):
\[ a^{(\vec{i},\vec{j})}_{kl} = \begin{cases}
b_{k}   & \text{if } k\in   U^C_{i_\lambda j_\lambda} ,\\[1mm]
c_{k j_\lambda}  & \text{if } k\notin   U^C_{i_\lambda j_\lambda} 
\end{cases}\quad \text{where} \quad C = [c_{k' l'}]_{1 \leq k' \leq n , 1 \leq l' \leq m}\]
As 
$k\in   U^A_{i_\lambda j_\lambda} $, we have   $b_{i_\lambda } > b_k$ and $a_{k j_\lambda} > b_k$, see (\ref{eq:sigmaG}), so if  $k\notin   U^C_{i_\lambda j_\lambda}$, we deduce   
     $c_{k j_\lambda} \leq b_k < a_{k j_\lambda}$.   As from the first statement of Proposition \ref{proposition:aijklbkorakl}, we have 
$c_{k j_\lambda} \in\{b_k ,  a_{k j_\lambda} \}$, we conclude that in this case $     b_k = c_{k j_\lambda} = a^{(\vec{i},\vec{j})}_{k j_\lambda}$ so in both cases, we have  $a^{(\vec{i},\vec{j})}_{k j_\lambda } = b_{k} = \widetilde a_{kj_\lambda}$.

 \noindent
  \item If $(k , l)\notin {\cal E}$, we know from (\ref{eq:atilde}) that  $\widetilde a_{kl} = a_{kl}$. As we suppose that $l\in\{j_1 , j_2 , \dots , j_h\}$, there is an index $\lambda\in\{1 , \dots , h\}$ such that $l = j_\lambda$. Let us write   $A^{(\vec{i},\vec{j})} = C^{(i_\lambda,j_\lambda)}$ with 
 $C = A^{(\vec{i'},\vec{j'})}$ as in (\ref{eq:intt1}), then from Lemma \ref{lemma:IncEhEprimeh}, we get that $(k , l) = (k , j_\lambda)\notin {\cal E'}$, so by applying the induction hypothesis to the pair $(\vec{i'} , \vec{j'}) $ , we obtain that $c_{k j_\lambda} = a_{k j_\lambda}$.

 \noindent
In both cases $k = i_\lambda$ or $k \not= i_\lambda$, we claim  that 
$a^{(\vec{i},\vec{j})}_{k j_\lambda} = c_{k j_\lambda} = a_{kj_\lambda}$:

\noindent
\begin{itemize}
\item if  $k = i_\lambda$, then $(k , j_\lambda)\notin {\cal E}$ implies that $b_{i_\lambda} \leq a_{i_\lambda j_\lambda}$. As we have $c_{i_\lambda j_\lambda} = a_{i_\lambda j_\lambda}$, so $b_{i_\lambda} \leq c_{i_\lambda j_\lambda} $, then  from  (\ref{eq:Aijkl1}) applied to the matrix $C$ and the pair $(i_\lambda , j_\lambda)$, we get:
\[ a^{(\vec{i},\vec{j})}_{i_\lambda  j_\lambda} = c_{i_\lambda j_\lambda} = a_{i_\lambda j_\lambda}.\]

\noindent
\item if  $k \not= i_\lambda$, then $(k , j_\lambda)\notin {\cal E}$  implies that $k\notin U^A_{i_\lambda j_\lambda}$ , see (\ref{eq:combinedE1E2Eh}). From the equality  $c_{k j_\lambda} = a_{kj_\lambda}$ , we get  from  (\ref{eq:UDij}) that 
$k\notin U^C_{i_\lambda j_\lambda}$ and from  (\ref{eq:Aijkl2})    applied to the matrix $C$ and the pair $(i_\lambda , j_\lambda)$, we get:
\[ a^{(\vec{i},\vec{j})}_{k  j_\lambda} = c_{k j_\lambda} = a_{k j_\lambda}.\]
\end{itemize}
\end{itemize}
\end{proof}

\corollarynormlambdainfty*
\begin{proof}
By the second statement of Proposition \ref{proposition:modOnegAijvec}, for every $\lambda\in\{1 , \dots , h\}$,   we have  $\delta^{A^{(\vec{i},\vec{j})}}(i_\lambda, j_\lambda) = 0$. Using the second statement of Lemma \ref{lemma:matrixTwithdetaTijzero} with $T:= A^{(\vec{i},\vec{j})}$, we deduce the inequality $\Vert A^{(i_\lambda,j_\lambda)} - A \Vert_\infty \leq \Vert A^{(\vec{i},\vec{j})} - A \Vert_\infty$ for every $\lambda\in\{1 , \dots , h\}$. 

We have proven the  inequality 
$\max_{\lambda\in\{1 , \dots , h\}}\, \Vert A^{(i_\lambda,j_\lambda)} - A \Vert_\infty \, \leq \Vert A^{(\vec{i},\vec{j})} - A \Vert_\infty$.

As  (\ref{eq:normlambda}) is clear if $A^{(\vec{i},\vec{j})} = A$, we suppose that 
$\Vert A^{(\vec{i},\vec{j})} - A \Vert_\infty > 0$.

Let us   fix $(k , v)\in I \times J$  such that $0 < \Vert A^{(\vec{i},\vec{j})} - A \Vert_\infty = \mid a^{(\vec{i},\vec{j})}_{k v} - a _{k v} \mid  $. 
 
From the second statement of Proposition \ref{proposition:aijklbkorakl},  we deduce that $v\in\{j_1 , j_2 , \dots , j_h\}$.
Fix  $\widetilde\lambda\in\{1 , 2 , \dots ,h\}$    such that  $v = j_{\widetilde\lambda}$ and note that with this additional notation, we have: 
\begin{equation}\label{eq:int2_b}
0 < \Vert A^{(\vec{i},\vec{j})} - A \Vert_\infty = \mid a^{(\vec{i},\vec{j})}_{k j_{\widetilde\lambda}} - a _{k j_{\widetilde\lambda}} \mid.     
\end{equation}

Using  Proposition \ref{proposition:SimpleDefAijvec} with the set ${\cal E}$ associated with $(\vec{i},\vec{j})$, we assert that $(k , j_{\widetilde\lambda})\in{\cal E}$. As ${\cal E} = {\cal E}_1 \cup {\cal E}_2$, we distinguish  the following two cases:

\begin{itemize}
\item if $(k, j_{\widetilde\lambda})\in {\cal E}_1$, then we have $a^{(\vec{i},\vec{j})}_{k j_{\widetilde\lambda}} =  b_k $. Moreover, 
there exists some $\lambda\in\{1 , 2 , \dots ,h\}$ such that 
$$k = i_{\lambda}, \quad j_{\widetilde\lambda} = j_{\lambda} ,   \quad \text{and} \quad 
   b_{i_{\lambda}} > a_{i_{\lambda} j_{\lambda}}
 $$
Then, we have:
\begin{align}
\Vert A^{(\vec{i},\vec{j})} - A \Vert_\infty = \mid a^{(\vec{i},\vec{j})}_{k j_{\widetilde\lambda}} - a _{k j_{\widetilde\lambda}} \mid & = \mid b_{i_\lambda} - a_{i_\lambda  j_\lambda}\mid = \mid a^{(i_\lambda,j_\lambda)}_{i_\lambda j_\lambda}- a_{i_\lambda  j_\lambda} \mid\nonumber\\
& \leq \Vert A^{(i_{\lambda},j_{\lambda})} - A\Vert_\infty\nonumber\\
& \leq \max_{1 \leq  \lambda' \leq h}\,\Vert A^{(i_{\lambda'},j_{\lambda'})} - A\Vert_\infty. 
\nonumber
\end{align}
\item if $(k, j_{\widetilde\lambda})\in {\cal E}_2$,  then we have $a^{(\vec{i},\vec{j})}_{k j_{\widetilde\lambda}} =  b_k $. Moreover,     there exists some $\lambda\in\{1 , 2 , \dots ,h\}$ such that 
$$k \in  U^A_{i_{\lambda} j_{\lambda}} \quad \text{and} \quad    j_{\widetilde\lambda} = j_{\lambda}.   $$
Then, we have:
\begin{align}
\Vert A^{(\vec{i},\vec{j})} - A \Vert_\infty = \mid a^{(\vec{i},\vec{j})}_{k j_{\widetilde\lambda}} - a _{k j_{\widetilde\lambda}} \mid & = \mid b_{k} - a_{k  j_\lambda}\mid = \mid a^{(i_\lambda,j_\lambda)}_{k j_\lambda}- a_{k  j_\lambda} \mid\nonumber\\
& 
\leq \Vert A^{(i_{\lambda},j_{\lambda})} - A\Vert_\infty\nonumber\\
& \leq \max_{1 \leq  \lambda' \leq h}\,\Vert A^{(i_{\lambda'},j_{\lambda'})} - A\Vert_\infty. 
\nonumber 
\end{align}
\end{itemize}
We have established (\ref{eq:normlambda}).
\end{proof}

\subsection{Proof of the results in Section \ref{sec:mainresults}}
\label{sec:proofs:mainresults}

\theoremtheoOneconsSij*
\begin{proof}
We must show that for every $s\in I$, we have, see (\ref{eq:deltaTi}), (\ref{eq:Delta}), (\ref{eq:equivDelta}):
$$
\delta^{A^{(\vec{i},\vec{j})}}_s = 0.
$$
We have the partition, see (\ref{eq:nconsninc}):
$$
I = N_{\text{inc}} \cup N_{\text{cons}},
$$ 
with 
$$
N_{\text{inc}} = \{i_1, \dots, i_h\} \quad \text{and} \quad \vec{i} := (i_1, \dots, i_h).
$$

\begin{itemize}
    \item If $s\in N_{\text{inc}}$, then there exists a unique $\lambda\in\{1,2,\dots,h\}$ such that $s = i_\lambda$. By statement 2 of Proposition \ref{proposition:modOnegAijvec}, it follows that
    $$
    \delta^{A^{(\vec{i},\vec{j})}}_s \le \delta^{A^{(\vec{i},\vec{j})}}(s,j_\lambda) = \delta^{A^{(\vec{i},\vec{j})}}(i_\lambda,j_\lambda) = 0.
    $$
    
    \item If $s\in N_{\text{cons}}$, then $\delta^A_s = 0$. Let $l\in J$ be such that $\delta^A(s,l) = 0$. Then, by statements 3 and 4 of Proposition \ref{proposition:modOnegAijvec}, we deduce that
    $$
    \delta^{A^{(\vec{i},\vec{j})}}_s \le \delta^{A^{(\vec{i},\vec{j})}}(s,l) = \delta^A(s,l) = 0.
    $$
\end{itemize}    
\end{proof}

\theoremtheoTwoAijnorm*
\begin{proof}
It is an immediate consequence of Proposition \ref{proposition:mod2gAijvec} and Theorem \ref{theorem:theo1consSij}.

We use the partition 
$$
I = N_{\text{inc}} \cup N_{\text{cons}},
$$
with 
$$
N_{\text{inc}} = \{i_1, \dots, i_h\} \quad \text{and} \quad \vec{i} := (i_1, \dots, i_h).
$$
For every $\lambda \in \{1,2,\dots,h\}$, since $\delta^T_{i_\lambda} = 0$, we choose $j_\lambda \in J$ such that 
$$
\delta^T(i_\lambda, j_\lambda) = 0,
$$ 
and form the $h$-tuple $(j_1, \dots, j_h)$ ,   see (\ref{eq:deltaTi}), (\ref{eq:Delta}), (\ref{eq:equivDelta}).

By Theorem \ref{theorem:theo1consSij}, we know that the system 
$$
A^{(\vec{i},\vec{j})} \Box_{\min}^{\max} x = b
$$
is consistent.

By Proposition \ref{proposition:mod2gAijvec}, we have
$$
\forall (k,l) \in I \times J,\quad \bigl| a^{(\vec{i},\vec{j})}_{kl} - a_{kl} \bigr| \le \bigl| t_{kl} - a_{kl} \bigr|.
$$    
\end{proof}

\corollarynormminJh*
\begin{proof}
Let $p\in\{1,2,\infty\}$. Define
$$
\Delta' = \min_{\vec{j}\in J^h} \Vert A^{(\vec{i},\vec{j})} - A \Vert_p.
$$
We must show that $\Delta' = \mathring{\Delta}_{p}$, where
$$
\mathring{\Delta}_{p} = \min_{T\in{\cal T}} \Vert T - A \Vert_p.
$$

\begin{itemize}
    \item By Theorem \ref{theorem:theo1consSij}, for every $\vec{j}\in J^h$ we have 
    $$
    A^{(\vec{i},\vec{j})} \in {\cal T},
    $$
    so that
    $$
    \Vert A^{(\vec{i},\vec{j})} - A \Vert_p \ge \mathring{\Delta}_{p}.
    $$
    Hence, $\Delta' \ge \mathring{\Delta}_{p}$.
    
    \item By Theorem \ref{theorem:theo2Aijnorm}, for every matrix $T\in{\cal T}$ there exists an $h$-tuple $\vec{j}\in J^h$ such that the matrix 
    $$
    A^{(\vec{i},\vec{j})} = [a^{(\vec{i},\vec{j})}_{kl}]
    $$
    satisfies
    $$
    \forall (k,l)\in I\times J,\quad \bigl|a^{(\vec{i},\vec{j})}_{kl} - a_{kl}\bigr| \le \bigl|t_{kl} - a_{kl}\bigr|.
    $$
    For each of the three norms considered, it then follows that
    $$
    \Delta' \le \Vert A^{(\vec{i},\vec{j})} - A \Vert_p \le \Vert T - A \Vert_p,
    $$
    and therefore,
    $$
    \Delta' \le \min_{T\in{\cal T}} \Vert T - A \Vert_p = \mathring{\Delta}_{p}.
    $$
\end{itemize}

Thus, we deduce that $\Delta' = \mathring{\Delta}_{p}$.    
\end{proof}

\theoremformulaforLinftydist*
\begin{proof}
First, we note that the second statement can be deduced from the first one and   Corollary \ref{corollary:normlambdainfty}.

By the definition (\ref{eq:lAinfty}) of the set $L_{A,\infty}$, Theorem \ref{theorem:theo1consSij} and the inconsistency of the system $(S)$, we deduce that we have:
\[ 0 < \mathring{\Delta}_{\infty} \leq \Vert A^{(\vec{i},\vec{j})} - A \Vert_\infty.\]
To prove the first statement, one must check that for any matrix $T\in{\cal T}$, i.e., the system $T \Box_{\min}^{\max} x = b$
 is consistent, we have $\Vert A^{(\vec{i},\vec{j})} - A \Vert_\infty \leq\Vert T - A \Vert_\infty$.

Let $T $ be any matrix in the set $\mathcal T$. By the consistency  of the system $T \Box_{\min}^{\max} x = b$, we have: 
\[ \forall i\in I   \,\,   \exists l\in J \,\, \text{such that} \,\, \delta^T(i , l) = 0, \]   
see (\ref{eq:deltaTi}), (\ref{eq:Delta})and (\ref{eq:equivDelta}).

For each $\lambda \in \{1 , 2 , \dots , h\}$, we choose an index $l_\lambda\in J$
 such that $\delta^T(i_\lambda , l_\lambda) = 0$ and construct (Definition \ref{def:matAij}) the matrix 
 $A^{(i_\lambda , l_\lambda)} = [a^{(i_\lambda , l_\lambda)}_{kv}]$ which, by Lemma \ref{lemma:matrixTwithdetaTijzero}, satisfies:
\begin{equation}\label{eq:int0}
\forall (k , v)\in I \times J\,,\, \mid a^{(i_\lambda , l_\lambda)}_{kv} - a_{kv} \mid \leq \Vert A^{(i_\lambda , l_\lambda)} - A \Vert_\infty \leq \Vert T - A \Vert_\infty    
\end{equation}

As $(\vec{i},\vec{j})\in L_{A,\infty}$ we deduce  from the definition (\ref{eq:lAinfty}) of the set $L_{A,\infty}$ and 
(\ref{eq:int0}): 
\begin{equation}\label{eq:int1}
\forall \lambda \in \{1 , 2 , \dots , h\},\,   \forall (k , v)\in I \times J,\, \mid a^{(i_\lambda,j_\lambda)}_{k v}  - a_{k v} \mid \leq  \Vert A^{(i_\lambda,j_\lambda)} - A \Vert_\infty \leq \Vert A^{(i_\lambda,l_\lambda)} - A \Vert_\infty \leq \Vert T - A \Vert_\infty.    
\end{equation}

Let us   fix $(k , v)\in I \times J$  such that $0 < \Vert A^{(\vec{i},\vec{j})} - A \Vert_\infty = \mid a^{(\vec{i},\vec{j})}_{k v} - a _{k v} \mid  $. 
 
From the second statement of Proposition \ref{proposition:aijklbkorakl},  we deduce that $v\in\{j_1 , j_2 , \dots , j_h\}$.
Fix  $\widetilde\lambda\in\{1 , 2 , \dots ,h\}$    such that  $v = j_{\widetilde\lambda}$ and note that with this additional notation, we have: 
\begin{equation}\label{eq:int2_a}
0 < \Vert A^{(\vec{i},\vec{j})} - A \Vert_\infty = \mid a^{(\vec{i},\vec{j})}_{k j_{\widetilde\lambda}} - a _{k j_{\widetilde\lambda}} \mid.     
\end{equation}

Using  Proposition \ref{proposition:SimpleDefAijvec} with the set ${\cal E}$ associated with $(\vec{i},\vec{j})$, we assert that $(k , j_{\widetilde\lambda})\in{\cal E}$. As ${\cal E} = {\cal E}_1 \cup {\cal E}_2$, we distinguish the following two cases:
\begin{itemize}
  \item  If $(k, j_{\widetilde\lambda})\in {\cal E}_1$, then we have $a^{(\vec{i},\vec{j})}_{k j_{\widetilde\lambda}} =  b_k $. Moreover, 
there exists some $\lambda\in\{1 , 2 , \dots ,h\}$ such that 
$$k = i_{\lambda}, \quad j_{\widetilde\lambda} = j_{\lambda} ,   \quad \text{and} \quad 
   b_{i_{\lambda}} > a_{i_{\lambda} j_{\lambda}}.
 $$
Then, we have:
\begin{align}
\Vert A^{(\vec{i},\vec{j})} - A \Vert_\infty = \mid a^{(\vec{i},\vec{j})}_{k j_{\widetilde\lambda}} - a _{k j_{\widetilde\lambda}} \mid & = \mid b_{i_\lambda} - a_{i_\lambda  j_\lambda}\mid = \mid a^{(i_\lambda,j_\lambda)}_{i_\lambda j_\lambda}- a_{i_\lambda  j_\lambda} \mid, \quad \text{see       (\ref{eq:Aijkl1}) }\nonumber\\
& \leq \Vert A^{(i_\lambda,j_\lambda)} - A \Vert_\infty \leq \Vert A^{(i_\lambda,l_\lambda)} - A \Vert_\infty \leq \Vert T - A \Vert_\infty, \quad \text{see (\ref{eq:int1})}.\nonumber
\end{align}

\item If $(k, j_{\widetilde\lambda})\in {\cal E}_2$,  then we have $a^{(\vec{i},\vec{j})}_{k j_{\widetilde\lambda}} =  b_k $. Moreover,     there exists some $\lambda\in\{1 , 2 , \dots ,h\}$ such that: 
$$k \in  U^A_{i_{\lambda} j_{\lambda}} \quad \text{and} \quad    j_{\widetilde\lambda} = j_{\lambda}.$$
Then, we have:
\begin{align}
\Vert A^{(\vec{i},\vec{j})} - A \Vert_\infty = \mid a^{(\vec{i},\vec{j})}_{k j_{\widetilde\lambda}} - a _{k j_{\widetilde\lambda}} \mid & = \mid b_{k} - a_{k  j_\lambda}\mid = \mid a^{(i_\lambda,j_\lambda)}_{k j_\lambda}- a_{k  j_\lambda} \mid, \quad \text{see       (\ref{eq:Aijkl1}) }\nonumber\\
& \leq \Vert A^{(i_\lambda,j_\lambda)} - A \Vert_\infty \leq \Vert A^{(i_\lambda,l_\lambda)} - A \Vert_\infty \leq \Vert T - A \Vert_\infty,\,\, \text{see (\ref{eq:int1})}\nonumber.
\end{align}
\end{itemize}
\end{proof}
\corollaryformulaforLinftydist*
\begin{proof}
We distinguish  the following two cases:
\begin{itemize}
    
\item Suppose that the system $(S)$ is consistent. As in this case 
$\mathring{\Delta}_{\infty} = \min_{T \in{\cal T}}\, \Vert T - A \Vert_\infty = 0$, we must prove that for all $i\in I$ we have: 
\[\min_{j\in J} \,\max\Bigl[(b_i - a_{ij})^+, \, \max_{k \in I, k\neq i}\Bigl(\theta(i,k) \cdot (a_{kj} - b_k)^+\Bigr)\Bigr] = 0.\]
By (\ref{eq:Delta}) and (\ref{eq:equivDelta}), we know that  for all $i\in I$, we have:
\[ \delta^A_i :=  \min_{j\in J}\delta^A(i ,j) := \min_{j\in J}\,\max [(b_i - a_{ij})^+ , \max_{k\in I} \sigma_G(b_i , a_{kj} ,b_k) ] = 0,\]
see (\ref{eq:deltaTi}) and (\ref{eq:deltaTij}).

As $\sigma_G(x , y , z) = \min(\dfrac{(x - z)^+ }{2} , (y - z)^+)$  for all $x , y , z$ in $[0 , 1]$, the following implication is true:
\begin{equation}\label{eq:deltanul}
\delta^A_i = 0 \Longrightarrow \min_{j\in J} \,\max[(b_i - a_{ij})^+, \, \max_{k \in I, k\neq i}(\theta(i,k) \cdot (a_{kj} - b_k)^+)] = 0.    
\end{equation}
We haven proven  (\ref{eq:formulaforLinftydistance}) in the consistent case.

\item Suppose that the system $(S)$ is inconsistent. We take a pair $(\vec i , \vec j)\in L_{A,\infty}$     so by the first statement of  Theorem \ref{theorem:formulaforLinftydistance}, we have 
$$\mathring{\Delta}_{\infty} = \Vert A^{(\vec{i},\vec{j})} - A \Vert_\infty.$$  
As $\vec i = (i_1 , \dots , i_h)$ where $\{i_1 , \dots , i_h\}$ is an enumeration of the set $N_{\text{inc}} = \{ i\in I \, \mid \, \delta^A_i > 0\}$, we deduce from Corollary \ref{corollary:normlambdainfty}, the definition  (\ref{eq:lAinfty}) of the set $L_{A,\infty}$ and  the fourth statement of Lemma \ref{lemma:deltaaijijequalzero}  that  we have:
\begin{align}
\Vert A^{(\vec{i},\vec{j})} - A \Vert_\infty & = \max_{1 \leq \lambda \leq h} 
\Vert A^{( i_\lambda ,j_\lambda)} - A \Vert_\infty\nonumber\\
&= \max_{1 \leq \lambda \leq h} \min_{j\in J} \Vert A^{( i_\lambda ,j)} - A \Vert_\infty\nonumber\\
& = \max_{1 \leq \lambda \leq h} \min_{j\in J}   \max[(b_{i_\lambda} - a_{{i_\lambda} {j}})^+, \, \max_{k \in I, k\neq {i_\lambda}}(\theta({i_\lambda},k) \cdot (a_{k{j}} - b_k)^+)]\nonumber.
\end{align}
We conclude that 
 the equality $\mathring{\Delta}_{\infty} = \max_{i\in I} \min_{j\in J} \,\max\Bigl[(b_i - a_{ij})^+, \, \max_{k \in I, k\neq i}\Bigl(\theta(i,k) \cdot (a_{kj} - b_k)^+\Bigr)\Bigr]$ is equivalent to:
\[ \max_{i\in \overline{N_{\text{inc}}}} \min_{j\in J} \,\max[(b_i - a_{ij})^+, \, \max_{k \in I, k\neq i}(\theta(i,k) \cdot (a_{kj} - b_k)^+)] = 0 \quad \text{with} \quad \max_\emptyset = 0.\]
But, as $\overline{N_{\text{inc}}} = \{i\in I \,\mid \, \delta^A_i = 0\}$, by using (\ref{eq:deltanul}) for all $i\in \overline{N_{\text{inc}}}$, we get:
$$ \min_{j\in J} \,\max[(b_i - a_{ij})^+, \, \max_{k \in I, k\neq i}(\theta(i,k) \cdot (a_{kj} - b_k)^+)] = 0.$$ We haven proven  (\ref{eq:formulaforLinftydistance}) in the inconsistent case.
\end{itemize}
\end{proof}
\end{document}